\documentclass[10pt]{article}
 \usepackage{hyperref}
 \let\savedaddcontentsline\addcontentsline
\usepackage[preprint]{rlj}
 \let\addcontentsline\savedaddcontentsline

\usepackage{amssymb}
\usepackage{mathtools}
\usepackage{mathrsfs}
\usepackage{graphicx}
\usepackage{url}
\usepackage{booktabs}
\usepackage{amsmath}
\usepackage{amsfonts}
\usepackage{amsthm}
\usepackage[capitalize,noabbrev]{cleveref}
\usepackage{tikz}
\usepackage{xurl}
\usepackage{multirow}
\usepackage{titletoc}
\usepackage{thm-restate}
\usepackage{longtable}

\crefname{equation}{}{} 

\newtheorem{definition}{Definition}
\newtheorem{theorem}{Theorem}
\newtheorem{apxdefinition}[definition]{Definition}
\newtheorem{apxtheorem}[theorem]{Theorem}
\newtheorem{apxproposition}{Proposition}
\newtheorem{apxcorollary}{Corollary}

\renewcommand\paragraph[1]{\textbf{#1}\ }

\newcommand\heatmapwidth{0.175\linewidth}
\newcommand\halfheatmapwidth{0.08\linewidth}

\newcommand\Reals{\mathbb{R}}
\newcommand\Expect[2][]{\mathbb{E}_{#1}\!\left[\,#2\,\right]}
\newcommand{\argmax}{\operatornamewithlimits{\textnormal{arg\,max}}}
\newcommand{\argmin}{\operatornamewithlimits{\textnormal{arg\,min}}}
\newcommand{\xargmax}[1][\eps]{\operatornamewithlimits{\textnormal{arg-\ensuremath{#1}-max}}}
\newcommand{\xargmin}[1][\eps]{\operatornamewithlimits{\textnormal{arg-\ensuremath{#1}-min}}}
\newcommand{\supp}{\operatorname{\textnormal{supp}}}
\renewcommand{\max}{\operatornamewithlimits{\textnormal{max}}}
\renewcommand{\min}{\operatornamewithlimits{\textnormal{min}}}

\newcommand\Distr[1]{\Delta(#1)}
\newcommand\eps{\varepsilon}
\newcommand{\blank}{{-}}
\newcommand\States{\mathcal{S}}
\newcommand\Actions{\mathcal{A}}
\newcommand\InitialStateDistribution{\mathcal{I}}
\newcommand\TransitionMap{\mathcal{T}}
\newcommand\level\theta
\newcommand\LevelSpace{\Theta}
\newcommand\TupleUMDP{\langle\LevelSpace,\allowbreak\Actions,\allowbreak\States,\allowbreak\InitialStateDistribution,\allowbreak\TransitionMap\rangle}
\newcommand\LevelDistrs{\Distr{\LevelSpace}}
\newcommand\distr{\Lambda}
\newcommand\distrTrain[1][]{\distr^{\textnormal{Train}}_{#1}}
\newcommand\distrTest[1][]{\distr^{\textnormal{Deploy}}_{#1}}
\newcommand\TupleShift{\langle\alpha,\allowbreak\beta,\allowbreak{}C,\allowbreak\distrTrain,\allowbreak\distrTest\rangle}

\newcommand\distrTrainDistg{\distrTrain[\textnormal{Distg.}]}
\newcommand\distrTrainNonDistg{\distrTrain[\neg\textnormal{Distg.}]}

\newcommand\distrTestDistg{\distrTest[\textnormal{Distg.}]}
\newcommand\distrTestNonDistg{\distrTest[\neg\textnormal{Distg.}]}
\newcommand\distrNonDistg{\distr_{\neg\textnormal{Distg.}}}
\newcommand\distrDistg{\distr_{\textnormal{Distg.}}}
\newcommand\RewardFunction{R}
\newcommand\ProxyRewardFunction{\tilde{\RewardFunction}}
\newcommand\Policy{\pi}

\newcommand\AllPolicies{\Pi}
\newcommand\SomePolicies{\Phi}
\newcommand\DiscountRate{\gamma}
\newcommand\ReturnSymbol{{V}}
\newcommand\Return[3]{\ReturnSymbol^{#1}(#3; #2)}
\newcommand\RegretSymbol{{G}}
\newcommand\Regret[4][]{\RegretSymbol_{#1}^{#2}(#4; #3)}
\newcommand\MMSymbol{{M}}

\newcommand\MEVPolicy{\Policy^{\text{\sc MEV}}}
\newcommand\MMERPolicy{\Policy^{\text{\sc MMER}}}
\newcommand\MEVOptimalPolicies[3][\eps]{\AllPolicies^{\text{\sc MEV}}_{#1}(#2, #3)}
\newcommand\MMEROptimalPolicies[2][\eps]{\AllPolicies^{\text{\sc MMER}}_{#1}(#2)}
\newcommand\MMEROnePolicies[1][\eps]{\AllPolicies^{\textnormal{MMER(1)}}_{#1}}
\newcommand\MMERTwoPolicies[1][\eps,\delta]{\AllPolicies^{\textnormal{MMER(2)}}_{#1}}
\newcommand\MMERThreePolicies[1][\eps,\lambda]{\AllPolicies^{\textnormal{MMER(3)}}_{#1}}
\newcommand\PartialMMERPolicies[2][\SomePolicies,\eps]{\AllPolicies^{\text{\sc MMER}}_{#1}(#2)}

\newcommand\OptimalPolicy{\Policy^{\star}}
\newcommand\OptimalPolicies{\AllPolicies^{\star}}
\newcommand\ProxyOptimalPolicies{\tilde{\AllPolicies}^{\star}}
\newcommand\MMEVOptimalPolicies[2][\eps,\delta]{\AllPolicies^{\text{\sc MMEV}}_{#1}(#2)}
\newcommand\MMEVPolicy{\Policy^{\text{\sc MMEV}}}
\newcommand\MMEVDistr{\distr^{\text{\sc MMEV}}}

\newcommand\MaxMCActor[3]{\hat\RegretSymbol^{#1}_{\textnormal{max-latest}}(#3; #2)}
\newcommand\OracleActor[3]{\hat\RegretSymbol^{#1}_{\textnormal{oracle-latest}}(#3; #2)}
\newcommand\MMActor[3]{\hat\MMSymbol^{#1}_{\textnormal{latest}}(#3; #2)}

\newcommand\RobustPLR{PLR$^\bot$}

\newcommand\ACCELc{ACCEL$^{\textnormal{constant}}$}
\newcommand\ACCELid{ACCEL$^{\textnormal{identity}}$}
\newcommand\ACCELbin{ACCEL$^{\textnormal{binomial}}$}
\newcommand\ACCELunr{ACCEL$^{\textnormal{unrestricted}}$}
\newcommand\env[1]{{\textsc{#1}}}

\title{%
    \centering
    Mitigating Goal Misgeneralization \\
    via Minimax Regret
}

\author{%
    Karim Abdel Sadek\textsuperscript{$=$,1,2},
    Matthew Farrugia-Roberts\textsuperscript{$=$,3},
    \\
    Usman Anwar\textsuperscript{4},
    Hannah Erlebach\textsuperscript{5},
    Christian Schroeder de Witt\textsuperscript{3},
    \\
    David Krueger\textsuperscript{6},
    Michael Dennis\textsuperscript{7}
}

\emails{\textnormal{Correspondence to\ }karimabdel@berkeley.edu\textnormal{,} matthew@far.in.net}

\affiliations{
    $^{1}$\textbf{University of California, Berkeley}\\
    $^{2}$\textbf{University of Amsterdam}\\
    $^{3}$\textbf{University of Oxford}\\
    $^{4}$\textbf{University of Cambridge}\\
    $^{5}$\textbf{University College London}\\
    $^{6}$\textbf{Mila, University of Montreal}\\
    $^{7}$\textbf{Google DeepMind}\\
    \par%
    $^{=}$Equal contribution
}

\begin{document}

\maketitle

\fancypagestyle{firststyle}{
   \fancyhf{}
   \fancyhead[L]{Published as a conference paper at RLC 2025}
   \fancyfoot[C]{\thepage}
}
\thispagestyle{firststyle}

\begin{abstract}
    Safe generalization in reinforcement learning requires not only that a learned policy acts capably in new situations, but also that it uses its capabilities towards the pursuit of the designer's \emph{intended goal.}
    The latter requirement may fail when a \emph{proxy goal} incentivizes similar behavior to the intended goal within the training environment, but not in novel deployment environments.
    This creates the risk that policies will behave as if in pursuit of the proxy goal, rather than the intended goal, in deployment---a phenomenon known as \emph{goal misgeneralization.}
    In this paper, we formalize this problem setting in order to theoretically study the possibility of goal misgeneralization under different training objectives.
    We show that goal misgeneralization is possible under approximate optimization of the maximum expected value (MEV) objective, but not the minimax expected regret (MMER) objective.
    We then empirically show that the standard MEV-based training method of domain randomization exhibits goal misgeneralization in procedurally-generated grid-world environments, whereas current regret-based unsupervised environment design (UED) methods are more robust to goal misgeneralization (though they don't find MMER policies in all cases).
    Our findings suggest that minimax expected regret is a promising approach to mitigating goal misgeneralization.
\end{abstract}

\section{Introduction}
\label{sec:introduction}
As reinforcement learning (RL) is increasingly applied in complex, open-ended, real-world environments, it is becoming infeasible for training to comprehensively cover all situations an agent will face in deployment. We therefore need training methods to produce policies that \emph{generalize}, behaving as intended when faced with a novel scenario \citep{Kirk+2023}.

A particular challenge arises when incomplete coverage of the environment space during training creates a \emph{proxy goal.}
A proxy goal is a reward function that, compared to the true goal, induces similar optimal behavior in most situations encountered during training, but induces radically different behavior in some novel situations.
Proxy goals create the risk of \emph{goal misgeneralization}---learning a policy that retains its capabilities in novel situations, but behaves as if to pursue the proxy goal instead of the true goal when the two diverge \citep{Langosco+2022,Shah+2022}.

Goal misgeneralization can arise when such ``proxy-distinguishing'' situations---where the proxy goal diverges from the true goal---are rare within training, making policies that pursue the wrong goal approximately optimal in terms of the standard RL objective of maximum expected value (MEV).
This motivates the need for training methods that can somehow identify proxy-distinguishing situations within a complex environment, and ensure they are adequately covered in the training distribution.

We observe that favoring the proxy goal in proxy-distinguishing situations leads to high \emph{expected regret,} defined as the shortfall of expected return compared to that obtained by an optimal policy.
An environment selected to maximize a policy's expected regret will naturally include proxy-distinguishing situations as long as the policy ignores the true goal.
Therefore, we propose mitigating goal misgeneralization via the \emph{minimax expected regret (MMER)} objective \citep{Savage1951}.

\begin{figure}
    \centering
    \includegraphics{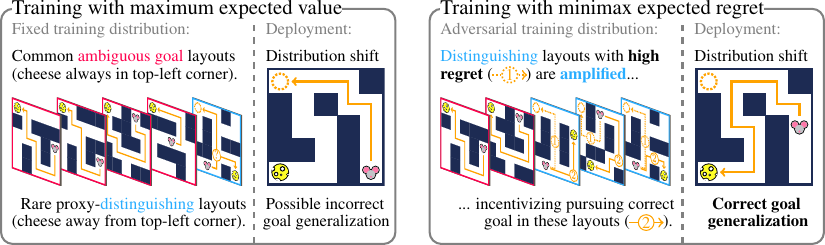}
    \caption{\label{fig:opening}%
        \textbf{Maximum expected value and minimax expected regret vs.\ goal misgeneralization.}
        A mouse searches a maze for cheese that is usually located in the top-left corner. There is a proxy goal (``go to the corner'') that mostly incentivizes the same optimal behavior as the true goal (``go to the cheese'').
        \textit{(Left):}
        Standard RL methods that \emph{approximately} maximize expected value/return could find a policy that behaves as if pursuing the proxy goal rather than the true goal, since layouts where this policy fails are rare in training. This would lead to incorrect generalization.
        \textit{(Right):}
        If a policy ignores the cheese, a regret-maximizing adversary can move the cheese away from the corner until the agent internalizes the correct goal, leading to correct generalization.
    }
\end{figure}

In this paper, we conduct a theoretical and empirical investigation of the possibility of goal misgeneralization under the MEV and MMER objectives.
An outline of our contributions is as follows.
\begin{enumerate}
    \item 
        In \cref{sec:setting}, we introduce a problem setting called a \emph{proxy-distinguishing distribution shift}, formalizing a class of situations in which goal misgeneralization can arise.

    \item
        In \cref{sec:theory}, we show formally that
            (1)~approximately optimizing MEV is susceptible to goal misgeneralization if proxy-distinguishing situations are sufficiently rare (\cref{thm:mev-susceptibility}),
        and (2)~approximately optimizing MMER is provably robust to goal misgeneralization (\cref{thm:mmer-robustness}).
    
    \item 
        In \cref{sec:methods,sec:results}, we empirically study the robustness to goal misgeneralization of a standard MEV-based training method
            \citetext{domain randomization; \citealp{Tobin+2017}},
        and recent MMER-based training methods
            \citetext{regret-based unsupervised environment design; \citealp{Dennis+2020,Jiang+2022,ParkerHolder+2023}}
        under a proxy-distinguishing distribution shift.
\end{enumerate}
Our theoretical results show that, in the limit of idealized training methods, MMER-based training is guaranteed to be robust against goal misgeneralization, whereas MEV-based training is not.
Our empirical results show that current MMER-based training methods are indeed more robust to goal misgeneralization than MEV-based training is, and, while they sometimes still exhibit goal misgeneralization, this happens less for more advanced methods.
Together, these results establish MMER-based training as a promising approach to preventing goal misgeneralization.

\section{Related work}

\textbf{Goal misgeneralization.}
Ensuring learned systems generalize as intended in novel situations is a perennial challenge for deep learning and deep RL \citep{Kirk+2023}.
\citet{Christiano2018} distinguishes \emph{benign} generalization failures, where an agent fails to behave capably in a novel situation, from \emph{malign} generalization failures, where the agent demonstrates capable behavior towards the pursuit of an unintended objective.
\citet{Langosco+2022} and \citet{Shah+2022} demonstrate behavioral examples of malign generalization failures in deep RL, introducing the term \emph{goal misgeneralization.}
Goal misgeneralization is similar to \emph{shortcut learning} in supervised learning \citep{Geirhos+2020}, but emphasizes shortcut \emph{reward functions,} rather than shortcut policies \citep[cf.,][]{suau2024badhabits}.

Recent work proposes complementary approaches to mitigating the risk of goal misgeneralization.
\citet{starace2023addressing} investigates influencing the agent's inductive bias in favor of correct goal generalization using goal-conditioned RL with natural language task descriptions.
\citet{trinh2024getting} studies methods for detecting when the agent is in an unfamiliar situation and choosing to ask an expert (at a cost) to clarify the optimal action.

\textbf{Training in complex environments.}
The standard technique for RL in complex environments is to train on situations sampled from a fixed distribution, a technique known as domain randomization \citep[e.g.,][]{Tobin+2017,peng2018sim}. Maximizing expected return over such situations corresponds to pursuing the MEV objective with respect to the fixed training distribution.

\citet{Dennis+2020} proposed \emph{regret-based unsupervised environment design (UED)},
an RL training technique featuring an adversarial environment designer that continually adapts the training distribution aiming to maximize the agent's expected regret. Maximizing expected return on this adversarial distribution corresponds to the MMER objective \citep{Dennis+2020}.
UED has been promoted as a technique for
    (1)~improving sample efficiency by creating an emergent curriculum;
and (2)~improving capability generalization via adversarial robustness
    \citep{Dennis+2020,Jiang+2022,ParkerHolder+2023}.
We show that UED also helps to mitigate goal misgeneralization.

Alternative adversarial approaches, such as maximin expected value \citep{Dennis+2020,Wang+2023}, maximizing diversity \citep{OpenAI+2019}, or maximizing learnability \citep{Rutherford+2024}, have not been studied in the context of goal misgeneralization.
These approaches may also mitigate goal misgeneralization to the extent that they promote training in proxy-distinguishing situations, incentivizing the agent to internalize the true goal. We show that directly optimizing the training environment for regret is sufficient.
\cref{apx:maximin} shows that minimax expected value can exhibit goal misgeneralization when some situations have low maximum expected return.

\section{Preliminaries}
\label{sec:preliminaries}

A (reward-free) underspecified Markov decision process (UMDP) is a tuple
    $M = \TupleUMDP$
where
    $\LevelSpace$ is a space of free parameters (also called \textbf{levels}),
    $\Actions$ is the agent's action space,
    $\States$ is a state space,
    $\InitialStateDistribution : \LevelSpace \to \Distr{\States}$ is an initial state distribution, and
    $\TransitionMap : \LevelSpace \times \States \times \Actions \to \Distr{\States}$ is a conditional transition distribution.
For simplicity, we assume $\LevelSpace$, $\States$, $\Actions$ are finite.
Given a level $\level \in \LevelSpace$ we have a fully-specified (reward-free) MDP
\(\langle
    \Actions,
    \States,
    \InitialStateDistribution(\level),
    \TransitionMap(\level, \blank, \blank)
\rangle\).
We aggregate these MDPs into a single complex environment using a \textbf{level distribution} $\distr \in \LevelDistrs$.
A \textbf{reward function} (or \textbf{goal}) is a function
    $\RewardFunction: \States \times \Actions \times \States \to \Reals$.
Taken together, $M$ and $\RewardFunction$ define a proper (non-reward-free) UMDP. We define reward functions and reward-free UMDPs separately to facilitate considering multiple goals for an otherwise fixed environment.
We usually denote by $\RewardFunction$ and $\ProxyRewardFunction$ the \textbf{true goal} and the \textbf{proxy goal,} respectively. 

An agent's \textbf{policy} is a conditional action distribution
    $\Policy : \LevelSpace \times \States \to \Distr{\Actions}$.
Note that we assume the policy observes the level (we consider the partially observable case in \cref{apx:partial-observability}).
The set of all policies is denoted by $\AllPolicies$.
We define the \textbf{expected return} (or \textbf{expected value})
    of policy $\Policy$ in level $\level$ under goal $\RewardFunction$
as the discounted cumulative reward
\(
    \Return{\RewardFunction}{\level}{\Policy} = 
    \Expect{
        \sum_{t=0}^\infty \DiscountRate^t \RewardFunction(s_t, a_t, s_{t+1})
    }
\)
where
    $\gamma \in (0,1)$ is a discount factor
and the expectation is over
    $s_0 \sim \InitialStateDistribution(\level)$,
    $a_t \sim \Policy(\level, s_t)$, and
    $s_{t+1} \sim \TransitionMap(\level, s_t, a_t)$.
We lift this definition to level distributions as
    $\Return{\RewardFunction}{\distr}{\Policy} = \Expect[\theta \sim \distr]{\Return{\RewardFunction}{\level}{\Policy}}$. 
A \textbf{normalized} goal is one such that the return has support in $[0, 1]$.

We define the \textbf{expected regret} of a policy $\Policy$ in the level $\level$ under a goal $\RewardFunction$ as the shortfall of expected value achieved by the policy compared to an optimal policy for that level,
\begin{equation}\label{eq:level-regret}
    \Regret{\RewardFunction}{\level}{\Policy}
    =
    \max_{\Policy' \in \AllPolicies}
        \Return{\RewardFunction}{\level}{\Policy'}
    -
    \Return{\RewardFunction}{\level}{\Policy}
.
\end{equation}
Once again, we lift this definition to level distributions as
    $\Regret{\RewardFunction}{\distr}{\Policy} = \Expect[\level \sim \distr]{\Regret{\RewardFunction}{\level}{\Policy}}$.
Since the policy is conditioned on $\level$, we also have the following identity (see \cref{apx:regret-identity}):
\begin{equation}\label{eq:distr-regret}
    \Regret{\RewardFunction}{\distr}{\Policy}
    =
    \max_{\Policy' \in \AllPolicies}
        \Return{\RewardFunction}{\distr}{\Policy'}
    -
    \Return{\RewardFunction}{\distr}{\Policy}
.
\end{equation}

\section{Problem setting}
\label{sec:setting}

\citet{Langosco+2022} and \citet{Shah+2022} provide case studies of several situations in which goal misgeneralization arises.
In order to theoretically study goal misgeneralization, we formalize an abstract class of situations in which goal misgeneralization can arise as a problem setting called a \emph{proxy-distinguishing distribution shift}.

\subsection{Level classification}

The problem setting is defined in terms of the following classification of levels, based on whether a given proxy goal incentivizes optimal or suboptimal behavior under a true goal.

\begin{definition}[Proxy non-distinguishing level]
\label{def:non-distinguishing}
Consider
    an UMDP $\TupleUMDP$,
    a true goal $\RewardFunction$,
and 
    a proxy goal $\ProxyRewardFunction$.
A level $\level \in \LevelSpace$ is
    \emph{proxy non-distinguishing}
with respect to $\RewardFunction$ and $\ProxyRewardFunction$
if all optimal policies with respect to $\ProxyRewardFunction$ are also optimal with respect to $\RewardFunction$, that is,
\[\textstyle
    \argmax_{\Policy \in \AllPolicies} \Return{\ProxyRewardFunction}{\level}{\Policy}
    \subseteq
    \argmax_{\Policy \in \AllPolicies} \Return{\RewardFunction}{\level}{\Policy}
.\]
\end{definition}

\begin{definition}[Possibly proxy $C$-distinguishing level]
\label{def:distinguishing}
Consider
    an UMDP $\TupleUMDP$,
    a true goal $\RewardFunction$,
    a proxy goal $\ProxyRewardFunction$,
and
    a constant $C \geq 0$.
A level $\level \in \LevelSpace$ is
    \emph{possibly proxy $C$-distinguishing}
with respect to $\RewardFunction$ and $\ProxyRewardFunction$
if some policy that is optimal with respect to $\ProxyRewardFunction$ achieves sufficiently suboptimal expected return with respect to $\RewardFunction$, that is,
\[\textstyle
    \argmax_{\Policy \in \AllPolicies} \Return{\ProxyRewardFunction}{\level}{\Policy}
    \not\subseteq
    \xargmax[C]_{\Policy \in \AllPolicies}  \Return{\RewardFunction}{\level}{\Policy}
\]
where \(
    \xargmax[C]_{\Policy \in \AllPolicies}
        \Return{\RewardFunction}{\level}{\Policy}
    =
    \lbrace
        \Policy \in \AllPolicies
    \mid
        \Return{\RewardFunction}{\level}{\Policy}
        \geq
        \max_{\Policy' \in \AllPolicies}
            \Return{\RewardFunction}{\level}{\Policy'}
        - C
    \rbrace
\).
\end{definition}

For brevity, we usually drop the prefixes ``proxy'' and ``possibly proxy'' and refer simply to ``non-distinguishing'' and ``distinguishing'' levels.
The prefix ``proxy'' signifies that the two goals play asymmetric roles in the definitions, because our interest is in whether training in a level disincentivizes policies that pursue the proxy goal at the expense of the true goal, and not the other way around.
Proxy non-distinguishing levels never offer a training signal against optimally pursuing the proxy goal, because all such policies are necessarily also optimal under the true goal.
In contrast, in proxy distinguishing levels, there exist policies that pursue the proxy goal at the expense of the true goal. These policies are disincentivized when training in these levels.
Note that while this is the case for some policies that are optimal under the proxy goal, it may not be the case for a given policy, hence ``possibly.''

If $C=0$, the classification is exhaustive. \Cref{fig:level-types} gives several examples.
If $C > 0$, the classification is not necessarily exhaustive: there may exist levels for which it is possible to optimally pursue the proxy goal while remaining approximately optimal under the true goal. However, it is often true that optimizing a misspecified goal leads to arbitrarily low return \citep[cf.,][]{zhuang2020consequences}. Therefore, possibly proxy $C$-distinguishing levels represent an important class of levels.

\begin{figure}[t]
    \centering
    \begingroup
    \setlength{\tabcolsep}{1.8pt}
    \begin{tabular}{cccccc}
    \toprule
        \multicolumn{2}{c}{\bf Proxy non-distinguishing levels}
    &   \null\quad\null
    &   \multicolumn{3}{c}{\bf Possibly proxy distinguishing levels ($C=0$)}
    \\
    \cmidrule{1-2}\cmidrule{4-6}
        $\OptimalPolicies = \ProxyOptimalPolicies$
    &   $\ProxyOptimalPolicies \subsetneq \OptimalPolicies$
    &
    &   $\OptimalPolicies \subsetneq \ProxyOptimalPolicies$
    &   $\OptimalPolicies \cap \ProxyOptimalPolicies = \emptyset$
    &   otherwise
    \\
    \midrule
      \includegraphics{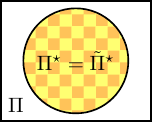}
    & \includegraphics{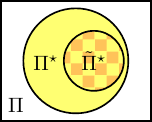}
    &
    & \includegraphics{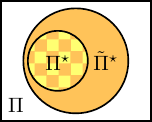}
    & \includegraphics{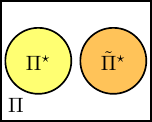}
    & \includegraphics{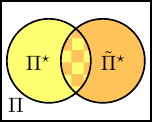}
    \\
      \includegraphics[width=1in]{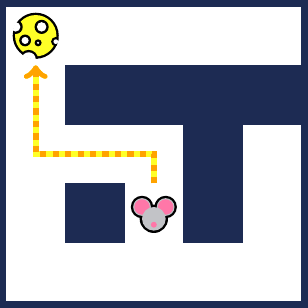}
    & \includegraphics[width=1in]{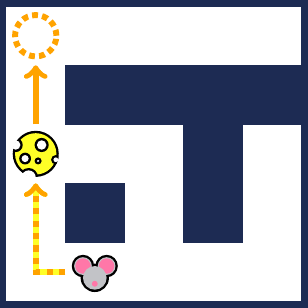}
    &
    & \includegraphics[width=1in]{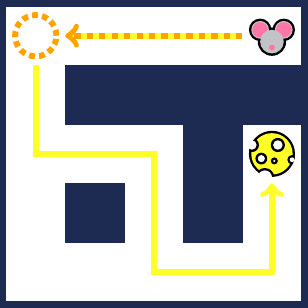}
    & \includegraphics[width=1in]{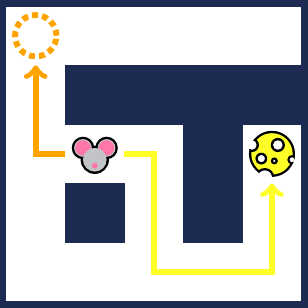}
    & \includegraphics[width=1in]{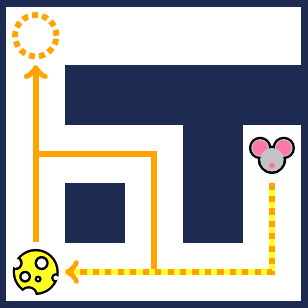}
    \\
    \bottomrule
    \end{tabular}
    \endgroup
    \caption{\label{fig:level-types}%
        \textbf{Illustration and examples of non-distinguishing and distinguishing levels.}
        A level $\level \in \LevelSpace$ can be classified as \emph{non-distinguishing} or \emph{$0$-distinguishing} (\cref{def:non-distinguishing,def:distinguishing}).
        \textit{(1st~row):}
            Possible relationships between sets
            $\OptimalPolicies = \argmax_{\Policy \in \AllPolicies} \Return{\RewardFunction}{\level}{\Policy}$
        and 
            $\ProxyOptimalPolicies = \argmax_{\Policy \in \AllPolicies} \Return{\ProxyRewardFunction}{\level}{\Policy}$.
        \textit{(2nd~row):}
            Example levels for a navigation environment. Yellow and orange arrows show optimal behaviors for the true goal (``go to the cheese'') and a proxy goal (``go to the corner''), respectively.
    }
\end{figure}

\subsection{Proxy-distinguishing distribution shift}

We model a shift from training to deployment as a change in the distribution of available levels (from a \emph{training distribution} to a \emph{deployment distribution}).
The distribution shift is \emph{proxy-distinguishing} when the training distribution concentrates mostly on non-distinguishing levels, but the deployment distribution concentrates mostly on distinguishing levels.

\begin{definition}[Proxy-distinguishing distribution shift]
\label{def:proxy-distinguishing-distribution-shift}
Consider
    an UMDP $\TupleUMDP$,
    a true goal $\RewardFunction$,
and 
    a proxy goal $\ProxyRewardFunction$. 
A \emph{proxy-distinguishing distribution shift} is a tuple  $\TupleShift$ where 
    $\alpha, \beta$ are ratios such that $0 \leq \alpha < \beta \leq 1$,
    $C \geq 0$ is a constant,
and
    $\distrTrain, \distrTest \in \LevelDistrs$ are distributions over levels with the following classifications (with respect to $\RewardFunction$ and $\ProxyRewardFunction$):
\begin{enumerate}
    \item 
        $\distrTrain$ has probability $\alpha$ on $C$-distinguishing levels and the rest on non-distinguishing levels.
    \item 
        $\distrTest$ has probability $\beta$ on $C$-distinguishing levels and the rest on non-distinguishing levels.
\end{enumerate}
\end{definition}

We are mainly interested in the case where $\alpha$ is very close to zero (where goal misgeneralization is a particular risk) and $\beta$ is very close to one (where goal misgeneralization is a particular concern).

\subsection{Assumptions}

We don't assume prior knowledge of the proxy goal or distinguishing levels.
However, we \emph{do} assume the ability to train in distinguishing levels, once identified.
In practice, one can train in a very wide space of situations, whether via a simulator \citep{Tobin+2017,peng2018sim,kumar2021rma,makoviychuk2021isaac,muratore2022robot,ma2024dreureka}, a generative environment model \citep{bruce2024genie},
or a world model \citep{ha2018world, hafner2019dream, schrittwieser2020mastering, hafner2023mastering, valevski2024diffusion}.
If \emph{all} levels accessible before deployment are non-distinguishing, we may require alternative assumptions \citep[see, e.g.,][]{trinh2024getting}.

Moreover, we assume access to a reliable reward signal in favor of the true goal in distinguishing levels.
This mirrors assumptions made in work on spurious correlations in supervised learning \citep[e.g.,][]{liu2021just,zhang2022correct}. 
However, in practice, reward functions may be subject to misspecification in such corner cases \citep[cf.,][]{HadfieldMenell+2017}.
Future work could develop methods that treat the true goal as underspecified in rare, distinguishing levels and find ways to incentivize safe generalization behavior despite this uncertainty.

\section{Theoretical results}
\label{sec:theory}

In this section, we prove that under a proxy-distinguishing distribution shift, the maximum expected value (MEV) objective permits an approximately optimal policy that exhibits goal misgeneralization.
On the other hand, we show that any policy that is approximately optimal with respect to minimax expected regret (MMER) must avoid goal misgeneralization.
All proofs are in \cref{sec:main-proofs}.

We consider \emph{approximately} optimal policies because, in practice, training uses finite optimization power and will not always find policies that are exactly optimal.
We model approximate optimization as instead finding an arbitrary policy within a small threshold of optimal for the given objective.
We use the notation
\(
    \xargmax_{x \in \mathcal{X}} f(x) = \lbrace
        x \in \mathcal{X}
    \mid
        f(x) \geq \max_{\xi \in \mathcal{X}} f(\xi) - \eps
    \rbrace
\)
(likewise, $\xargmin$) for approximate optimization of a function
    $f : \mathcal{X} \to \Reals$
with approximation threshold $\eps \geq 0$.

\subsection{Approximate maximum expected value is susceptible to goal misgeneralization}

The standard objective used in RL is the maximum expected value (MEV) objective with respect to the fixed training distribution. We formalize approximate solutions under this objective as follows.

\begin{definition}[Approximate MEV]\label{def:approx-mev}
    Consider
        an UMDP $\TupleUMDP$,
        a goal $\RewardFunction$,
        an approximation threshold $\eps \geq 0$,
    and a fixed level distribution
        $\distr \in \LevelDistrs$.
    The \emph{approximate MEV policy set} with respect to $\distr$ is then
    \begin{equation*}
        \MEVOptimalPolicies{\RewardFunction}{\distr}
        =
        \xargmax_{\Policy \in \AllPolicies}
            \Return{\RewardFunction}{\distr}{\Policy}
    .
    \end{equation*}
\end{definition}

The MEV objective permits goal misgeneralization under a proxy-distinguishing distribution shift if the proportion of distinguishing levels in training is too small.
Intuitively, a policy pursuing the proxy goal in all levels achieves enough return on non-distinguishing levels to be approximately optimal.
Note: rather than modeling inductive bias, we characterize the \emph{possibility} of goal misgeneralization.

\begin{restatable}[MEV is susceptible to goal misgeneralization]{theorem}{MEVTheorem}
\label{thm:mev-susceptibility}
Consider
    an UMDP $\TupleUMDP$,
    a pair of normalized goals $\RewardFunction, \ProxyRewardFunction$,
    a proxy-distinguishing distribution shift $\TupleShift$,
and
    an approximation threshold $\eps \geq 0$.
If $\alpha \leq \eps$, then there exists
    $\MEVPolicy \in \MEVOptimalPolicies[\eps]{\RewardFunction}{\distrTrain}$
such that
\begin{equation*}
    \MEVPolicy
    \in
    \argmax_{\Policy \in \AllPolicies} \Return{\ProxyRewardFunction}{\distrTest}{\Policy}
    \setminus
    \xargmax[\beta C]_{\Policy \in \AllPolicies} \Return{\RewardFunction}{\distrTest}{\Policy}.
\tag*{[\hyperlink{proof:thm:mev-susceptibility}{proof}]}
\end{equation*}
\end{restatable}

\subsection{Approximate minimax expected regret is robust to goal misgeneralization}
\label{sec:mmer-robustness}

The MMER objective is to \emph{minimize} the expected regret assuming an \emph{adversarially-chosen} (\emph{maximum} expected regret) level distribution for each policy.
We consider approximate minimization against an exactly optimal adversary, but a similar robustness property holds when using an approximate adversary (the bound worsens linearly in the suboptimality of the adversary, see \cref{apx:approximinimax}).

\begin{definition}[Approximate MMER]\label{def:approx-mmer}
    Consider
        an UMDP $\TupleUMDP$,
        a goal $\RewardFunction$,
    and
        an approximation threshold $\eps \geq 0$.
    The \emph{approximate MMER policy set} is then
    \begin{equation*}
        \MMEROptimalPolicies[\eps]{\RewardFunction}
        =
        \xargmin[\eps]_{\pi\in\AllPolicies}
            \max_{\distr \in \LevelDistrs}
                \Regret{\RewardFunction}{\distr}{\Policy}.
    \end{equation*}
\end{definition}

The MMER objective does not permit goal misgeneralization under any distribution shift within the specified level space.
Intuitively, the adversarial level distribution for a misgeneralizing policy would concentrate on distinguishing levels, generating high expected regret.

\begin{restatable}[MMER is robust to goal misgeneralization]{theorem}{MMERTheorem}
\label{thm:mmer-robustness}
Consider
    an UMDP $\TupleUMDP$,
    a pair of goals $\RewardFunction, \ProxyRewardFunction$,
    a proxy-distinguishing distribution shift $\TupleShift$,
and
    an approximation threshold $\eps \geq 0$.
Then
\[
    \forall \MMERPolicy \in \MMEROptimalPolicies[\eps]{\RewardFunction},
    \text{ we have }
    \MMERPolicy \in \xargmax[\eps]_{\Policy \in \AllPolicies} \Return{\RewardFunction}{\distrTest}{\Policy}
.
\tag*{[\hyperlink{proof:thm:mmer-robustness}{proof}]}\]
\end{restatable}

\paragraph{Remarks.}
As a corollary, any policy that is robust against all possible distribution shifts must be an MMER policy (see \cref{apx:necessitymmer}). \pagebreak
A slightly modified bound holds for partially observable environments after accounting for the minimum expected regret \emph{realizable by a fixed policy} (see \cref{apx:partial-observability}).

\section{Experimental methods}
\label{sec:methods}

In this section, we outline our methods for investigating the robustness to goal misgeneralization of MEV-based and MMER-based training methods.
We construct three custom procedurally-generated grid-world environments approximating proxy-distinguishing distribution shifts (\cref{sec:environments}).
We compare a standard MEV-based training method and two recently proposed MMER-based methods with adversaries of varying flexibility (\cref{sec:algorithms}), paired with regret estimators that leverage varying amounts of domain knowledge (either using the ground truth maximum return, or estimating it from samples; \cref{sec:estimators}).
\Cref{sec:results} presents the results of our experiments.

\subsection{Procedurally-generated grid-world environments}
\label{sec:environments}

\begin{figure}[b!]
    \centering
    \begingroup
    \setlength{\tabcolsep}{0.25em}
    \renewcommand\arraystretch{0.97}
    \begin{tabular}{ccccccccc}
    \toprule
    \multicolumn{2}{c}{\bf \env{Cheese in the corner}}
    &&
    \multicolumn{2}{c}{\bf \env{Cheese on a dish}}
    && 
    \multicolumn{2}{c}{\bf \env{Keys and chests}}
    \\
      \makebox[0pt]{\footnotesize Non-distinguishing}
    & \makebox[0pt]{\footnotesize Distinguishing}
    &
    & \makebox[0pt]{\footnotesize Non-distinguishing}
    & \makebox[0pt]{\footnotesize Distinguishing}
    &
    & \makebox[0pt]{\footnotesize Non-distinguishing}
    & \makebox[0pt]{\footnotesize Distinguishing}
    \\
    \midrule
    \includegraphics[width=0.15\linewidth]{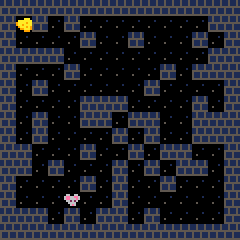}
    &
    \includegraphics[width=0.15\linewidth]{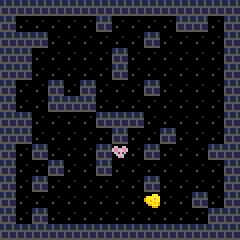}
    &&
    \includegraphics[width=0.15\linewidth]{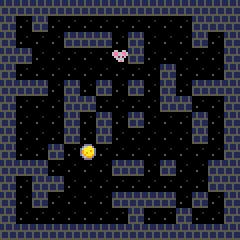}
    &
    \includegraphics[width=0.15\linewidth]{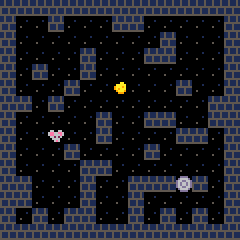}
    &&
    \includegraphics[width=0.15\linewidth]{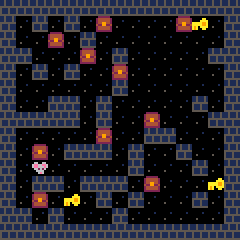}
    &
    \includegraphics[width=0.15\linewidth]{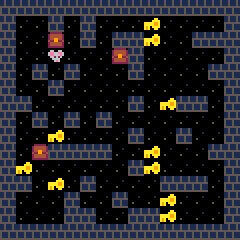}
    \\
    \bottomrule
    \end{tabular}
    \endgroup
    \caption{\label{fig:example-levels}%
        Example procedurally-generated non-distinguishing/distinguishing levels. The agent's observation is a \(15 \times 15 \times c\) Boolean grid (where $c$ is an environment-dependent number of channels).
    }
\end{figure}

\citet{Langosco+2022} exhibited goal misgeneralization in several environments from OpenAI Procgen \citep{Cobbe+2020}, suitably modified to implement a proxy-distinguishing distribution shift with $\alpha = 0$.
We implement three similar procedurally-generated grid-world environments in JAX \citep{JAX}, allowing us to more easily implement custom level generation and analysis.

For each environment, we construct two procedural level generators $\distrNonDistg, \distrDistg \in \LevelDistrs$,  approximately concentrated on non-distinguishing and distinguishing levels, respectively.
From these, we define training distributions 
\(
    \distrTrain[\alpha]
    = (1-\alpha) \distrNonDistg + \alpha \distrDistg
\)
where $\alpha$ is the proportion of distinguishing levels.
In our experiments, we vary $\alpha$ from $10^{-5}$ to $10^{-1}$, with $\alpha\in\{0,1\}$ as baselines.
We evaluate on $\distrTest = \distrDistg$, approximating a proxy-distinguishing distribution shift.

The three environments are as follows.
\Cref{fig:example-levels} illustrates example levels (note we use Boolean observations).
\Cref{apx:environments} comprehensively documents each environment, including the details of classifying levels as non-distinguishing or distinguishing and procedural level generation.
\begin{enumerate}
    \item \textbf{\env{Cheese in the corner}.}
        A mouse navigates a maze.
        The true goal assigns $+1$ reward for reaching a piece of cheese, while a proxy goal assigns $+1$ reward for reaching the top left corner for the first time.
        Levels with the cheese in the top left corner are non-distinguishing and levels with the cheese away from the corner are (in most cases) distinguishing.
    \item \textbf{\env{Cheese on a dish}.}
        This time the mouse navigates a maze containing cheese and also a dish.
        The true goal assigns $+1$ reward for reaching the cheese, while a proxy goal assigns $+1$ reward for reaching the dish.
        Levels with the cheese and dish co-located are non-distinguishing, and levels with the cheese and dish separated are (in most cases) distinguishing.
    \item \textbf{\env{Keys and chests}.}
        A more complex, multi-stage task, in which the mouse navigates a maze, collects keys, and spends keys to open chests.
        Levels with 3 keys and 10 chests are approximately non-distinguishing. Levels with 10 keys and 3 chests are mostly distinguishing---a misgeneralizing policy would overprioritize key collection beyond what is necessary for opening chests.
\end{enumerate}

\subsection{Training methods}
\label{sec:algorithms}

For both MEV-based and MMER-based training, we follow \citet{Langosco+2022} and use an agent network architecture based on that of IMPALA \citep{Espeholt+2018} with a dense feed-forward layer replacing the LSTM block.
We perform policy updates with PPO \citep{schulman2017proximal} and GAE \citep{schulman2015high}.
We document hyperparameters and compute usage in \cref{apx:training-details}.

For MEV, we use a standard method for training in UMDPs given a fixed level distribution.
\begin{enumerate}
    \item \textbf{Domain randomization \textnormal{\citep[DR;][]{Tobin+2017}}.}
        For each iteration of PPO, we sample (procedurally generate) a new batch of levels from the fixed training level distribution $\distrTrain[\alpha]$, collect experience in this batch of levels, and then train on the collected experience.
\end{enumerate}

For MMER, we use two methods of regret-based unsupervised environment design \citep[UED;][]{Dennis+2020}.
UED methods implement the two-level optimization from \cref{def:approx-mmer} by training the policy on levels selected from a distribution chosen by a regret-maximizing \textbf{adversary}.
The first UED method is a regret-based form of prioritized level replay \citep[PLR;][]{Jiang+2021}.
\begin{enumerate}[resume]
    \item \textbf{Robust prioritized level replay \textnormal{\citep[\RobustPLR;][]{Jiang+2022}}.}
        The adversary parametrizes its level distribution using a fixed-size \textbf{level buffer.}
        Throughout training, the adversary refines the buffer by either
            (1)~sampling a new batch of levels from the underlying training distribution $\distrTrain[\alpha]$ and estimating the expected regret of the current policy on these levels;
        or
            (2)~sampling from the current buffer, conducting a PPO training step with the chosen levels, and updating their expected regret estimates;
        keeping the highest-regret levels in the buffer.
\end{enumerate}

\RobustPLR{} has the advantage of being domain-agnostic, but has the disadvantage of only being able to \emph{replay} levels once they have been sampled from the underlying distribution.
We also consider a more advanced adversary with an independent means of exploring the space of level distributions.
\begin{enumerate}[resume]
    \item \textbf{Adversarially compounding complexity by editing levels \textnormal{\citep[ACCEL;][]{ParkerHolder+2023}}.}
        The adversary continually refines a level buffer with steps~(1) and~(2) from \RobustPLR{}, and additionally by (3)~applying stochastic \textbf{edits} to the levels used for PPO training to generate similar levels, and estimating the expected regret of the current policy on these new levels.
\end{enumerate}
ACCEL additionally requires an \textbf{edit distribution.} We edit levels by sampling a sequence of random elementary level modifications, none of which change whether the level is non-distinguishing or distinguishing.
\Cref{apx:mutators} details this edit distribution and compares it to edit distributions with more or less ability to introduce distinguishing levels.

\subsection{Expected regret estimation methods}
\label{sec:estimators}

Both UED methods require an \textbf{(expected) regret estimator} for deciding which levels to keep in the buffer.
To represent the current capabilities of UED methods, we use the following domain-agnostic estimator, similar to the MaxMC estimator proposed by \citet{Jiang+2022}.
\begin{enumerate}
    \item \textbf{Max-latest estimator.}
    We estimate the expected regret of policy $\Policy$ in level $\level$ under goal $\RewardFunction$ as
    \begin{equation}\label{eq:max-latest-estimator}
        \MaxMCActor{\RewardFunction}{\level}{\Policy}
        =
        \hat{\ReturnSymbol}^{\RewardFunction}_{\textnormal{max}}(\level)
        - \hat{\ReturnSymbol}^{\RewardFunction}_{\textnormal{latest}}(\Policy; \level)
    \end{equation}
    where
        $\hat{\ReturnSymbol}^{\RewardFunction}_{\textnormal{max}}(\level)$
        is the highest empirical return ever achieved for this level throughout training;
    and
        $\hat{\ReturnSymbol}^{\RewardFunction}_{\textnormal{latest}}(\Policy;\level)$
        is the empirical average return achieved by the current policy.
\end{enumerate}
To simulate a more advanced regret estimator than is currently available in practice, we also consider a domain-specific estimator that solves each procedurally-generated level using a graph algorithm to compute the exact maximum expected return (details in \cref{apx:oracles}).
\begin{enumerate}[resume]
    \item \textbf{Oracle-latest estimator.}
        We estimate the expected regret of policy $\Policy$ in level $\level$ under goal $\RewardFunction$ as
        \begin{equation}\label{eq:oracle-latest-estimator}
            \OracleActor{\RewardFunction}{\level}{\Policy}
            =
            \max_{\Policy'}  \Return{\RewardFunction}{\theta}{\pi'}
            -
            \hat{\ReturnSymbol}^{\RewardFunction}_{\textnormal{latest}}(\Policy; \level)
        \end{equation}
        where
            $\max_{\Policy'} \Return{\RewardFunction}{\theta}{\Policy'}$ is the maximum expected return for the level;
        and
            $\hat{\ReturnSymbol}^{\RewardFunction}_{\textnormal{latest}}(\Policy;\level)$
        is as above.
\end{enumerate}

\section{Experimental results}
\label{sec:results}

In this section, we report the results of our main experiments.
Consistent with \cref{thm:mev-susceptibility}, MEV-based training is susceptible to goal misgeneralization unless the proportion of distinguishing levels in the training distribution is sufficiently high (\cref{sec:dr-susceptibility}).
Consistent with \cref{thm:mmer-robustness}, MMER-based training methods are typically capable of identifying and increasing the proportion of rare, high-regret, distinguishing levels, thereby preventing goal misgeneralization in many situations where MEV-based training misgeneralizes (\cref{sec:amplification}).
In some cases, UED methods fail to find MMER policies, and exhibit goal misgeneralization. We see generally that the more advanced UED methods are more robust to goal misgeneralization (\cref{sec:ued-advances}).
In \env{keys and chests}, DR outperforms ACCEL with max-latest regret estimation, underscoring reliable regret estimation as a particular challenge for future work on MMER-based training (\cref{sec:outlier}).

\begin{figure}[t]
    \centering
    \begingroup
    \setlength{\tabcolsep}{1pt}
    \begin{tabular}{cccc}
    \toprule
    & \bf \env{Cheese in the corner}
    & \bf \env{Cheese on a dish}
    & \bf \env{Keys and chests}
    \\
    & \small Seeds $N{=}8$, steps $T{=}200$M
    & \small Seeds $N{=}3$, steps $T{=}200$M
    & \small Seeds $N{=}5$, steps $T{=}400$M
    \\
    \midrule
    \multirow{2}{*}[7em]{\rotatebox[origin=l]{90}{Avg. return (distg.)}}
    & \includegraphics{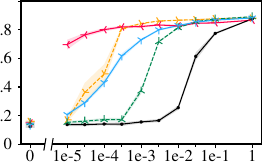}
    & \includegraphics{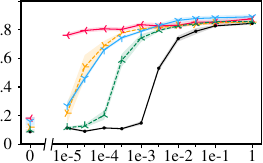}
    & \includegraphics{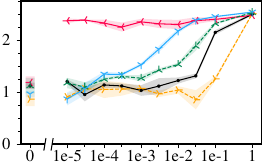}
    \\
    & \multicolumn{3}{c}{Proportion of distinguishing levels in underlying training distribution, $\alpha$ (augmented log scale)}
    \\
    \midrule
    & \multicolumn{3}{c}{\includegraphics{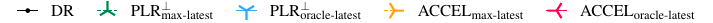}}
    \\[-0.5ex]
    \bottomrule
    \end{tabular}
    \endgroup
    \caption{\label{fig:main-distinguishing-return}%
        \textbf{Distribution shift performance for various training distributions.}
        Average return over 512 steps for an evaluation batch of 256 distinguishing levels sampled from $\distrTest = \distrDistg$.
        High performance indicates policies generalizing as intended; low performance indicates goal misgeneralization.
        Each policy is trained on $T$ environment steps using the indicated training method with underlying training distribution $\distrTrain[\alpha] = (1-\alpha) \distrNonDistg + \alpha \distrDistg$.
        Mean over $N$ seeds, shaded to one standard error.
        Note the split in the horizontal axis used to show zero on the log scale.
    }
\end{figure}

\subsection{Domain randomization exhibits goal misgeneralization with rare distinguishing levels}
\label{sec:dr-susceptibility}

\Cref{thm:mev-susceptibility} says that if the proportion of distinguishing levels in the fixed training distribution is small enough, then approximately optimizing MEV \emph{possibly} leads to goal misgeneralization.
Our experiments show that DR, an MEV-based training method, indeed exhibits goal misgeneralization when the proportion of distinguishing levels in the training distribution is small enough.
\Cref{fig:main-distinguishing-return} shows end-of-training performance on distinguishing levels. There is a threshold below which DR's performance on distinguishing levels falls.
DR achieves high return on non-distinguishing levels and high proxy return on distinguishing levels (\cref{apx:non-distinguishing}), indicating a case of goal misgeneralization.

In \env{cheese in the corner} and \env{keys and chests}, DR exhibits goal misgeneralization until there is around $\alpha=1\text{e-}1$ ($10\%$) mass on distinguishing levels. For \env{cheese on a dish}, DR is robust to goal misgeneralization from as low as $\alpha=1\text{e-}2$ ($1\%$) (see also \cref{apx:robustness-dish}).
\Cref{apx:moresteps} shows that training for substantially longer slightly increases DR's robustness in \env{cheese on a dish}.

We note that \citet{Langosco+2022} previously demonstrated goal misgeneralization while training with DR without distinguishing levels in similar environments.
Moreover, \citet{Langosco+2022} demonstrated that for a modified version of OpenAI ProcGen's \env{CoinRun} environment \citep{Cobbe+2019,Cobbe+2020}, training with $\alpha=2\text{e-2}$ $(2\%)$ prevents goal misgeneralization. They did not experiment with smaller proportions of distinguishing levels.
We show that with small but nonzero proportions of distinguishing levels, DR can still exhibit goal misgeneralization.

\subsection{Regret-based prioritization amplifies distinguishing levels, mitigating misgeneralization}
\label{sec:amplification}

\begin{figure}[t]
    \centering
    \begingroup
    \setlength{\tabcolsep}{1pt}
    \begin{tabular}{cccc}
    \toprule
    & \bf \env{Cheese in the corner}
    & \bf \env{Cheese on a dish}
    & \bf \env{Keys and chests}
    \\
    & \small Seeds $N{=}8$, steps $T{=}200$M
    & \small Seeds $N{=}3$, steps $T{=}200$M
    & \small Seeds $N{=}5$, steps $T{=}400$M
    \\
    \midrule
    \multirow{2}{*}[7em]{\rotatebox[origin=l]{90}{Adv.\ prop.\ distg.\ (log)}}
    & \includegraphics{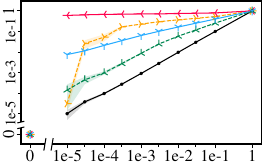}
    & \includegraphics{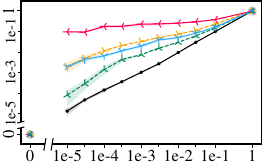}
    & \includegraphics{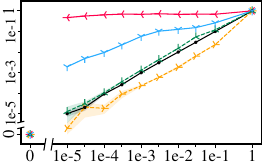}
    \\
    & \multicolumn{3}{c}{Proportion of distinguishing levels in underlying training distribution, $\alpha$ (augmented log scale)}
    \\
    \midrule
    & \multicolumn{3}{c}{\includegraphics{figures/plots/legends/main-both.pdf}}
    \\[-0.5ex]
    \bottomrule
    \end{tabular}
    \endgroup
    \caption{\label{fig:main-effective-alpha}%
        \textbf{Rate at which adversary plays distinguishing levels.}
        We plot the proportion of adversarially sampled levels classified as distinguishing across training for $T$ environment steps.
        The diagonal represents the proportion from the underlying training distribution $\distrTrain[\alpha] = (1-\alpha) \distrNonDistg + \alpha \distrDistg$ (as used in DR).
        Points above the diagonal indicate the adversary increasing the proportion of distinguishing levels relative to the underlying training distribution.
        Mean over $N$ seeds, shaded to one standard error.
        Note the splits in \emph{both} axes used to show zero on the log scales.
    }
\end{figure}

\Cref{thm:mmer-robustness} says that, if a policy pursues the proxy goal at the expense of the true goal in distinguishing levels, then the adversary should select a distribution of distinguishing levels that generates high regret.
\Cref{fig:main-effective-alpha} shows the average proportion of distinguishing levels selected from the adversary throughout training, showing that, with the exception of max-latest estimation in the \env{keys and chests} environment, the adversary selects distinguishing levels disproportionately often compared to sampling from the underlying distribution, thereby incentivizing policies that pursue the true goal.

\cref{fig:main-distinguishing-return} shows that this increase in the proportion of training levels is, in most cases, enough to lead to correct generalization.
In each environment, MMER-based training methods are robust to goal misgeneralization at $\alpha$ values for which DR exhibits goal misgeneralization. For example, in \env{cheese in the corner}, all UED methods are robust to goal misgeneralization at $\alpha=1\text{e-}2$ ($1\%$), and some remain robust for even lower $\alpha$.
Note that some evaluation levels are unsolvable---the highest return to be expected is given by the agents trained with $\alpha=1$.

\begin{figure}
    \begingroup
    \centering
    \setlength{\tabcolsep}{0.25em}
    \renewcommand\arraystretch{1.12}
    \begin{tabular}{cccccc}
    \toprule
    \null & \bf \multirow{2}{*}{DR}
    & \multicolumn{2}{c}{\bf \RobustPLR{}}
    & \multicolumn{2}{c}{\bf ACCEL}
    \\
    $\alpha$
    & 
    & \small max-latest
    & \small oracle-latest
    & \small max-latest
    & \small oracle-latest
    \\
    \midrule
    \raisebox{\halfheatmapwidth}{$1\text{e-}4$}
    & \includegraphics[width=\heatmapwidth]{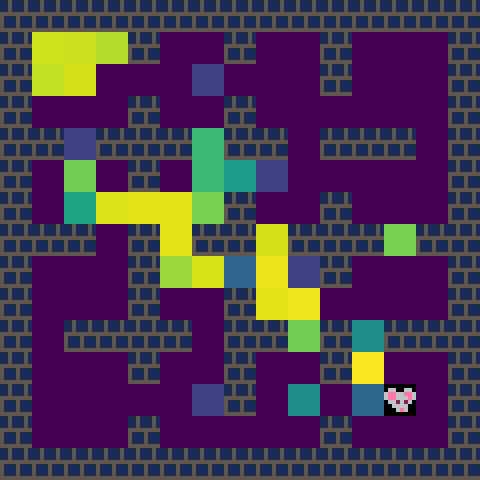}
    & \includegraphics[width=\heatmapwidth]{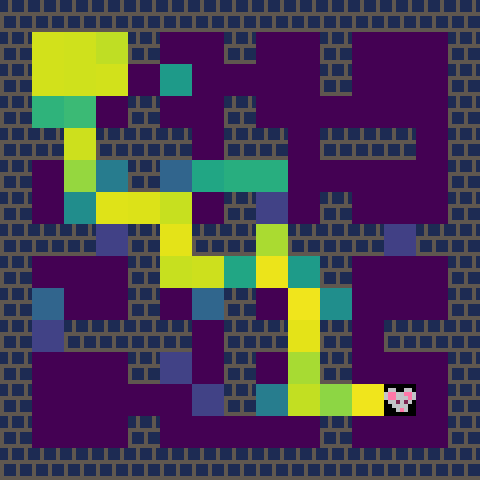}
    & \includegraphics[width=\heatmapwidth]{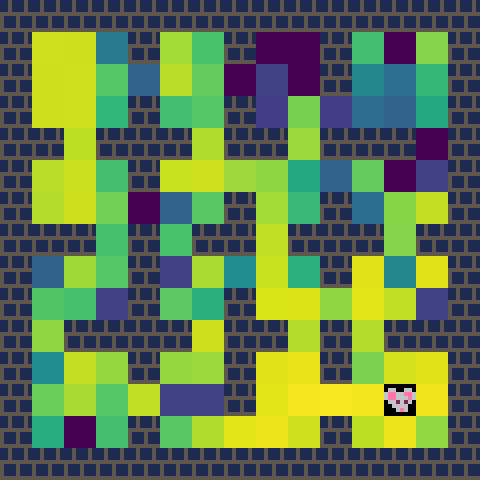}
    & \includegraphics[width=\heatmapwidth]{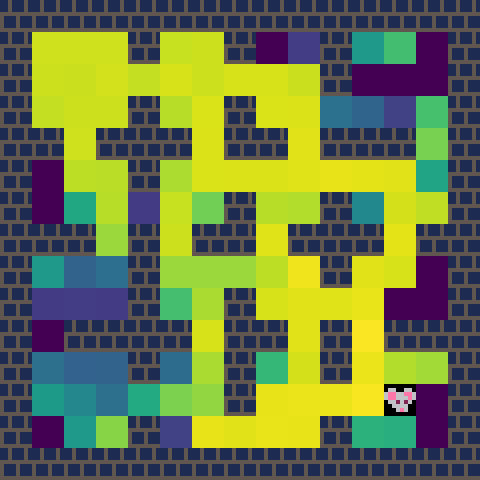}
    & \includegraphics[width=\heatmapwidth]{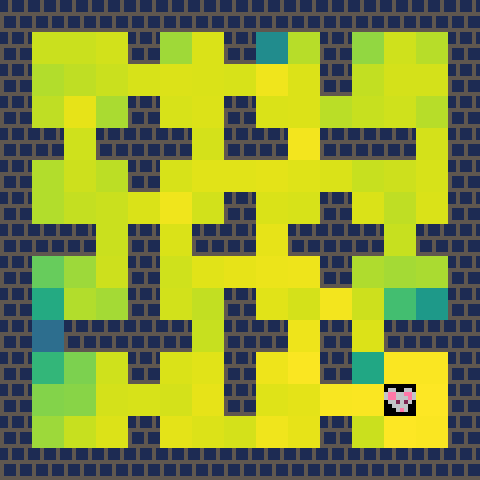}
    \\
    \raisebox{\halfheatmapwidth}{$1\text{e-}2$}
    & \includegraphics[width=\heatmapwidth]{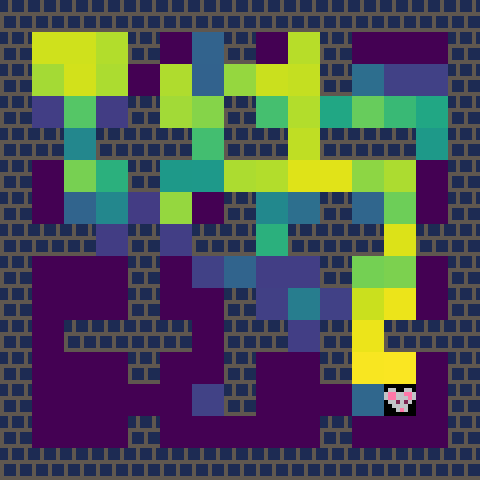}
    & \includegraphics[width=\heatmapwidth]{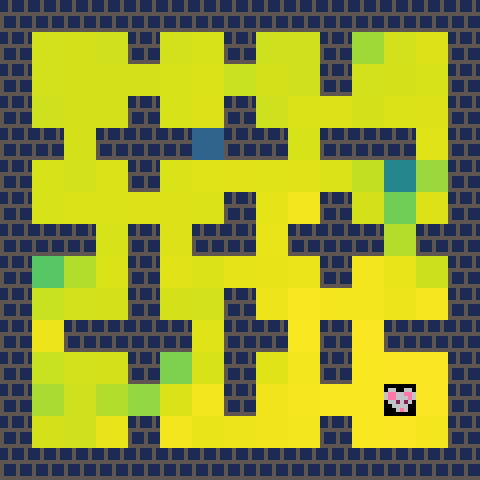}
    & \includegraphics[width=\heatmapwidth]{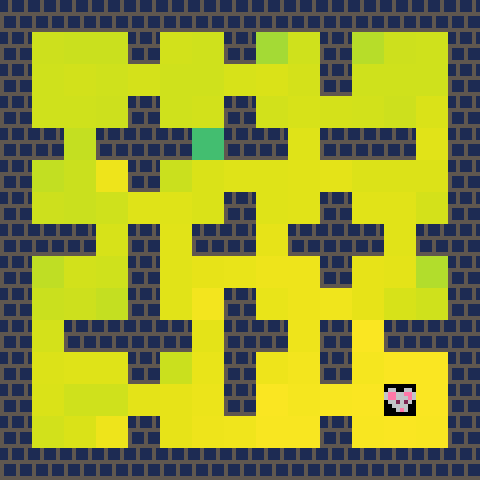}
    & \includegraphics[width=\heatmapwidth]{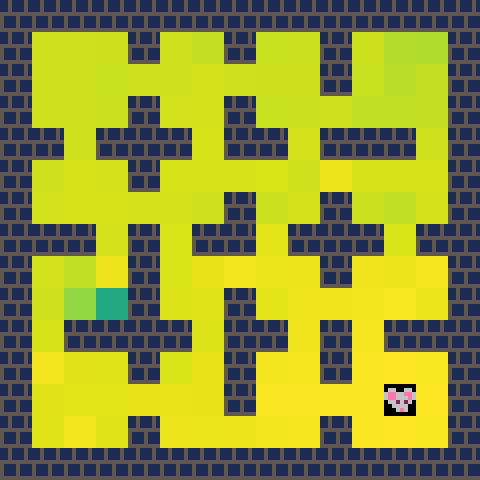}
    & \includegraphics[width=\heatmapwidth]{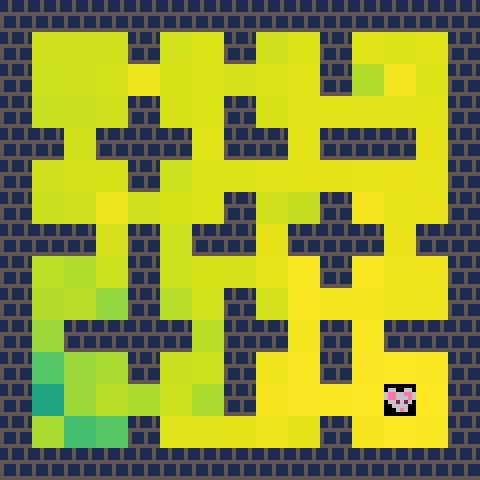}
    \\
    \midrule
    & \multicolumn{5}{c}{\includegraphics[trim={0pt 5pt 0pt 0pt}]{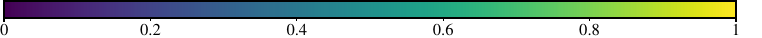}}
    \\
    & \multicolumn{5}{c}{Average return}
    \\
    \bottomrule
    \end{tabular}
    \endgroup
    \caption{\label{fig:heatmaps}%
        \textbf{Performance on \env{cheese in the corner} levels with varying cheese position.}
        For each training configuration, we evaluate the trained policy (first of 8 seeds) on a batch of 122 levels with shared wall layout and mouse spawn position but different cheese positions. We indicate average return on levels with different cheese positions by the color of the corresponding grid square.
        We see a progression whereby for more advanced algorithms or higher $\alpha$, the agent is robust to a greater proportion of cheese positions.
        See \cref{apx:heatmaps} for more details and full range of $\alpha$ values.
    }
\end{figure}

\subsection{Increasingly advanced UED methods are more robust to goal misgeneralization}
\label{sec:ued-advances}

\Cref{thm:mmer-robustness} says that MMER-based training should be robust to goal misgeneralization regardless of the distribution shift.
In contrast, in our experiments, the proportion of distinguishing levels played by the adversary decreases as we decrease $\alpha$ (\cref{fig:main-effective-alpha}), and each UED method exhibits a threshold below which it fails to converge to an MMER policy, and exhibits goal misgeneralization (\cref{fig:main-distinguishing-return}).

This performance trend reflects how the adversaries  construct level distributions.
When distinguishing levels are very rare (or never arise), the adversary is hindered (prevented) from increasing the number of distinguishing levels in the buffer.
Compared to \RobustPLR{}, ACCEL can replicate similar levels throughout its buffer through edits, but we used edits that don't create new distinguishing levels.
\Cref{apx:mutators} investigates ACCEL variants with different edit distributions, showing that ACCEL can prevent goal misgeneralization even when $\alpha=0$ if edits can introduce distinguishing levels.

Overall, robustness correlates with how advanced the adversaries are. 
Our most flexible adversary (ACCEL) paired with our most powerful regret estimator (oracle-latest) is remarkably robust to goal misgeneralization in all environments for all positive $\alpha$ tested.
The less flexible \RobustPLR{} using the less powerful max-latest expected regret estimator is the least robust.
This motivates work pursuing regret-based UED methods with better convergence properties \citep[cf.][]{monette2025an}.

\subsection{Biased regret estimation can undermine UED in more complex environments}
\label{sec:outlier}

The poor performance of ACCEL with max-latest the estimator in the \env{keys and chests} environment underscores the challenge of expected regret estimation.
Estimating maximum return from samples is particularly challenging in this environment, where high return is unlikely to be achieved in distinguishing levels by chance, since chests are substantially rarer than in non-distinguishing levels.
It appears that the increased flexibility of ACCEL in this case works as a disadvantage, leading to the adversary being led astray by biased regret estimates even more so than \RobustPLR{}.

The challenges of regret estimation are known, and are an active area of research \citep[cf.][]{Rutherford+2024}.
Our results highlight the importance of future work on reliable regret estimation methods, towards achieving the improved performance shown by our domain-specific oracle-latest estimator.
Such work could investigate using a separate policy network to estimate the maximum return \citep[cf.][]{Dennis+2020}, or incorporating the predictions of a value network \citep[cf.][]{Jiang+2022}.

\section{Conclusion}

In this paper, we introduce the setting of a proxy-distinguishing distribution shift, and offer a theoretical and empirical investigation of the robustness of MEV-based and MMER-based training to goal misgeneralization.
We show theoretically and empirically that MEV-based training on a fixed training distribution can lead to goal misgeneralization.
In contrast, we show that MMER-based training is provably robust against goal misgeneralization in the limit of idealized training methods, and regret-based unsupervised environment design (UED) methods are empirically more robust than MEV-based training.
Current UED methods do not find MMER policies and prevent goal misgeneralization in all the cases we studied, indicating there is still room for improvement between current methods and the theoretical ideal.
These findings highlight MMER-based training as a promising approach to preventing goal misgeneralization.

\clearpage
\appendix

\section*{Broader impact statement}
\addcontentsline{toc}{section}{Broader impact statement}
\label{sec:impact}

\citet{Ngo+2023} cast goal misgeneralization as a key risk mechanism for advanced deep learning systems, noting that techniques that improve capability robustness without preventing goal misgeneralization could \emph{worsen} outcomes, since the system's greater capabilities would then be devoted to the pursuit of an incorrect goal.
Preventing this dangerous mode of generalization failure is a key challenge in assuring the safety of advanced RL agents.

In this section, we briefly note that minimax expected regret appears to be well-suited in principle to mitigating goal misgeneralization as deep learning systems become increasingly capable.
This is because more generally capable deep learning systems should also be more capable regret-maximizing adversaries in particular.
A more capable adversary will, in turn, be better at detecting or synthesizing rare, high-regret training situations, and then amplifying the training signal from these situations so as to induce correct generalization in an advanced deep RL agent (cf., \cref{apx:mutators}).

Our work highlights training with the minimax expected regret (MMER) objective as a promising avenue for preventing goal misgeneralization.
This objective has desirable theoretical properties, and we have found promising initial empirical results, though current MMER-based training techniques are not mature enough to prevent goal misgeneralization in all cases.
As MMER-based training methods improve and as goal misgeneralization leads to more severe consequences, the ability of MMER-based training to mitigate goal misgeneralization should also improve.

Ultimately, we are hopeful that our work will instigate further research on the problem of goal misgeneralization, which remains a critical, open problem in the alignment and safe generalization of future advanced reinforcement learning agents.

\section*{Acknowledgments}
\addcontentsline{toc}{section}{Acknowledgements}
\label{sec:acknowledgements}

We thank Jason Brown, Robert Kirk, and Lauro Langosco for helpful discussions.
We thank Stephen Chung and Samuel Coward for advice with training algorithms.
We thank
    Micah Carroll,
    Dmitrii Krasheninnikov,
    Niklas Lauffer,
    Michelle Li,
    Clare Lyle,
    Benjamin Plaut,
    Rohin Shah,
and
    our anonymous reviewers (especially Reviewer DDwf)
for helpful feedback on the research and the manuscript.
KAS was supported in part by the Cambridge ERA:AI Fellowship.

\bibliography{main}

\begin{thebibliography}{43}
\providecommand{\natexlab}[1]{#1}
\providecommand{\url}[1]{\texttt{#1}}
\expandafter\ifx\csname urlstyle\endcsname\relax
  \providecommand{\doi}[1]{DOI: #1}\else
  \providecommand{\doi}{DOI: \begingroup \urlstyle{rm}\Url}\fi

\bibitem[Barnett(2019)]{Barnett2019}
Matthew Barnett.
\newblock A simple environment for showing mesa misalignment.
\newblock Alignment Forum, September 2019.
\newblock URL \url{https://www.alignmentforum.org/posts/AFdRGfYDWQqmkdhFq}.

\bibitem[Beukman et~al.(2024)Beukman, Coward, Matthews, Fellows, Jiang, Dennis,
  and Foerster]{Beukman+2024}
Michael Beukman, Samuel Coward, Michael Matthews, Mattie Fellows, Minqi Jiang,
  Michael Dennis, and Jakob~N. Foerster.
\newblock Refining minimax regret for unsupervised environment design.
\newblock In \emph{Proceedings of the 41st International Conference on Machine
  Learning}, volume 235 of \emph{Proceedings of Machine Learning Research},
  pp.\  3637--3657. PMLR, 2024.

\bibitem[Bradbury et~al.(2018)Bradbury, Frostig, Hawkins, Johnson, Leary,
  Maclaurin, Necula, Paszke, {VanderPlas}, {Wanderman-Milne}, and Zhang]{JAX}
James Bradbury, Roy Frostig, Peter Hawkins, Matthew~James Johnson, Chris Leary,
  Dougal Maclaurin, George Necula, Adam Paszke, Jake {VanderPlas}, Skye
  {Wanderman-Milne}, and Qiao Zhang.
\newblock {JAX:} composable transformations of {Python}+{NumPy} programs, 2018.
\newblock URL \url{http://github.com/jax-ml/jax}.

\bibitem[Bruce et~al.(2024)Bruce, Dennis, Edwards, Parker-Holder, Shi, Hughes,
  Lai, Mavalankar, Steigerwald, Apps, Aytar, Bechtle, Behbahani, Chan, Heess,
  Gonzalez, Osindero, Ozair, Reed, Zhang, Zolna, Clune, Freitas, Singh, and
  Rockt\"{a}schel]{bruce2024genie}
Jake Bruce, Michael Dennis, Ashley Edwards, Jack Parker-Holder, Yuge Shi,
  Edward Hughes, Matthew Lai, Aditi Mavalankar, Richie Steigerwald, Chris Apps,
  Yusuf Aytar, Sarah Maria~Elisabeth Bechtle, Feryal Behbahani, Stephanie C.~Y.
  Chan, Nicolas Heess, Lucy Gonzalez, Simon Osindero, Sherjil Ozair, Scott
  Reed, Jingwei Zhang, Konrad Zolna, Jeff Clune, Nando~De Freitas, Satinder
  Singh, and Tim Rockt\"{a}schel.
\newblock Genie: Generative interactive environments.
\newblock In \emph{Proceedings of the 41st International Conference on Machine
  Learning}, volume 235 of \emph{Proceedings of Machine Learning Research},
  pp.\  4603--4623. PMLR, 2024.

\bibitem[Christiano(2018)]{Christiano2018}
Paul Christiano.
\newblock Techniques for optimizing worst-case performance.
\newblock AI Alignment (Blog), February 2018.
\newblock URL
  \url{https://ai-alignment.com/techniques-for-optimizing-worst-case-performance-39eafec74b99}.

\bibitem[Cobbe et~al.(2019)Cobbe, Klimov, Hesse, Kim, and Schulman]{Cobbe+2019}
Karl Cobbe, Oleg Klimov, Chris Hesse, Taehoon Kim, and John Schulman.
\newblock Quantifying generalization in reinforcement learning.
\newblock In \emph{Proceedings of the 36th International Conference on Machine
  Learning}, volume~97 of \emph{Proceedings of Machine Learning Research}, pp.\
   1282--1289. PMLR, 2019.

\bibitem[Cobbe et~al.(2020)Cobbe, Hesse, Hilton, and Schulman]{Cobbe+2020}
Karl Cobbe, Chris Hesse, Jacob Hilton, and John Schulman.
\newblock Leveraging procedural generation to benchmark reinforcement learning.
\newblock In \emph{Proceedings of the 37th International Conference on Machine
  Learning}, volume 119 of \emph{Proceedings of Machine Learning Research},
  pp.\  2048--2056. PMLR, 2020.

\bibitem[Dennis et~al.(2020)Dennis, Jaques, Vinitsky, Bayen, Russell, Critch,
  and Levine]{Dennis+2020}
Michael Dennis, Natasha Jaques, Eugene Vinitsky, Alexandre Bayen, Stuart~J.
  Russell, Andrew Critch, and Sergey Levine.
\newblock Emergent complexity and zero-shot transfer via unsupervised
  environment design.
\newblock In \emph{Advances in Neural Information Processing Systems 33}, pp.\
  13049--13061. Curran Associates, Inc., 2020.

\bibitem[Espeholt et~al.(2018)Espeholt, Soyer, Munos, Simonyan, Mnih, Ward,
  Doron, Firoiu, Harley, Dunning, Legg, and Kavukcuoglu]{Espeholt+2018}
Lasse Espeholt, Hubert Soyer, Remi Munos, Karen Simonyan, Vlad Mnih, Tom Ward,
  Yotam Doron, Vlad Firoiu, Tim Harley, Iain Dunning, Shane Legg, and Koray
  Kavukcuoglu.
\newblock {IMPALA}: Scalable distributed deep-{RL} with importance weighted
  actor-learner architectures.
\newblock In \emph{Proceedings of the 35th International Conference on Machine
  Learning}, volume~80 of \emph{Proceedings of Machine Learning Research}, pp.\
   1407--1416. PMLR, 2018.

\bibitem[Geirhos et~al.(2020)Geirhos, Jacobsen, Michaelis, Zemel, Brendel,
  Bethge, and Wichmann]{Geirhos+2020}
Robert Geirhos, J{\"o}rn-Henrik Jacobsen, Claudio Michaelis, Richard Zemel,
  Wieland Brendel, Matthias Bethge, and Felix~A. Wichmann.
\newblock Shortcut learning in deep neural networks.
\newblock \emph{Nature Machine Intelligence}, 2:\penalty0 665--673, 2020.

\bibitem[Ha \& Schmidhuber(2018)Ha and Schmidhuber]{ha2018world}
David Ha and J\"{u}rgen Schmidhuber.
\newblock Recurrent world models facilitate policy evolution.
\newblock In \emph{Advances in Neural Information Processing Systems 31}, pp.\
  2450--2462. Curran Associates, Inc., 2018.

\bibitem[Hadfield-Menell et~al.(2017)Hadfield-Menell, Milli, Abbeel, Russell,
  and Dragan]{HadfieldMenell+2017}
Dylan Hadfield-Menell, Smitha Milli, Pieter Abbeel, Stuart~J. Russell, and Anca
  Dragan.
\newblock Inverse reward design.
\newblock In \emph{Advances in Neural Information Processing Systems 30}, pp.\
  6765--6774. Curran Associates, Inc., 2017.

\bibitem[Hafner et~al.(2020)Hafner, Lillicrap, Ba, and
  Norouzi]{hafner2019dream}
Danijar Hafner, Timothy Lillicrap, Jimmy Ba, and Mohammad Norouzi.
\newblock Dream to control: Learning behaviors by latent imagination.
\newblock In \emph{8th International Conference on Learning Representations}.
  OpenReview, 2020.

\bibitem[Hafner et~al.(2025)Hafner, Pasukonis, Ba, and
  Lillicrap]{hafner2023mastering}
Danijar Hafner, Jurgis Pasukonis, Jimmy Ba, and Timothy Lillicrap.
\newblock Mastering diverse control tasks through world models.
\newblock \emph{Nature}, 640\penalty0 (8059):\penalty0 647--653, 2025.

\bibitem[Hubinger(2019)]{Hubinger2019}
Evan Hubinger.
\newblock Towards an empirical investigation of inner alignment.
\newblock Alignment Forum, September 2019.
\newblock URL \url{https://www.alignmentforum.org/posts/2GycxikGnepJbxfHT}.

\bibitem[Jiang et~al.(2021{\natexlab{a}})Jiang, Dennis, Parker-Holder,
  Foerster, Grefenstette, and Rockt\"{a}schel]{Jiang+2022}
Minqi Jiang, Michael Dennis, Jack Parker-Holder, Jakob~N. Foerster, Edward
  Grefenstette, and Tim Rockt\"{a}schel.
\newblock Replay-guided adversarial environment design.
\newblock In \emph{Advances in Neural Information Processing Systems 34}, pp.\
  1884--1897. Curran Associates, Inc., 2021{\natexlab{a}}.

\bibitem[Jiang et~al.(2021{\natexlab{b}})Jiang, Grefenstette, and
  Rockt{\"a}schel]{Jiang+2021}
Minqi Jiang, Edward Grefenstette, and Tim Rockt{\"a}schel.
\newblock Prioritized level replay.
\newblock In \emph{Proceedings of the 38th International Conference on Machine
  Learning}, volume 139 of \emph{Proceedings of Machine Learning Research},
  pp.\  4940--4950. PMLR, 2021{\natexlab{b}}.

\bibitem[Kirk et~al.(2023)Kirk, Zhang, Grefenstette, and
  Rockt{\"a}schel]{Kirk+2023}
Robert Kirk, Amy Zhang, Edward Grefenstette, and Tim Rockt{\"a}schel.
\newblock A survey of zero-shot generalisation in deep reinforcement learning.
\newblock \emph{Journal of Artificial Intelligence Research}, 76:\penalty0
  201--264, 2023.

\bibitem[Kumar et~al.(2021)Kumar, Fu, Pathak, and Malik]{kumar2021rma}
Ashish Kumar, Zipeng Fu, Deepak Pathak, and Jitendra Malik.
\newblock {RMA}: Rapid motor adaptation for legged robots.
\newblock In \emph{Proceedings of Robotics: Science and Systems XVII}, 2021.

\bibitem[Langosco et~al.(2022)Langosco, Koch, Sharkey, Pfau, and
  Krueger]{Langosco+2022}
Lauro Langosco, Jack Koch, Lee~D. Sharkey, Jacob Pfau, and David Krueger.
\newblock Goal misgeneralization in deep reinforcement learning.
\newblock In \emph{Proceedings of the 39th International Conference on Machine
  Learning}, volume 162 of \emph{Proceedings of Machine Learning Research},
  pp.\  12004--12019. PMLR, 2022.

\bibitem[Liu et~al.(2021)Liu, Haghgoo, Chen, Raghunathan, Koh, Sagawa, Liang,
  and Finn]{liu2021just}
Evan~Z. Liu, Behzad Haghgoo, Annie~S. Chen, Aditi Raghunathan, Pang~Wei Koh,
  Shiori Sagawa, Percy Liang, and Chelsea Finn.
\newblock Just train twice: Improving group robustness without training group
  information.
\newblock In \emph{Proceedings of the 38th International Conference on Machine
  Learning}, volume 139 of \emph{Proceedings of Machine Learning Research},
  pp.\  6781--6792. PMLR, 2021.

\bibitem[Ma et~al.(2024)Ma, Liang, Wang, Wang, Zhu, Fan, Bastani, and
  Jayaraman]{ma2024dreureka}
Yecheng~Jason Ma, William Liang, Hung-Ju Wang, Sam Wang, Yuke Zhu, Linxi Fan,
  Osbert Bastani, and Dinesh Jayaraman.
\newblock {DrEureka:} language model guided sim-to-real transfer.
\newblock In \emph{Proceedings of Robotics: Science and Systems XX}, 2024.

\bibitem[Makoviychuk et~al.(2021)Makoviychuk, Wawrzyniak, Guo, Lu, Storey,
  Macklin, Hoeller, Rudin, Allshire, Handa, and State]{makoviychuk2021isaac}
Viktor Makoviychuk, Lukasz Wawrzyniak, Yunrong Guo, Michelle Lu, Kier Storey,
  Miles Macklin, David Hoeller, Nikita Rudin, Arthur Allshire, Ankur Handa, and
  Gavriel State.
\newblock Isaac gym: High performance {GPU} based physics simulation for robot
  learning.
\newblock In \emph{Proceedings of the Neural Information Processing Systems
  Track on Datasets and Benchmarks}, volume~1, 2021.

\bibitem[Monette et~al.(2025)Monette, Letcher, Beukman, Jackson, Rutherford,
  Goldie, and Foerster]{monette2025an}
Nathan Monette, Alistair Letcher, Michael Beukman, Matthew~Thomas Jackson,
  Alexander Rutherford, Alexander~David Goldie, and Jakob~N. Foerster.
\newblock An optimisation framework for unsupervised environment design.
\newblock \emph{Reinforcement Learning Journal}, 2025.
\newblock To appear.

\bibitem[Muratore et~al.(2022)Muratore, Ramos, Turk, Yu, Gienger, and
  Peters]{muratore2022robot}
Fabio Muratore, Fabio Ramos, Greg Turk, Wenhao Yu, Michael Gienger, and Jan
  Peters.
\newblock Robot learning from randomized simulations: A review.
\newblock \emph{Frontiers in Robotics and AI}, 9:\penalty0 799893, 2022.

\bibitem[Ngo et~al.(2024)Ngo, Chan, and Mindermann]{Ngo+2023}
Richard Ngo, Lawrence Chan, and S\"oren Mindermann.
\newblock The alignment problem from a deep learning perspective.
\newblock In \emph{12th International Conference on Learning Representations}.
  OpenReview, 2024.

\bibitem[OpenAI et~al.(2019)OpenAI, Akkaya, Andrychowicz, Chociej, Litwin,
  McGrew, Petron, Paino, Plappert, Powell, Ribas, Schneider, Tezak, Tworek,
  Welinder, Weng, Yuan, Zaremba, and Zhang]{OpenAI+2019}
OpenAI, Ilge Akkaya, Marcin Andrychowicz, Maciek Chociej, Mateusz Litwin, Bob
  McGrew, Arthur Petron, Alex Paino, Matthias Plappert, Glenn Powell, Raphael
  Ribas, Jonas Schneider, Nikolas Tezak, Jerry Tworek, Peter Welinder, Lilian
  Weng, Qiming Yuan, Wojciech Zaremba, and Lei Zhang.
\newblock Solving {Rubik's Cube} with a robot hand.
\newblock Preprint arXiv:1910.07113 [cs.LG], 2019.

\bibitem[Parker-Holder et~al.(2022)Parker-Holder, Jiang, Dennis, Samvelyan,
  Foerster, Grefenstette, and Rockt{\"a}schel]{ParkerHolder+2023}
Jack Parker-Holder, Minqi Jiang, Michael Dennis, Mikayel Samvelyan, Jakob~N.
  Foerster, Edward Grefenstette, and Tim Rockt{\"a}schel.
\newblock Evolving curricula with regret-based environment design.
\newblock Preprint arXiv:2203.01302 [cs.LG], 2022.

\bibitem[Peng et~al.(2018)Peng, Andrychowicz, Zaremba, and Abbeel]{peng2018sim}
Xue~Bin Peng, Marcin Andrychowicz, Wojciech Zaremba, and Pieter Abbeel.
\newblock Sim-to-real transfer of robotic control with dynamics randomization.
\newblock In \emph{2018 IEEE International Conference on Robotics and
  Automation}, pp.\  3803--3810. IEEE, 2018.

\bibitem[Rutherford et~al.(2024)Rutherford, Beukman, Willi, Lacerda, Hawes, and
  Foerster]{Rutherford+2024}
Alexander Rutherford, Michael Beukman, Timon Willi, Bruno Lacerda, Nick Hawes,
  and Jakob Foerster.
\newblock No regrets: Investigating and improving regret approximations for
  curriculum discovery.
\newblock In \emph{Advances in Neural Information Processing Systems 37}, pp.\
  16071--16101. Curran Associates, Inc., 2024.

\bibitem[Savage(1951)]{Savage1951}
Leonard~J. Savage.
\newblock The theory of statistical decision.
\newblock \emph{Journal of the American Statistical Association}, 46\penalty0
  (253):\penalty0 55--67, 1951.

\bibitem[Schrittwieser et~al.(2020)Schrittwieser, Antonoglou, Hubert, Simonyan,
  Sifre, Schmitt, Guez, Lockhart, Hassabis, Graepel, Lillicrap, and
  Silver]{schrittwieser2020mastering}
Julian Schrittwieser, Ioannis Antonoglou, Thomas Hubert, Karen Simonyan,
  Laurent Sifre, Simon Schmitt, Arthur Guez, Edward Lockhart, Demis Hassabis,
  Thore Graepel, Timothy Lillicrap, and David Silver.
\newblock Mastering {Atari,} {Go,} chess and shogi by planning with a learned
  model.
\newblock \emph{Nature}, 588\penalty0 (7839):\penalty0 604--609, 2020.

\bibitem[Schulman et~al.(2015)Schulman, Moritz, Levine, Jordan, and
  Abbeel]{schulman2015high}
John Schulman, Philipp Moritz, Sergey Levine, Michael Jordan, and Pieter
  Abbeel.
\newblock High-dimensional continuous control using generalized advantage
  estimation.
\newblock Preprint arXiv:1506.02438 [cs.LG], 2015.

\bibitem[Schulman et~al.(2017)Schulman, Wolski, Dhariwal, Radford, and
  Klimov]{schulman2017proximal}
John Schulman, Filip Wolski, Prafulla Dhariwal, Alec Radford, and Oleg Klimov.
\newblock Proximal policy optimization algorithms.
\newblock Preprint arXiv:1707.06347 [cs.LG], 2017.

\bibitem[Shah et~al.(2022)Shah, Varma, Kumar, Phuong, Krakovna, Uesato, and
  Kenton]{Shah+2022}
Rohin Shah, Vikrant Varma, Ramana Kumar, Mary Phuong, Victoria Krakovna,
  Jonathan Uesato, and Zac Kenton.
\newblock Goal misgeneralization: Why correct specifications aren't enough for
  correct goals.
\newblock Preprint arXiv:2210.01790 [cs.LG], 2022.

\bibitem[Starace(2023)]{starace2023addressing}
Giulio Starace.
\newblock Addressing goal misgeneralization with natural language interfaces.
\newblock Master's thesis, University of Amsterdam, 2023.

\bibitem[Suau et~al.(2024)Suau, Spaan, and Oliehoek]{suau2024badhabits}
Miguel Suau, Matthijs T.~J. Spaan, and Frans~A. Oliehoek.
\newblock Bad habits: Policy confounding and out-of-trajectory generalization
  in reinforcement learning.
\newblock \emph{Reinforcement Learning Journal}, 4:\penalty0 1711--1732, 2024.

\bibitem[Tobin et~al.(2017)Tobin, Fong, Ray, Schneider, Zaremba, and
  Abbeel]{Tobin+2017}
Josh Tobin, Rachel Fong, Alex Ray, Jonas Schneider, Wojciech Zaremba, and
  Pieter Abbeel.
\newblock Domain randomization for transferring deep neural networks from
  simulation to the real world.
\newblock In \emph{2017 IEEE/RSJ International Conference on Intelligent Robots
  and Systems}, 2017.

\bibitem[Trinh et~al.(2024)Trinh, Danesh, Khanh, and Plaut]{trinh2024getting}
Tu~Trinh, Mohamad~H. Danesh, Nguyen~X. Khanh, and Benjamin Plaut.
\newblock Getting by goal misgeneralization with a little help from a mentor.
\newblock In \emph{The First Workshop on Safe \& Trustworthy Agents}, 2024.
\newblock Workshop at NeurIPS 2024. Preprint arXiv:2410.21052 [cs.LG].

\bibitem[Valevski et~al.(2025)Valevski, Leviathan, Arar, and
  Fruchter]{valevski2024diffusion}
Dani Valevski, Yaniv Leviathan, Moab Arar, and Shlomi Fruchter.
\newblock Diffusion models are real-time game engines.
\newblock In \emph{13th International Conference on Learning Representations}.
  OpenReview, 2025.

\bibitem[Wang et~al.(2023)Wang, Si, Blanchet, and Zhou]{Wang+2023}
Shengbo Wang, Nian Si, Jose Blanchet, and Zhengyuan Zhou.
\newblock On the foundation of distributionally robust reinforcement learning.
\newblock Preprint arxiv:2311.09018 [cs.LG], 2023.

\bibitem[Zhang et~al.(2022)Zhang, Sohoni, Zhang, Finn, and
  R{\'e}]{zhang2022correct}
Michael Zhang, Nimit~S. Sohoni, Hongyang~R. Zhang, Chelsea Finn, and
  Christopher R{\'e}.
\newblock {Correct-N-Contrast:} a contrastive approach for improving robustness
  to spurious correlations.
\newblock In \emph{Proceedings of the 39th International Conference on Machine
  Learning}, volume 162 of \emph{Proceedings of Machine Learning Research},
  pp.\  26484--26516. PMLR, 2022.

\bibitem[Zhuang \& Hadfield-Menell(2020)Zhuang and
  Hadfield-Menell]{zhuang2020consequences}
Simon Zhuang and Dylan Hadfield-Menell.
\newblock Consequences of misaligned {AI}.
\newblock In \emph{Advances in Neural Information Processing Systems 33}, pp.\
  15763--15773. Curran Associates, Inc., 2020.

\end{thebibliography}
\bibliographystyle{rlj}
\addcontentsline{toc}{section}{References}

\clearpage
\counterwithin{figure}{section}
\counterwithin{table}{section}

\section*{Appendix contents}
\addcontentsline{toc}{section}{Supplementary materials}
\startcontents[appendix]
\renewcommand{\baselinestretch}{0.85}\normalsize
\printcontents[appendix]{}{1}{}
\renewcommand{\baselinestretch}{1.0}\normalsize

\clearpage

\section{Proofs for theoretical results from \texorpdfstring{\cref{sec:theory}}{Section~\ref*{sec:theory}}}
\label{sec:main-proofs}

In this section, we restate and prove \cref{thm:mev-susceptibility,thm:mmer-robustness}.

\MEVTheorem*

\begin{proof}[Proof of \cref{thm:mev-susceptibility}]
\hypertarget{proof:thm:mev-susceptibility}
Assume $\alpha \leq \eps$. Construct a policy
    $\MEVPolicy$
such that, for all levels $\level \in \LevelSpace$,
\begin{align*}
    \MEVPolicy
    &\in
    \argmax_{\Policy \in \AllPolicies} \Return{\ProxyRewardFunction}{\level}{\Policy}
    \setminus
    \xargmax[C]_{\Policy \in \AllPolicies}  \Return{\RewardFunction}{\level}{\Policy}
    && \text{if $\theta$ is $C$-distinguishing, or}
\\
    \MEVPolicy
    &\in
    \argmax_{\Policy \in \AllPolicies} \Return{\ProxyRewardFunction}{\level}{\Policy}
    && \text{otherwise.}
\end{align*}
The former is non-empty by \cref{def:distinguishing}.
By construction,
    $\MEVPolicy \in \argmax_{\Policy \in \AllPolicies}\Return{\ProxyRewardFunction}{\distrTest}{\Policy}$.
It remains to prove
    (i)~$\MEVPolicy \in \MEVOptimalPolicies{\RewardFunction}{\distrTrain}$
and (ii)~$\MEVPolicy \notin \xargmax[\beta C]_{\Policy \in \AllPolicies} \Return{\RewardFunction}{\distrTest}{\Policy}$.

For notational convenience, construct a policy
    $\OptimalPolicy \in \AllPolicies$
that is optimal under $\RewardFunction$ in all levels.
Moreover, let $\distrTrainDistg, \distrTrainNonDistg, \distrTestDistg, \distrTestNonDistg \in \LevelDistrs$ be $\distrTrain$ and $\distrTest$ conditioned on the level being $C$-distinguishing or non-distinguishing, respectively.

Then for condition~(i), we have
\begin{align*}
    \Return{\RewardFunction}{\distrTrain}{\MEVPolicy}
  & = \alpha \Return{\RewardFunction}{\distrTrainDistg}{\MEVPolicy}
      + (1-\alpha) \Return{\RewardFunction}{\distrTrainNonDistg}{\MEVPolicy}
\tag{by \cref{def:proxy-distinguishing-distribution-shift}}
\\& = \alpha \Return{\RewardFunction}{\distrTrainDistg}{\MEVPolicy}
      + (1-\alpha) \Return{\RewardFunction}{\distrTrainNonDistg}{\OptimalPolicy}
\tag{by \cref{def:non-distinguishing}}
\\& \geq \alpha \cdot 0
      + (1-\alpha) \Return{\RewardFunction}{\distrTrainNonDistg}{\OptimalPolicy}
\tag{since $\ReturnSymbol^\RewardFunction \geq 0$}
\\& = \Return{\RewardFunction}{\distrTrain}{\OptimalPolicy}
      - \alpha \Return{\RewardFunction}{\distrTrainDistg}{\OptimalPolicy}
\tag{by \cref{def:proxy-distinguishing-distribution-shift}}
\\& \geq \Return{\RewardFunction}{\distrTrain}{\OptimalPolicy}
      - \eps \cdot 1.
\tag{since $\alpha \leq \eps$; $\ReturnSymbol^\RewardFunction \leq 1$}
\end{align*}

For condition~(ii), we have
\begin{align*}
    \Return{\RewardFunction}{\distrTest}{\MEVPolicy}
  & = \beta \Return{\RewardFunction}{\distrTestDistg}{\MEVPolicy}
      + (1-\beta) \Return{\RewardFunction}{\distrTestNonDistg}{\MEVPolicy}
\tag{by \cref{def:proxy-distinguishing-distribution-shift}}
\\& = \beta \Return{\RewardFunction}{\distrTestDistg}{\MEVPolicy}
      + (1-\beta) \Return{\RewardFunction}{\distrTestNonDistg}{\OptimalPolicy}
\tag{by \cref{def:non-distinguishing}}
\\& < \beta \big(
        \Return{\RewardFunction}{\distrTestDistg}{\OptimalPolicy} - C
      \big)
      + (1-\beta) \Return{\RewardFunction}{\distrTestNonDistg}{\OptimalPolicy}
\tag{$\displaystyle\MEVPolicy \notin \xargmax[C]_{\Policy \in \AllPolicies} \Return{\RewardFunction}{\level}{\Policy}$}
\\& = \Return{\RewardFunction}{\distrTest}{\OptimalPolicy}
      - \beta C.
\tag*{(by \cref{def:proxy-distinguishing-distribution-shift}) \qedhere}
\end{align*}
\end{proof}

\clearpage

\MMERTheorem*

\begin{proof}[Proof of \cref{thm:mmer-robustness}]
\hypertarget{proof:thm:mmer-robustness}
Suppose \(\MMERPolicy \in \MMEROptimalPolicies[\eps]{\RewardFunction}\).
Then we have the following bound on expected regret:
\begin{align*}
    \Regret{\RewardFunction}{\distrTest}{\MMERPolicy}
    & \leq \max_{\distr \in \LevelDistrs} \Regret{\RewardFunction}{\distr}{\MMERPolicy}
\tag{$\distrTest \in \LevelDistrs$}
\\  
    & \leq \min_{\Policy\in\AllPolicies}
        \max_{\distr \in \LevelDistrs}
            \Regret{\RewardFunction}{\distr}{\Policy}
        + \eps.
\tag{by \cref{def:approx-mmer}}
\end{align*}
We can convert this upper bound on expected regret to a lower bound on expected return:
\begin{align*}
    \Return{\RewardFunction}{\distrTest}{\MMERPolicy}
    &=
    \max_{\Policy\in\AllPolicies}
        \Return{\RewardFunction}{\distrTest}{\Policy}
    -
    \Regret{\RewardFunction}{\distrTest}{\MMERPolicy}
\tag{by equation~\ref{eq:distr-regret}}
\\  &\geq
    \max_{\Policy\in\AllPolicies}
        \Return{\RewardFunction}{\distrTest}{\Policy}
    - \min_{\Policy\in\AllPolicies}
        \max_{\distr \in \LevelDistrs}
            \Regret{\RewardFunction}{\distr}{\Policy}
    - \eps.
\tag{by above bound}
\end{align*}
The theorem follows, since, for all $\distr \in \LevelDistrs$,
    $\min_{\Policy\in\AllPolicies} \Regret{\RewardFunction}{\distr}{\Policy}$
vanishes by equation~\cref{eq:distr-regret}:
\begin{align*}
    \min_{\Policy\in\AllPolicies}
        \Regret{\RewardFunction}{\distr}{\Policy}
    &=
    \min_{\Policy\in\AllPolicies}\left(
        \max_{\Policy'\in\AllPolicies}
            \Return{\RewardFunction}{\distr}{\Policy'}
        - \Return{\RewardFunction}{\distr}{\Policy}
    \right)
\tag{by equation~\ref{eq:distr-regret}}
\\  &=
    \max_{\Policy'\in\AllPolicies}
        \Return{\RewardFunction}{\distr}{\Policy'}
    -
    \max_{\Policy\in\AllPolicies}
        \Return{\RewardFunction}{\distr}{\Policy}
\tag{$\max$ term is constant wrt.\ $\Policy$}
\\ &= 0.
\tag*{\qedhere}
\end{align*}
\end{proof}

\clearpage

\section{Expected regret identity for UMDPs}
\label{apx:regret-identity}

In this section, we prove equation~\cref{eq:distr-regret} for UMDPs.
Recall the following definitions from \cref{sec:preliminaries}.
\begin{align}
    \Return{\RewardFunction}{\level}{\Policy}
    &= 
    \Expect[
        s_0 \sim \InitialStateDistribution(\level),
        a_t \sim \Policy(\level, s_t),
        s_{t+1} \sim \TransitionMap(\level, s_t, a_t)
    ]{
        \sum_{t=0}^\infty \DiscountRate^t \RewardFunction(s_t, a_t, s_{t+1})
    }
\label{xdef:level-return}
\\  
    \Return{\RewardFunction}{\distr}{\Policy}
    &=
    \Expect[\theta \sim \distr]{\Return{\RewardFunction}{\level}{\Policy}}
\label{xdef:distr-return}
\\
    \Regret{\RewardFunction}{\level}{\Policy}
    &=
    \max_{\Policy' \in \AllPolicies} \Return{\RewardFunction}{\level}{\Policy'} - \Return{\RewardFunction}{\level}{\Policy}
\label{xdef:level-regret}
\\
    \Regret{\RewardFunction}{\distr}{\Policy}
    &=
    \Expect[\theta \sim \distr]{\Regret{\RewardFunction}{\level}{\Policy}}
\label{xdef:distr-regret}
\end{align}
In \cref{sec:preliminaries}, we observe that, for UMDPs, we have the additional basic identity
\[
    \Regret{\RewardFunction}{\distr}{\Policy}
    =
    \max_{\Policy' \in \AllPolicies}
        \Return{\RewardFunction}{\distr}{\Policy'}
    -
    \Return{\RewardFunction}{\distr}{\Policy}.
    \tag{this is equation~\ref{eq:distr-regret}}
\]
This is a nontrivial identity that does not hold for partially observable underspecified environments in which the level is not observable to the policy (see \cref{apx:partial-regret-identity}).
However, for policies that are conditioned on the level, the identity holds, as we now prove.

\begin{apxproposition}[Expected regret identity for UMDPs]\label{prop:identity}
Consider an UMDP $\TupleUMDP$, a goal $\RewardFunction$, and a level distribution $\distr \in \LevelDistrs$.
Let $\AllPolicies$ be the set of all policies of the form
    $\Policy : \LevelSpace \times \States \to \Distr{\Actions}$.
Then we have
\[
    \Regret{\RewardFunction}{\distr}{\Policy}
    =
    \max_{\Policy' \in \AllPolicies}
        \Return{\RewardFunction}{\distr}{\Policy'}
    - \Return{\RewardFunction}{\distr}{\Policy}
.\]
\end{apxproposition}

\begin{proof}
\leavevmode\vspace{-\baselineskip-\abovedisplayskip}
\begin{align*}
    \Regret{\RewardFunction}{\distr}{\Policy}
    &=
    \Expect[\theta \sim \distr]{
        \Regret{\RewardFunction}{\level}{\Policy}
    }
\tag{equation~\ref{xdef:distr-regret}}
\\  &=
    \Expect[\theta \sim \distr]{
        \max_{\Policy' \in \AllPolicies}
            \Return{\RewardFunction}{\level}{\Policy'}
        - \Return{\RewardFunction}{\level}{\Policy}
    }
\tag{by equation~\ref{xdef:level-regret}}
\\  &=
    \max_{\Policy' \in \AllPolicies}
        \Expect[\theta \sim \distr]{
            \Return{\RewardFunction}{\level}{\Policy'}
        }
    -
    \Expect[\theta \sim \distr]{
        \Return{\RewardFunction}{\level}{\Policy}
    }
\tag{by \cref{prop:exchange}, below}
\\  &=
    \max_{\Policy' \in \AllPolicies}
        \Return{\RewardFunction}{\distr}{\Policy'}
    - \Return{\RewardFunction}{\distr}{\Policy}.
\tag*{(by equation~\ref{xdef:distr-return}) \qedhere}
\end{align*}
\end{proof}

The above proof relies on \cref{prop:exchange}, which says that we can exchange expectation and maximization for the expected return since the policy is conditioned on the level.

\begin{apxproposition}[Expectation and maximization of expected return commute for UMDPs]
\label{prop:exchange}
Consider an UMDP $\TupleUMDP$, a goal $\RewardFunction$, and a level distribution $\distr \in \LevelDistrs$.
Let $\AllPolicies$ be the set of all policies of the form
    $\Policy : \LevelSpace \times \States \to \Distr{\Actions}$.
Then we have
\[
    \Expect[\level \sim \distr]{
        \max_{\Policy \in \AllPolicies}
            \Return{\RewardFunction}{\level}{\Policy}
    }
    =
    \max_{\Policy \in \AllPolicies}
        \Expect[\level \sim \distr]{
            \Return{\RewardFunction}{\level}{\Policy}
        }
.\]
\end{apxproposition}

\begin{proof}
\textit{($\geq$):} Note that this direction holds regardless of whether we condition policies on the level. Let
$
    \OptimalPolicy
    \in
    \argmax_{\Policy\in\AllPolicies}
        \Expect[\level \sim \distr]{
            \Return{\RewardFunction}{\level}{\Policy}
        }
$.
Then we have 
\begin{equation*}
    \max_{\Policy\in\AllPolicies}
        \Expect[\level \sim \distr]{
            \Return{\RewardFunction}{\level}{\Policy}
        }
    =
    \Expect[\level \sim \distr]{
        \Return{\RewardFunction}{\level}{\OptimalPolicy}
    }
    \leq
    \Expect[\level \sim \distr]{
        \max_{\Policy\in\AllPolicies}
            \Return{\RewardFunction}{\level}{\Policy}
    }.
\end{equation*}

\textit{($\leq$):} Observe that, per equation~\cref{xdef:level-return},
    $\Return{\RewardFunction}{\level}{\Policy}$
depends only on $\Policy$ through
    $\Policy(\level, \blank) : \States \to \Distr{\Actions}$,
that is, through the policy conditioned on the fixed level $\level$.
Therefore, we can construct a single policy that achieves the maximum expected return under all levels. For $\level \in \LevelSpace$, let
$
    \OptimalPolicy_\level
    \in
    \argmax_{\Policy\in\AllPolicies}
    \Return{\RewardFunction}{\level}{\Policy}
$.
Then define $\OptimalPolicy : \LevelSpace \times \States \to \Distr{\Actions}$ such that for $\level \in \LevelSpace$, $s \in \States$, and $a \in \Actions$,
\(
    \OptimalPolicy(a \mid \level, s) = \OptimalPolicy_\level(a \mid \level, s)
\).
By construction, we have
$
    \OptimalPolicy_\level
    \in
    \argmax_{\Policy\in\AllPolicies}
    \Return{\RewardFunction}{\level}{\Policy}
$
for all $\level \in \LevelSpace$.
It follows that
\begin{equation*}
    \Expect[\level \sim \distr]{
        \max_{\Policy \in \AllPolicies}
            \Return{\RewardFunction}{\level}{\Policy}
    }
    =
    \Expect[\level \sim \distr]{
        \Return{\RewardFunction}{\level}{\OptimalPolicy}
    }
    \leq
    \max_{\Policy \in \AllPolicies}
        \Expect[\level \sim \distr]{
            \Return{\RewardFunction}{\level}{\Policy}
        }.
\qedhere
\end{equation*}
\end{proof}

\section{Approximate relaxations of the minimax expected regret decision rule}
\label{apx:approximinimax}

The minimax expected regret decision rule says to choose a policy that \emph{minimizes} the expected regret with respect to the \emph{maximum} expected regret level distribution for the given policy.
In \cref{sec:mmer-robustness}, we consider one possible approximate relaxation of this decision rule, where we replace only the minimization step with approximate minimization (but retain the exact maximization step).

In this appendix, we formulate two alternative approximate relaxations of the minimax expected regret decision rule that also relax the maximization step (\cref{apx:approximinimax-definitions}). We also show that these three formulations are asymptotically equivalent (\cref{apx:approximinimax-equivalence}), and we derive robustness guarantees akin to \cref{thm:mmer-robustness} corresponding to the two new definitions (\cref{apx:approximinimax-corollaries}).

\subsection{Alternative definitions of approximate minimax expected regret}
\label{apx:approximinimax-definitions}

First, we restate the approximate MMER definition from \cref{sec:mmer-robustness}. The only difference is that we add a qualifier ``(1)'' in preparation for distinguishing this definition from the two alternatives to follow, and suppress the dependence on $\RewardFunction$ in the notation for brevity.
\begin{definition}[Approximate MMER(1), restating \cref{def:approx-mmer}]\label{def:approx-mmer1}
    Consider
        an UMDP $\TupleUMDP$,
        a goal $\RewardFunction$,
    and
        an approximation threshold $\eps \geq 0$.
    The \emph{approximate MMER(1) policy set} is then
    \begin{equation*}
        \MMEROnePolicies[\eps]
        =
        \xargmin[\eps]_{\pi\in\AllPolicies}
            \max_{\distr \in \LevelDistrs}
                \Regret{\RewardFunction}{\distr}{\Policy}.
    \end{equation*}
\end{definition}
This decision rule is approximate in that we don't assume we can find a policy that achieves the true \emph{minimum} of the maximum expected regret. However, we still assume we can find the true \emph{maximum} expected regret for each policy.
We consider next two approaches for relaxing this assumption.

The first approach casts the MMER objective as finding a Nash equilibrium of a two-player, simultaneous-play zero-sum game in which the first player is the agent selecting a policy and the second player is an adversary selecting a level distribution.
We can therefore relax both the minimization and the maximization simultaneously by using the concept of an \emph{approximate Nash equilibrium,} in which each player plays an \emph{approximate} best response.
\begin{definition}[Approximate MMER(2)]\label{def:approx-mmer2}
    Consider
        an UMDP $\TupleUMDP$,
        a goal $\RewardFunction$,
    and
        approximation thresholds $\eps, \delta \geq 0$.
    Consider the two-player zero-sum game
    $\big\langle
        \langle
            \AllPolicies
            ,
            \LevelDistrs
        \rangle,
        \langle
            -\RegretSymbol^\RewardFunction
            ,
            \RegretSymbol^\RewardFunction
        \rangle
    \big\rangle$,
    where
        an \emph{agent} plays a policy $\Policy \in \AllPolicies$
    and
        an \emph{adversary} plays a level distribution $\distr \in \LevelDistrs$,
    aiming to minimize or maximize
        $\Regret{\RewardFunction}{\distr}{\Policy}$
    respectively.
    A pair $(\Policy, \distr)$ is an \emph{$(\eps,\delta)$-equilibrium} if
    \[
        \Policy
        \in
        \xargmin_{\Policy' \in \AllPolicies}
            \Regret{\RewardFunction}{\distr}{\Policy'}
    \qquad\text{and}\qquad
        \distr
        \in
        \xargmax[\delta]_{\distr' \in \LevelDistrs}
            \Regret{\RewardFunction}{\distr'}{\Policy}
    .\]
    The \emph{approximate MMER(2) policy set} is then
    \begin{equation*}
        \MMERTwoPolicies
        =
        \bigl\{
            \Policy \in \AllPolicies
        \bigm\vert%
            \text{%
                $\exists\distr \in \LevelDistrs$
                such that
                $(\Policy,\distr)$
                is an $(\eps,\delta)$-equilibrium%
            }
        \bigr\}.
    \end{equation*}
\end{definition}

Our second approach conditions on a concrete mapping capturing an approximately optimal response from the adversary to each possible policy, with respect to which we require the chosen MMER policy to be approximately optimal. This definition could alternatively be formulated in terms of a \emph{sequential} zero-sum game where the agent chooses the policy and reveals it to the adversary prior to the adversary choosing a level distribution aiming to maximize expected regret.
\begin{definition}[Approximate MMER(3)]\label{def:approx-mmer3}
    Consider
        an UMDP $\TupleUMDP$,
        a goal $\RewardFunction$,
    and
        approximation thresholds $\eps, \eta \geq 0$.
    Let $\lambda : \AllPolicies \to \LevelDistrs$ be a function such that for all $\Policy \in \AllPolicies$,
    \begin{equation*}
        \lambda(\Policy) \in \xargmax[\eta]_{\distr \in \LevelDistrs}
            \Regret{\RewardFunction}{\distr}{\Policy}.
    \end{equation*}
    Call a function with this property an \emph{$\eta$-approximate adversarial map}.
    The \emph{approximate MMER(3) policy set} with respect to $\lambda$ is then
    \begin{equation*}
        \MMERThreePolicies
        = \xargmin[\eps]_{\Policy \in \AllPolicies}
            \Regret{\RewardFunction}{\lambda(\Policy)}{\Policy}
        .
    \end{equation*}
\end{definition}

\subsection{Asymptotic equivalence of the definitions}
\label{apx:approximinimax-equivalence}

Definitions~\ref{def:approx-mmer1}, \ref{def:approx-mmer2}, and \ref{def:approx-mmer3} do not necessarily define equal sets of policies. 
However, the three sets of policies are closely related.
\Cref{prop:asymptotic-equivalence}, below, shows that each set is contained in the others for appropriately-chosen values of the approximation thresholds $\eps, \delta, \eta \geq 0$.

\begin{apxproposition}[Asymptotic equivalence of approximate MMER definitions]\label{prop:asymptotic-equivalence}
    Consider
        an UMDP $\TupleUMDP$,
        a goal $\RewardFunction$,
        approximation thresholds $\eps, \delta, \eta \geq 0$,
    and 
        an $\eta$-approximate adversarial map $\lambda$.
    We have the following relations:
    \begin{align*}
            \MMEROnePolicies[\eps]            & \subseteq \MMERTwoPolicies[\eps,\eps]
        &   \MMERTwoPolicies[\eps,\delta]     & \subseteq \MMEROnePolicies[\eps+\delta]
        \\  \MMEROnePolicies[\eps]            & \subseteq \MMERThreePolicies[\eps+\eta,\lambda]
        &   \MMERThreePolicies[\eps,\lambda]  & \subseteq \MMEROnePolicies[\eps+\eta]
        \\  \MMERTwoPolicies[\eps,\delta]     & \subseteq \MMERThreePolicies[\eps+\delta+\eta,\lambda]
        &   \MMERThreePolicies[\eps,\lambda]  & \subseteq \MMERTwoPolicies[\eps+\eta,\eps+\eta]
    .
    \end{align*}
\end{apxproposition}

\begin{proof}
We prove the top four subset relationships. The remaining two relationships follow. Before proceeding, we note that since we assume $\LevelDistrs$, $\States$, and $\Actions$ are finite and our environments are fully observable, we have
\begin{equation}\label{eq:minimax-gap}
    \min_{\Policy\in\AllPolicies}
        \max_{\distr\in\LevelDistrs}
            \Regret{\RewardFunction}{\distr}{\Policy}
    = 
    \max_{\distr\in\LevelDistrs}
        \min_{\Policy\in\AllPolicies}
            \Regret{\RewardFunction}{\distr}{\Policy}
    .
\end{equation}
In the general case, modified bounds can be derived by accounting for the value of the difference $\displaystyle
    \min_{\Policy\in\AllPolicies}
        \max_{\distr\in\LevelDistrs}
            \Regret{\RewardFunction}{\distr}{\Policy}
    -
    \max_{\distr\in\LevelDistrs}
        \min_{\Policy\in\AllPolicies}
            \Regret{\RewardFunction}{\distr}{\Policy}
$ (see the remark after this proof).

($1 \subseteq 2$):
    Suppose $\displaystyle \Policy^{(1)}
    \in \MMEROnePolicies[\eps]
    = \xargmin[\eps]_{\Policy \in \AllPolicies}
            \max_{\mathclap{\distr\in\LevelDistrs}}
                \Regret{\RewardFunction}{\distr}{\Policy}
    $.
    Put
    $\displaystyle
        \distr^{(*)}
        \in
        \argmax_{\distr\in\LevelDistrs}
            \min_{\Policy\in\AllPolicies}
                \Regret{\RewardFunction}{\distr}{\Policy}
    $.
    Then, we have the following cycle of relations.
    \begin{align*}
        \Regret{\RewardFunction}{\distr^{(*)}}{\Policy^{(1)}}
        & \leq \max_{\distr\in\LevelDistrs} \Regret{\RewardFunction}{\distr}{\Policy^{(1)}}
        \tag{by definition of $\max$}
    \\
        & \leq \min_{\Policy\in\AllPolicies} \max_{\distr\in\LevelDistrs} \Regret{\RewardFunction}{\distr}{\Policy} + \eps
        \tag{$\displaystyle \Policy^{(1)} \in \xargmin[\eps]_{\Policy \in \AllPolicies} \max_{\distr\in\LevelDistrs} \Regret{\RewardFunction}{\distr}{\Policy}$}
    \\
        & = \max_{\distr\in\LevelDistrs} \min_{\Policy\in\AllPolicies} \Regret{\RewardFunction}{\distr}{\Policy} + \eps
        \tag{by equation~\ref{eq:minimax-gap}}
    \\
        & = \min_{\Policy\in\AllPolicies} \Regret{\RewardFunction}{\distr^{(*)}}{\Policy} + \eps
        \tag{$\displaystyle \distr^{(*)} \in \argmax_{\distr\in\LevelDistrs} \min_{\Policy\in\AllPolicies} \Regret{\RewardFunction}{\distr}{\Policy}$}
    \\
        & \leq \Regret{\RewardFunction}{\distr^{(*)}}{\Policy^{(1)}} + \eps.
        \tag{by definition of $\min$}
    \end{align*}
    Comparing the first and second-last terms, we have that
        $\Policy^{(1)} \in \xargmin[\eps]_{\Policy\in\AllPolicies} \Regret{\RewardFunction}{\distr^{(*)}}{\Policy}$.
    Comparing the second and the last terms, we have 
        $\distr^{(*)} \in \xargmax[\eps]_{\distr\in\LevelDistrs} \Regret{\RewardFunction}{\distr}{\Policy^{(1)}}$.
    It follows that
        $(\Policy^{(1)}, \distr^{(*)})$ is an $(\eps, \eps)$-equilibrium,
    which means $\Policy^{(1)} \in \MMERTwoPolicies[\eps,\eps]$.

($2 \subseteq 1$):
    Suppose $\Policy^{(2)} \in \MMERTwoPolicies[\eps,\delta]$.
    Then by definition there exists $\distr^{(2)} \in \LevelDistrs$ such that both
        $\displaystyle\Policy^{(2)} \in \xargmin_{\Policy \in \AllPolicies} \Regret{\RewardFunction}{\distr^{(2)}}{\Policy}$
    and
        $\displaystyle\distr^{(2)} \in \xargmax[\delta]_{\distr \in \LevelDistrs} \Regret{\RewardFunction}{\distr}{\Policy^{(2)}}$.
    Then we have
    \begin{align*}
        \max_{\distr\in\LevelDistrs} \Regret{\RewardFunction}{\distr}{\Policy^{(2)}}
        & \leq \Regret{\RewardFunction}{\distr^{(2)}}{\Policy^{(2)}} + \delta
        \tag{$\displaystyle\distr^{(2)} \in \xargmax[\delta]_{\distr \in \LevelDistrs} \Regret{\RewardFunction}{\distr}{\Policy^{(2)}}$}
    \\
        & \leq \min_{\Policy\in\AllPolicies} \Regret{\RewardFunction}{\distr^{(2)}}{\Policy} + \eps + \delta
        \tag{$\displaystyle\Policy^{(2)} \in \xargmin_{\Policy \in \AllPolicies} \Regret{\RewardFunction}{\distr^{(2)}}{\Policy}$}
    \\
        & \leq \max_{\distr\in\LevelDistrs} \min_{\Policy\in\AllPolicies} \Regret{\RewardFunction}{\distr}{\Policy} + \eps + \delta
        \tag{by definition of $\max$}
    \\ 
        & = \min_{\Policy\in\AllPolicies} \max_{\distr\in\LevelDistrs} \Regret{\RewardFunction}{\distr}{\Policy} + \eps + \delta.
        \tag{by equation~\ref{eq:minimax-gap}}
    \end{align*}
    
    Therefore,
        $\Policy^{(2)}
        \in
        \xargmin[(\eps{+}\delta)]_{\Policy \in \AllPolicies}
            \max_{\distr\in\LevelDistrs}
                \Regret{\RewardFunction}{\distr}{\Policy}
        = \MMEROnePolicies[\eps+\delta]$.

($1 \subseteq 3$):
    Suppose $\displaystyle \Policy^{(1)}
    \in \MMEROnePolicies[\eps]
    = \xargmin[\eps]_{\Policy \in \AllPolicies}
            \max_{\distr\in\LevelDistrs}
                \Regret{\RewardFunction}{\distr}{\Policy}
    $.
    Note also that, by definition, for all $\Policy \in \AllPolicies$,
        $\displaystyle \lambda(\Policy) \in \xargmax[\eta]_{\distr \in \LevelDistrs} \Regret{\RewardFunction}{\distr}{\Policy}$.
    Then we have
    \begin{align*}
    \Regret{\RewardFunction}{\lambda(\Policy^{(1)})}{\Policy^{(1)}}
        & \leq \max_{\distr\in\LevelDistrs} \Regret{\RewardFunction}{\distr}{\Policy^{(1)}}
        \tag{by definition of $\max$}
    \\
        & \leq \min_{\Policy\in\AllPolicies} \max_{\distr\in\LevelDistrs} \Regret{\RewardFunction}{\distr}{\Policy} + \eps
        \tag{$\displaystyle \Policy^{(1)} \in \xargmin[\eps]_{\Policy \in \AllPolicies} \max_{\distr\in\LevelDistrs} \Regret{\RewardFunction}{\distr}{\Policy}$}
    \\
        & \leq \min_{\Policy\in\AllPolicies} \left(
            \Regret{\RewardFunction}{\lambda(\Policy)}{\Policy} + \eta
        \right) + \eps
        \tag{by definition of $\lambda$; monotonicity of $\min$}
    \\
        & \leq \min_{\Policy\in\AllPolicies} \Regret{\RewardFunction}{\lambda(\Policy)}{\Policy} + \eta + \eps
    .
        \tag{$\eta$ constant wrt.\ $\Policy$}
    \end{align*}
    Therefore, $\Policy^{(1)} \in \xargmin[(\eps{+}\eta)]_{\Policy\in\AllPolicies} \Regret{\RewardFunction}{\lambda(\Policy)}{\Policy} = \MMERThreePolicies[\eps+\eta,\lambda]$.

($3 \subseteq 1$):
    Suppose $\displaystyle \Policy^{(3)} \in \MMERThreePolicies[\eps,\lambda] = \xargmin[\eps]_{\Policy\in\AllPolicies} \Regret{\RewardFunction}{\lambda(\Policy)}{\Policy}$.
    Then we have
    \begin{align*}
        \max_{\distr\in\LevelDistrs} \Regret{\RewardFunction}{\distr}{\Policy^{(3)}}
        & \leq \Regret{\RewardFunction}{\lambda(\Policy^{(3)})}{\Policy^{(3)}} + \eta
        \tag{by definition of $\lambda$}
    \\
        & \leq \min_{\Policy\in\AllPolicies}\Regret{\RewardFunction}{\lambda(\Policy)}{\Policy} + \eps + \eta
        \tag{$\displaystyle \Policy^{(3)} \in \xargmin[\eps]_{\Policy\in\AllPolicies} \Regret{\RewardFunction}{\lambda(\Policy)}{\Policy}$}
    \\
        & \leq \min_{\Policy\in\AllPolicies} \max_{\distr\in\LevelDistrs} \Regret{\RewardFunction}{\distr}{\Policy} + \eps + \eta
    .
        \tag{by definition of $\max$; monotonicity of $\min$}
    \end{align*}
    Therefore, $\Policy^{(3)} \in \xargmin[(\eps{+}\eta)]_{\Policy\in\AllPolicies} \max_{\distr\in\LevelDistrs} \Regret{\RewardFunction}{\distr}{\Policy} = \MMEROnePolicies[\eps+\eta]$.

($2 \subseteq 3$):
    follows from ($2 \subseteq 1$) and ($1 \subseteq 3$).

($3 \subseteq 2$):
    follows from ($3 \subseteq 1$) and ($1 \subseteq 2$).
\end{proof}

\paragraph{Remarks \textnormal{(Generalization to non-finite environments)}.}
Unlike our other results, \cref{prop:asymptotic-equivalence} relies on the assumption that the UMDP is finite. This assumption guarantees that there exists a Nash equilibrium for the game in \cref{def:approx-mmer2}.
The definitions and results can be generalized to infinite environments, so long as the necessary minima and maxima are defined. However, if there do not exist (approximate) Nash equilibria at some approximation thresholds, then it becomes necessary to account for the possibility that the set of MMER(2) policies is empty.
We can generalize the relations in terms of the \textbf{minimax gap},
\[\displaystyle
    \Delta
    =
    \min_{\Policy\in\AllPolicies}
        \max_{\distr\in\LevelDistrs}
            \Regret{\RewardFunction}{\distr}{\Policy}
    -
    \max_{\distr\in\LevelDistrs}
        \min_{\Policy\in\AllPolicies}
            \Regret{\RewardFunction}{\distr}{\Policy}
    \geq 0.
\]
The minimax gap is always non-negative when it is well-defined, due to the max--min inequality. If the game permits Nash equilibria, the max--min inequality is an equality and $\Delta = 0$.
In general, if the game permits an $(\eps, \delta)$-equilibrium for some $\eps,\delta \geq 0$, then $\Delta \leq \eps+\delta$, and there exists $\eps,\delta\geq0$ and an $(\eps,\delta)$-equilibrium such that $\Delta = \eps+\delta$.

We can now extend the proof of \cref{prop:asymptotic-equivalence} to derive the following relations:
\begin{align*}
        \MMEROnePolicies[\eps]            & \subseteq \MMERTwoPolicies[\eps+\Delta,\eps+\Delta]
    &   \MMERTwoPolicies[\eps,\delta]     & \subseteq   \MMEROnePolicies[\eps+\delta-\Delta]
    \\  \MMERTwoPolicies[\eps,\delta]     & \subseteq \MMERThreePolicies[\eps+\delta-\Delta+\eta,\lambda]
    &   \MMERThreePolicies[\eps,\lambda]  & \subseteq \MMERTwoPolicies[\eps+\eta+\Delta,\eps+\eta+\Delta].
\end{align*}
The relations between MMER(1) and MMER(3), not involving MMER(2) or the existence of equilibria, are unchanged.

\clearpage

\subsection{Generalizing the robustness result}
\label{apx:approximinimax-corollaries}

In this section, we combine \cref{thm:mmer-robustness,prop:asymptotic-equivalence} to show robustness results for the two alternative definitions of approximate MMER policies (\cref{def:approx-mmer2,def:approx-mmer3}).

\begin{apxcorollary}[MMER(2) is robust to goal misgeneralization]
Consider
    an UMDP $\TupleUMDP$,
    a pair of goals $\RewardFunction, \ProxyRewardFunction$,
    a proxy-distinguishing distribution shift $\TupleShift$,
and
    approximation thresholds $\eps, \delta \geq 0$.
Then
\[
    \forall \MMERPolicy \in \MMERTwoPolicies[\eps,\delta],
    \text{ we have }
    \MMERPolicy
    \in
    \xargmax[(\eps{+}\delta)]{}_{\Policy \in \AllPolicies} 
        \Return{\RewardFunction}{\distrTest}{\Policy}
.\]
\end{apxcorollary}

\begin{proof}
    Let $\MMERPolicy \in \MMERTwoPolicies[\eps,\delta]$.
    We have $\MMERPolicy \in \MMEROnePolicies[\eps+\delta]$ by \cref{prop:asymptotic-equivalence}.
    The corollary follows by \cref{thm:mmer-robustness}.
\end{proof}

\begin{apxcorollary}[MMER(3) is robust to goal misgeneralization]
Consider
    an UMDP $\TupleUMDP$,
    a pair of goals $\RewardFunction, \ProxyRewardFunction$,
    a proxy-distinguishing distribution shift $\TupleShift$,
    approximation thresholds $\eps, \eta \geq 0$,
and 
    an $\eta$-approximate adversarial map $\lambda$.
Then
\[
    \forall \MMERPolicy \in \MMERThreePolicies[\eps,\lambda],
    \text{ we have }
    \MMERPolicy \in \xargmax[(\eps{+}\eta)]{}_{\Policy \in \AllPolicies} \Return{\RewardFunction}{\distrTest}{\Policy}
.\]
\end{apxcorollary}

\begin{proof}
    Let $\MMERPolicy \in \MMERThreePolicies[\eps,\lambda]$.
    We have $\MMERPolicy \in \MMEROnePolicies[\eps+\eta]$ by \cref{prop:asymptotic-equivalence}.
    The corollary follows by \cref{thm:mmer-robustness}.
\end{proof}

\section{Optimizing minimax expected regret is necessary if you want to prevent misgeneralization under all possible distribution shifts}
\label{apx:necessitymmer}

In \cref{sec:mmer-robustness}, we show that approximately optimizing the MMER objective is sufficient for preventing misgeneralization under a distribution shift.
In this appendix, we show that if we want policies that are robust to \emph{all possible} distribution shifts, then finding an approximate MMER policy is both sufficient \emph{and also necessary.}

\begin{apxcorollary}
Consider
    an UMDP $\TupleUMDP$,
    a goal $\RewardFunction$,
and
    an approximation threshold $\eps \geq 0$.
We have that
\begin{equation*}
\left(\forall \distr \in \LevelDistrs,
    \Policy \in \xargmax[\eps]_{\Policy'\in\AllPolicies}
        \Return{\RewardFunction}{\distr}{\Policy'}
\right)
\qquad\text{if and only if}\qquad    
\Policy \in \MMEROptimalPolicies[\eps]{\RewardFunction}
.
\end{equation*}
\end{apxcorollary}

\begin{proof}
($\Rightarrow$):
    Suppose $\forall \distr \in \LevelDistrs,
    \Policy \in \xargmax[\eps]_{\Policy'\in\AllPolicies}
        \Return{\RewardFunction}{\distr}{\Policy'}$.
    Then
    \begin{equation*}
        \max_{\distr\in\LevelDistrs} \Regret{\RewardFunction}{\distr}{\Policy}
        = \max_{\distr\in\LevelDistrs}
            \left(\max_{\Policy'\in\AllPolicies}\Return{\RewardFunction}{\distr}{\Policy'}
            - \Return{\RewardFunction}{\distr}{\Policy}
            \right)
        \leq \eps.
    \end{equation*}
    Since regret is non-negative it follows that
    $\displaystyle\Policy
    \in \xargmin[\eps]_{\Policy'\in\AllPolicies}
        \max_{\distr\in\LevelDistrs}
        \Regret{\RewardFunction}{\distr}{\Policy'}
    = \MMEROptimalPolicies[\eps]{\RewardFunction}$.

($\Leftarrow$): Apply \cref{thm:mmer-robustness} for each $\distr \in \LevelDistrs$.
\end{proof}

While we prove a general result that holds for all possible level distributions, we are mainly interested in its implications for goal misgeneralization. In particular, imagine having any sort of distribution shift from training to deployment. This theorem implies that any policy that generalizes in this case is an MMER policy. In principle, it would be possible to find some of the generalizing policies in this set by other training methods. However, it is clear that MMER (or any refinement) should be the training objective utilized if we want to be certain to be able to identify the entire set of solutions that \textit{never} suffer from goal misgeneralization.

\clearpage
\section{Partially observable environments}
\label{apx:partial-observability}

In this section, we generalize equation~\cref{eq:distr-regret} and \cref{thm:mmer-robustness} to partially observable environments, with a slight modification of the bound to account for the fact that it may no longer be possible for any policy to achieve zero expected regret on a given distribution of levels.

Rather than defining underspecified partially observable MDPs in detail, we consider an arbitrary subset of the space of policies $\SomePolicies \subseteq \AllPolicies$. We can model partial observability by restricting to policies with tied outputs within any given partition of $\LevelSpace \times \States$ into information sets.

We define the \textbf{expected restricted regret} of a policy $\Policy \in \SomePolicies$ in a level $\level \in \LevelSpace$ under a goal $\RewardFunction$ based on the return of the best available policy within such a subset of policies:
\begin{equation}\label{eq:restricted-regret-level}
    \Regret[\SomePolicies]{\RewardFunction}{\level}{\Policy}
    =
    \max_{\Policy' \in \SomePolicies}
        \Return{\RewardFunction}{\level}{\Policy'}
    - \Return{\RewardFunction}{\level}{\Policy}.
\end{equation}
As before, we lift this definition to a level distribution $\distr$ by taking the expectation
\begin{equation}\label{eq:restricted-regret-distr}
    \Regret[\SomePolicies]{\RewardFunction}{\distr}{\Policy}
    =
    \Expect[\level\sim\distr]{
        \Regret[\SomePolicies]{\RewardFunction}{\level}{\Policy}
    }.
\end{equation}

\subsection{Generalizing the expected regret identity to partially observable environments}
\label{apx:partial-regret-identity}

\begin{apxproposition}[Expected regret identity for partially observable environments]
\label{prop:identity2}
Consider
    an UMDP $\TupleUMDP$,
    a goal $\RewardFunction$,
    a level distribution $\distr \in \LevelDistrs$,
and
    a subset of policies $\SomePolicies \subseteq \AllPolicies$.
Then
\[
    \Regret[\SomePolicies]{\RewardFunction}{\distr}{\Policy}
    =
    \max_{\Policy' \in \SomePolicies}
        \Return{\RewardFunction}{\distr}{\Policy'}
    - \Return{\RewardFunction}{\distr}{\Policy}
    + \min_{\Policy' \in \SomePolicies}
        \Regret[\SomePolicies]{\RewardFunction}{\distr}{\Policy'}
.\]    
\end{apxproposition}

\begin{proof}
Equivalently,
\begin{align*}
    &\Regret[\SomePolicies]{\RewardFunction}{\distr}{\Policy}
    - \min_{\Policy' \in \SomePolicies}
        \Regret[\SomePolicies]{\RewardFunction}{\distr}{\Policy'}
    \\
    &= 
    \Expect[\level\sim\distr]{
        \Regret[\SomePolicies]{\RewardFunction}{\level}{\Policy}
    }
    - \min_{\Policy' \in \SomePolicies}
        \Expect[\level\sim\distr]{
            \Regret[\SomePolicies]{\RewardFunction}{\level}{\Policy'}
        }
    \tag{by equation~\ref{eq:restricted-regret-distr}}
\\  &=
    \Expect[\level\sim\distr]{
        \max_{\Policy' \in \SomePolicies}
            \Return{\RewardFunction}{\level}{\Policy'}
    }
    - \Return{\RewardFunction}{\distr}{\Policy}
    - \min_{\Policy' \in \SomePolicies}\left(
        \Expect[\level\sim\distr]{
            \max_{\Policy'' \in \SomePolicies}
                \Return{\RewardFunction}{\level}{\Policy''}
        }
        - \Return{\RewardFunction}{\distr}{\Policy'}
    \right)
    \tag{by equations~\ref{eq:restricted-regret-level} and \ref{xdef:distr-return}}
\\  &=
    \Expect[\level\sim\distr]{
        \max_{\Policy' \in \SomePolicies}
            \Return{\RewardFunction}{\level}{\Policy'}
    }
    - \Return{\RewardFunction}{\distr}{\Policy}
    - \Expect[\level\sim\distr]{
            \max_{\Policy'' \in \SomePolicies}
                \Return{\RewardFunction}{\level}{\Policy''}
        }
    +
    \max_{\Policy' \in \SomePolicies}
        \Return{\RewardFunction}{\distr}{\Policy'}
\\  &=
    \max_{\Policy' \in \SomePolicies}
        \Return{\RewardFunction}{\distr}{\Policy'}
    - \Return{\RewardFunction}{\distr}{\Policy}.
\tag*{\qedhere}
\end{align*}
\end{proof}

Compared to equation~\cref{eq:distr-regret} (\cref{prop:identity}), \cref{prop:identity2} includes the term
    $\min_{\Policy \in \SomePolicies} \Regret[\SomePolicies]{\RewardFunction}{\distr}{\Policy}$.
This extra term represents the \textbf{irreducible (expected restricted) regret} for the level distribution $\distr$.
\Cref{prop:identity2} essentially says that the restricted expected regret for a level distribution can be decomposed into two components:
(1)~the shortfall in expected return compared to the optimal policy for the level distribution (as in equation~\ref{eq:distr-regret}, cf.\ the definition of regret for individual levels);
plus (2)~this irreducible regret.

Irreducible regret can arise when the level is partially observable to the policy.
For example, consider a mixture of two levels with two disjoint sets of optimal policies. Suppose the level is not observed by the policy, so the policy has to choose actions without knowing whether it is in the first level or the second level.
In each individual level, we define expected regret based on the performance of optimal policies for that level (these policies will naturally be the ones that \emph{assume} they are in the appropriate level).
However, since no single policy can perform optimally in both levels, the expected regret with respect to the mixture is always nonzero. This nonzero minimum expected regret is exactly the irreducible regret.

\newpage

\subsection{Generalizing \texorpdfstring{\cref{def:approx-mmer,thm:mmer-robustness}}{Definition~\ref*{def:approx-mmer} and Theorem~\ref*{thm:mmer-robustness}} to partially observable environments}

We are now in position to generalize the results of \cref{sec:mmer-robustness} to partially observable environments.
First, we adapt \cref{def:approx-mmer} to expected restricted regret.

\begin{apxdefinition}[Approximate MMER for partially observable environments]
\label{def:approx-mmer-partial}
Consider
    an UMDP $\TupleUMDP$,
    a goal $\RewardFunction$,
    an approximation threshold $\eps, \delta \geq 0$,
and
    a subset of policies $\SomePolicies \subseteq \AllPolicies$.
The \emph{restricted approximate MMER policy set} is then
\begin{equation*}
    \PartialMMERPolicies[\SomePolicies,\eps]{\RewardFunction}
    =
    \xargmin_{\Policy'\in\SomePolicies}
        \max_{\distr\in\LevelDistrs}
            \Regret[\SomePolicies]{\RewardFunction}{\distr}{\Policy'}
    .
\end{equation*}
\end{apxdefinition}

Next, we generalize \cref{thm:mmer-robustness}. In doing so, we need to account for the \textbf{irreducible regret gap}
\begin{equation}\label{eq:irreducible-regret-gap}
    g(\distrTest)
    = 
    \min_{\Policy \in \SomePolicies}
        \max_{\distr \in \LevelDistrs}
            \Regret[\SomePolicies]{\RewardFunction}{\distr}{\Policy}
    - \min_{\Policy \in \SomePolicies}
        \Regret[\SomePolicies]{\RewardFunction}{\distrTest}{\Policy}
    ,
\end{equation}
the difference in irreducible regret between an MMER level distribution and $\distrTest$. When the deployment distribution has lower irreducible regret than an MMER level distribution, the agent has no incentive to improve expected restricted regret on the deployment distribution once it is below the irreducible regret of the MMER distribution.

If the irreducible regret gap is large, then this undermines the robustness guarantee. This is a limitation of standard MMER. However, we note that it can be addressed by a lexicographic refinement of the decision rule along the lines of \citet{Beukman+2024}.

\begin{apxtheorem}[MMER is robust to goal misgeneralization in partially observable environments]
\label{thm:mmer-robustness-partial}
Consider
    an UMDP $\TupleUMDP$,
    a pair of goals $\RewardFunction, \ProxyRewardFunction$,
    a proxy-distinguishing distribution shift $\TupleShift$,
    an approximation threshold $\eps \geq 0$,
and a subset of policies
    $\SomePolicies \subseteq \AllPolicies$.
Let \(\displaystyle
    g(\distrTest)
    = 
    \min_{\Policy \in \SomePolicies}
        \max_{\distr \in \LevelDistrs}
            \Regret[\SomePolicies]{\RewardFunction}{\distr}{\Policy}
    - \min_{\Policy \in \SomePolicies}
        \Regret[\SomePolicies]{\RewardFunction}{\distrTest}{\Policy}
\)
be the irreducible regret gap.
Then
\begin{equation*}
    \forall \MMERPolicy \in \PartialMMERPolicies[\SomePolicies,\eps]{\RewardFunction},
    \text{\ we have\ }
    \MMERPolicy \in 
    \xargmax[(\eps{+}g(\distrTest))]{}_{\Policy \in \SomePolicies}
        \Return{\RewardFunction}{\distrTest}{\Policy}.
\end{equation*}
\end{apxtheorem}

\begin{proof}
Suppose $\MMERPolicy \in \PartialMMERPolicies{\RewardFunction}$.
Then, along similar lines to the proof of \cref{thm:mmer-robustness}, we have the following bound on expected restricted regret:
\begin{align*}
    \Regret[\SomePolicies]{\RewardFunction}{\distrTest}{\MMERPolicy}
    & \leq \max_{\distr \in \LevelDistrs} \Regret[\SomePolicies]{\RewardFunction}{\distr}{\MMERPolicy}
\tag{$\distrTest \in \LevelDistrs$}
\\  
    & \leq \min_{\Policy\in\SomePolicies}
        \max_{\distr \in \LevelDistrs}
            \Regret[\SomePolicies]{\RewardFunction}{\distr}{\Policy}
        + \eps.
\tag{by \cref{def:approx-mmer-partial}}
\end{align*}
Once again, we can convert this upper bound on expected restricted regret to a lower bound on expected return:
\begin{align*}
    &\Return{\RewardFunction}{\distrTest}{\MMERPolicy}
\\  &=
    \max_{\Policy\in\SomePolicies}
        \Return{\RewardFunction}{\distrTest}{\Policy}
    - \Regret[\SomePolicies]{\RewardFunction}{\distrTest}{\MMERPolicy}
    + \min_{\Policy \in \SomePolicies}
        \Regret[\SomePolicies]{\RewardFunction}{\distrTest}{\Policy}
\tag{by \cref{prop:identity2}}
\\  &\geq
    \max_{\Policy\in\SomePolicies}
        \Return{\RewardFunction}{\distrTest}{\Policy}
    - \left(
        \min_{\Policy\in\SomePolicies}
            \max_{\distr\in\LevelDistrs}
                \Regret[\SomePolicies]{\RewardFunction}{\distr}{\Policy}
        + \eps
    \right)
    + \min_{\Policy \in \SomePolicies}
        \Regret[\SomePolicies]{\RewardFunction}{\distrTest}{\Policy}
\tag{by above bound}
\\  &=
    \max_{\Policy\in\SomePolicies}
        \Return{\RewardFunction}{\distrTest}{\Policy}
    - \eps
    - g(\distrTest).
\tag{by equation~\ref{eq:irreducible-regret-gap}}
\end{align*}
Thus, $\MMERPolicy \in 
    \xargmax[(\eps{+}g(\distrTest))]{}_{\Policy \in \SomePolicies}
        \Return{\RewardFunction}{\distrTest}{\Policy}.$
\end{proof}

\clearpage
\section{Additional environment details}
\label{apx:environments}
\label{apx:oracles}

In this appendix, we provide additional details about each environment, including details about procedurally generating non-distinguishing and distinguishing levels, edit distributions, computation of maximum level value for the oracle-latest estimator, and the origin of each environment.

\subsection{The \env{cheese in the corner} environment}

\paragraph{Environment.}
In this environment, levels are parameterized by a $13 \times 13$ wall layout, a mouse spawn position within this grid, and a cheese position within the grid. We require that the cheese position and the mouse spawn position are not equal, and moreover that they are not obstructed by walls. We do not require that there is an unobstructed path between them.

In the initial state, the mouse begins in the spawn position.
The actions available to the agent are to attempt to move the mouse up, left, down, or right, which succeeds if the respective position is not obstructed by a wall or the edge of the grid.
If the mouse position equals the cheese position, the mouse collects the cheese.
The episode terminates when the cheese has been collected or after a maximum of 128 steps.

\paragraph{Observations.}
We represent states to the agent as a $15 \times 15 \times 3$ Boolean grid. The first of the three channels encodes the maze layout, including a border of width 1. The second channel one-hot encodes the position of the mouse. The third channel one-hot encodes the position of the cheese (if it has not been collected). All of the relevant information about the level and the state is encoded into this observation, therefore this environment is fully observable.

\paragraph{True goal and proxy goal.}
The training goal is for the mouse to collect the cheese. The reward function assigns $+1$ reward to transitions in which the mouse collects the cheese.

We also consider a proxy goal of navigating to the top left corner of the maze. This could be formulated as a reward function that assigns $+1$ reward the first time the mouse steps into the top left corner (this reward can be made Markovian by augmenting the state with a flag for whether the corner has previously been reached). Note that we never train with this proxy goal as the reward function.

\paragraph{Level classification.}
Given this environment and this pair of goals, we can classify levels according to the definitions in \cref{sec:setting}. Note that in the following, we assume that the discount factor is strictly between $0$ and $1$ (so that shorter paths obtain higher return), and that the property of \emph{reachability} accounts for episode termination conditions (including collecting the cheese).
\begin{enumerate}
    \item
        \textbf{Levels for which the cheese is in the top-left corner of the maze are non-distinguishing.}
        If the top-left corner is reachable from the mouse spawn position, optimally pursuing the proxy goal implies following a shortest path to the corner, which is also a shortest path to collecting the cheese.
        If the top-left corner is unreachable from the mouse spawn position, then all policies achieve zero return according to either goal, and are therefore equally optimal.
    
    \item
        \textbf{Most levels for which the cheese is not in the top-left corner of the maze are distinguishing.}
        If the cheese is not in the top-left corner, but is reachable from the mouse spawn position, then either (1)~the top-left corner is reachable from the mouse spawn position (other than via the cheese), or (2)~it is not.
        If (1), then a proxy-optimal policy could visit the top-left corner and then remain there until the episode terminates. If (2), then all policies are optimal under the proxy goal, and in particular there exists a proxy-optimal policy that avoids the cheese.
    \item
        \textbf{Some levels for which the cheese is in the top-left corner are non-distinguishing.}
        If the cheese is not in the top-left corner, and is not reachable from the mouse spawn position, then all policies are optimal under the true goal, and in particular all proxy-optimal policies are.
\end{enumerate}
When defining the proportion of distinguishing levels in the buffer, we use the approximately correct approach of checking whether the cheese is not in the corner. We note that all UED algorithms rapidly remove levels in which the cheese is unreachable from their buffers.

\paragraph{Procedural level generation.}
We construct two procedural level generators, $\distrNonDistg, \distrDistg \in \LevelDistrs$, (approximately) concentrated on non-distinguishing and distinguishing levels, respectively.
\begin{itemize}
    \item \textbf{Non-distinguishing level distribution \textnormal{($\distrNonDistg$)}.}
        We procedurally generate non-distinguishing levels as follows. We position the cheese in the top-left corner. For each remaining position, we place a wall independently with probability $25\%$. We position the mouse spawn in some remaining position, assuming there is at least one such position.
        
        All of the generated levels are technically non-distinguishing, though this may include levels for which the cheese position is unreachable from the mouse spawn position. Note that these levels still do not provide any training signal in favor of the true goal over the proxy goal.
        
    \item \textbf{Distinguishing level distribution \textnormal{($\distrDistg$)}.}
        We procedurally generate \emph{mostly} distinguishing levels as follows. For each position, we place a wall independently with probability $25\%$. We position the cheese somewhere where there is not a wall. We position the mouse spawn somewhere where there is not a wall, different from the cheese position.
        
        The generated levels are in most cases $C$-distinguishing levels for some $C$. It may arise that the cheese is unreachable from the mouse spawn location or is in the top-left corner, making the level non-distinguishing. Note that the effect of these levels is only to reduce the training signal in favor of the true goal available to the agent, as if to reduce $\alpha$.
\end{itemize}

\paragraph{Elementary edit distributions.}
ACCEL additionally requires specifying an edit distribution used to sample ``similar'' levels for potential entry into the level buffer.
In \cref{apx:mutators}, we discuss how we build our edit distributions from elementary edit distributions of the following three types.
\begin{enumerate}
    \item \textbf{Classification preserving edits.}
        Each classification preserving edit either removes an existing wall, positions a new wall, or moves the mouse spawn position to a random unobstructed position (other than the cheese position). These edits don't change the cheese position, though they may change whether the cheese position is reachable from the mouse spawn position. As a result, they may change the exact classification of the level (though not its approximate classification).
    \item \textbf{Biased classification transforming edits.}
        Given a probability $\alpha$, a biased classification transforming edit randomizes the cheese position with probability $\alpha$ or places the cheese in the top-left corner with probability $1-\alpha$. Note that when randomizing the cheese position, it's possible that the cheese will be positioned in the top-left corner.
    \item \textbf{Unrestricted classification transforming edits.}
        An unrestricted classification transforming edit randomizes the cheese position with probability~$1$. It's possible that the cheese will be positioned in the top-left corner.
\end{enumerate}

\paragraph{Oracle maximum return.}
In this environment, an optimal policy follows any shortest path from the mouse spawn position to the cheese position in the graph representing the maze layout. We compute the length of a shortest path using the Floyd--Warshall algorithm. Given a level $\theta \in \LevelSpace$. Let $R$ be the true reward function ($+1$ for collecting the cheese). If the shortest path has length $d \in \mathbb{N} \cup \{\infty\}$, then with discount factor $\gamma \in (0,1)$, we have
\begin{equation}\label{eq:corner-max-return}
    \max_{\Policy} \Return{\RewardFunction}{\theta}{\Policy} = \gamma^{d-1}
.
\end{equation}
This covers the case in which the cheese position is unreachable from the mouse spawn position ($d = \infty$). Shortest paths of (finite) length greater than the maximum episode length of 128 are impossible given this grid size.

\paragraph{Origin.}
\citet{Hubinger2019} originated the idea of creating a navigation task with a location proxy as a potential means of inducing goal misgeneralization.
\citet{Langosco+2022} implemented the \env{Maze I} based on this idea by modifying the \env{Maze} environment of OpenAI ProcGen \citep{Cobbe+2020} such that the cheese position could be restricted to a region surrounding the top-right corner (cf., \cref{apx:robustness-corner}).
\env{Cheese in the corner} is an interpretation of \env{Maze~I} using an original JAX implementation. We depart from \env{Maze~I} by using a fixed size maze and replacing the maze layout algorithm with a simpler algorithm based on random block placement.

\subsection{The \env{cheese on a dish} environment}
\label{apx:environments:dish}

\paragraph{Environment.}
In this environment, levels are parameterized by a $13 \times 13$ wall layout, and positions within this grid for the mouse spawn, the dish, and the cheese. We require that the cheese position and the mouse spawn position are not equal, and that the dish position and the mouse spawn position are not equal, though the dish can be co-located with the cheese. Moreover, we require that none of the three positions are obstructed by walls. We do not require that there is an unobstructed path between the positions.

In the initial state, the mouse begins in the spawn position.
The actions available to the agent are to attempt to move the mouse up, left, down, or right, which succeeds if the respective position is not obstructed by a wall or the edge of the grid.
If the mouse position equals the cheese position, the mouse collects the cheese, and likewise for the dish.
The episode terminates when the cheese \emph{or} the dish has been collected or after a maximum of 128 steps.

\paragraph{Observations.}
We represent states to the agent as a $15 \times 15 \times (3+D)$ Boolean grid where $D \in \mathbb{N}$. The first of the channels encodes the maze layout, including a border of width 1. The second channel one-hot encodes the position of the mouse. The third channel one-hot encodes the position of the cheese (if it has not been collected). The remaining $D$ channels each one-hot encode the dish position (if the dish has not been collected).

Redundantly coding the dish encourages the agent to learn a policy that is based on the dish position, eliciting a clearer case of goal misgeneralization. In our main experiments, $D=6$. 
In \cref{apx:robustness-dish}, we consider alternative values of $D$.
All of the relevant information about the level and the state is encoded into this observation, therefore this environment is fully observable.

\paragraph{True goal and proxy goal.}
The training goal is for the mouse to collect the cheese. The reward function assigns $+1$ reward to transitions in which the mouse collects the cheese.
We also consider a proxy goal of collecting the dish. This reward function assigns $+1$ reward to transitions in which the mouse collects the dish.
Note that we never train with this proxy goal as the reward function.

\paragraph{Level classification.}
Given this environment and this pair of goals, we can classify levels according to the definitions in \cref{sec:setting}. Note that in the following, we assume that the discount factor is strictly between $0$ and $1$ (so that shorter paths obtain higher return), and that the property of \emph{reachability} accounts for episode termination conditions.
\begin{enumerate}
    \item
        \textbf{Levels for which the cheese and the dish are co-located are non-distinguishing.}
        If the cheese/dish position is reachable from the mouse spawn position, optimally pursuing the proxy goal implies following a shortest path to the dish, which is also a shortest path to the cheese.
        If the cheese/dish position is unreachable from the mouse spawn position, then all policies achieve zero return under either goal, and are therefore equally optimal.
    
    \item
        \textbf{Most levels for which the cheese is not in the same position as the dish are distinguishing.}
        If the cheese is not in the same position as the dish, but is reachable from the mouse spawn position, then either (1)~the dish is reachable from the mouse spawn position (other than via the cheese), or (2)~it is not.
        If~(1), then a proxy-optimal policy could visit the dish, terminating the episode.
        If~(2), then all policies are optimal under the proxy goal, and in particular there exists a proxy-optimal policy that avoids the cheese.

    \item 
        \textbf{Some levels for which the cheese is not in the same position as the dish are non-distinguishing.}
        If the cheese is not in the same position as the dish, and moreover is not reachable from the mouse spawn position, then all policies are optimal under the true goal, and in particular all proxy-optimal policies are.
\end{enumerate}
When defining the proportion of distinguishing levels in the buffer, we use the approximately correct approach of checking whether the cheese and the dish have distinct positions. We note that all UED algorithms rapidly remove levels in which the cheese is unreachable from their buffers.

\newpage

\paragraph{Procedural level generation.}
We construct two procedural level generators, $\distrNonDistg, \distrDistg \in \LevelDistrs$, (approximately) concentrated on non-distinguishing and distinguishing levels, respectively.
\begin{itemize}
    \item \textbf{Non-distinguishing level distribution \textnormal{($\distrNonDistg$)}.}
        We procedurally generate non-distinguishing levels as follows. For each position, we place a wall independently with probability $25\%$. We position the mouse spawn somewhere where there is not a wall. We position the cheese and the dish somewhere where there is not a wall, other than the mouse spawn position (we assume there are at least two positions without walls).
        
        All of the generated levels are technically non-distinguishing, though this may include levels for which the cheese position is unreachable from the mouse spawn position. Note that these levels still do not provide any training signal in favor of the true goal over the proxy goal.
        
    \item \textbf{Distinguishing level distribution \textnormal{($\distrDistg$)}.}
        We procedurally generate \emph{mostly} distinguishing levels as follows. For each position, we place a wall independently with probability $25\%$. We position the mouse spawn somewhere where there is not a wall. We position the cheese and the dish, independently, somewhere where there is not a wall, different from the mouse spawn position.
        
        The generated levels are in most cases $C$-distinguishing levels for some $C$. It may arise that the cheese is unreachable from the mouse spawn location or is in the top-left corner, making the level non-distinguishing. Note that the effect of these levels is only to reduce the training signal in favor of the true goal available to the agent, as if to reduce $\alpha$.
\end{itemize}

\paragraph{Elementary edit distributions.}
ACCEL additionally requires specifying an edit distribution used to sample ``similar'' levels for potential entry into the level buffer.
In \cref{apx:mutators}, we discuss how we build our edit distributions from elementary edit distributions of the following three types.
\begin{enumerate}
    \item \textbf{Classification preserving edits.}
        Each classification preserving edit either removes an existing wall, positions a new wall, or moves the mouse spawn position to a random unobstructed position (other than the cheese position or the dish position). These edits don't change the cheese position or the dish position, though they may change whether the cheese position is reachable from the mouse spawn position. As a result, they may change the exact classification of the level (though not its approximate classification).
    \item \textbf{Biased classification transforming edits.}
        Given a probability $\alpha$, a biased classification transforming edit randomizes the dish position, and then either independently randomizes the cheese position (with probability $\alpha$) or positions the cheese at the new dish position (with probability $1-\alpha$). Note that when independently randomizing the cheese position, it's possible that the cheese will be co-located with the dish.
    \item \textbf{Unrestricted classification transforming edits.}
        An unrestricted classification transforming edit independently randomizes the cheese position and the dish position with probability~$1$. It's possible that the cheese and the dish will be co-located.
\end{enumerate}

\paragraph{Oracle maximum return.}
In this environment, an optimal policy follows any shortest path from the mouse spawn position to the cheese position in the graph representing the maze layout (in which the dish counts as an obstruction if it is not co-located with the cheese, since collecting the dish terminates the episode). Given this graph we compute the maximum expected return as in \cref{eq:corner-max-return}.

\paragraph{Origin.}
\citet{Langosco+2022} originally implemented the \env{Maze~II} environment by modifying the \env{Maze} environment from OpenAI ProcGen \citep{Cobbe+2020} such that the cheese is replaced by a yellow diamond during training, and subsequently by a pair of objects (a red diamond and a yellow diagonal line) during evaluation.
In their setup, the yellow diamond is rewarding because of its shape (diamond), rather than its color (yellow), so the intended extrapolation of the goal is for the policy to pursue the red diamond.
However, \citet{Langosco+2022} found that learned policies tend to pursue the yellow line instead.
\env{Cheese on a dish} is an interpretation of \env{Maze II} using an original JAX implementation. In addition to the differences discussed for \env{cheese in the corner}, we break the symmetry between the dish and the cheese by redundantly coding the dish (since we use Boolean grid observations rather than colored images).

\subsection{The \env{keys and chests} environment}
\label{apx:environments:keys}

\paragraph{Environment.}
In this environment, levels are parameterized by a $13 \times 13$ wall layout, and positions within this grid for the mouse spawn, $k \leq 10$ keys, and $c \leq 10$ chests. We require that the mouse spawn position and the key and chest positions are distinct and are not obstructed by walls. We do not require that there are unobstructed paths between the positions.

In the initial state, the mouse begins in the spawn position.
The actions available to the agent are to attempt to move the mouse up, left, down, or right, which succeeds if the respective position is not obstructed by a wall or the edge of the grid.
If the mouse position equals a key position, it collects the key, removing it from the maze and adding it to the mouse's inventory. If the mouse position equals a chest position, assuming it has at least one key in its inventory, it collects the chest, removing it from the maze and removing the key from its inventory.
The mouse can occupy the same position as a chest if it has an empty inventory.
In a level with $k$~keys and $c$~chests, the episode terminates when $\min(k, c)$ chests have been collected, or after a maximum of 128 steps.

\paragraph{Observations.}
We represent states to the agent as a $15 \times 15 \times 5$ Boolean grid.
The first channel encodes the maze layout, including a border of width~1. The second one-hot encodes the position of the mouse. The third encodes the positions of all the as-yet-uncollected chests. The fourth encodes the positions of all as-yet-uncollected keys. The fifth encodes the mouse's inventory, with one cell along the top row of the channel for each key.
All of the relevant information about the level and the state is encoded into this observation, therefore this environment is fully observable.

\paragraph{True goal and proxy goal.}
The training goal is to collect chests. The reward function assigns $+1$ reward to transitions in which the mouse collects a chest. This reward function is not normalized like those we consider in the theory sections, but it is bounded and could easily be normalized.
Under this goal, collecting keys has no intrinsic value, but since collecting a chest requires collecting a key, keys have instrumental value. We also consider a proxy goal that assigns reward for collecting keys as well as chests. This goal could be modeled as a reward function that assigns, for example, $1-\eta$ reward for collecting each key and $\eta$ reward for collecting each chest, where $\eta \in (0,1)$. Note that we never train with such a proxy goal as the reward function.

\paragraph{Level classification.}
Compared to the other environments, describing this environment in terms of the definitions in \cref{sec:setting} is not as straight forward.
Optimal behavior involves collecting keys and chests in an order that depends on subtle tradeoffs driven by the exponential discounting. For the true goal, the optimal behavior is to collect as many chests as fast as possible. However, it may make sense to make a brief detour to collect multiple keys if it slightly delays the collection of the next chest, as long as this sufficiently accelerates the collection of subsequent chests.
The proxy goal rewards key collection for its own sake, so as to increase the incentive for the policy to take larger and larger detours to collect keys and even collect more keys than there are reachable chests.
In our experiments, we mainly consider the following kinds of levels.
\begin{enumerate}
    \item
        \textbf{Most levels with 3 keys and 10 chests are approximately non-distinguishing.}
        Consider a level in which there are $k \approx 3$ reachable keys, and $c \approx 10$ reachable chests. These levels are approximately non-distinguishing, in the sense that while the proxy goal incentivizes collecting keys earlier than optimal, it still incentivizes eventually collecting chests. For many key/chest layouts, pursuing the proxy goal entails similar behavior to pursing the true goal.
    
    \item
        \textbf{Most levels with 10 keys and 3 chests are approximately distinguishing.}
        Consider a level in which there are $k \approx 10$ reachable keys, and $c \approx 3$ reachable chests.
        These levels are mostly distinguishing because optimizing the proxy goal will usually involve long detours to collect all reachable keys, delaying the collection of chests compared to optimizing the true goal.
\end{enumerate}
The exact classifications depend on the key/chest layout.
Keys and chests that are unreachable from the mouse spawn position have no influence on optimal behavior.
For the purpose of measuring the proportion of distinguishing levels, we approximately classify the level by the total number of keys and chests in the level without accounting for reachability or the exact key/chest layout (non-distinguishing levels have 3 keys and 10 chests, distinguishing levels have 10 keys and 3 chests).

\paragraph{Procedural level generation.}
We construct two procedural level generators, $\distrNonDistg, \distrDistg \in \LevelDistrs$, (approximately) concentrated on non-distinguishing and distinguishing levels, respectively.
\begin{itemize}
    \item \textbf{Non-distinguishing level distribution \textnormal{($\distrNonDistg$)}.}
        We procedurally generate \emph{mostly} non-distinguishing levels as follows. For each position, we place a wall independently with probability $25\%$. We then position the mouse spawn, 3 keys, and 10 chests in distinct, unobstructed positions (assuming there are enough positions).
        The generated levels are usually approximately non-distinguishing in the sense of the classification system described above. It is possible that fewer than 3 keys and 10 chests will be reachable from the mouse spawn position, and it is even possible that the exact layout will lead to a significant disadvantage for a policy that over-prioritizes key collection.
        
    \item \textbf{Distinguishing level distribution \textnormal{($\distrDistg$)}.}
        We procedurally generate \emph{mostly} distinguishing levels in the same fashion, but with 10 keys and 3 chests instead.
        The generated levels are usually distinguishing, along the lines of the classification system described above. It is possible that fewer than 10 keys and 3 chests will be reachable from the mouse spawn position, and it is possible that the exact layout will not substantially disincentivize key collection (for example, all keys may be positioned on a shortest path through the set of chests).
\end{itemize}

\paragraph{Elementary edit distributions.}
ACCEL additionally requires specifying an edit distribution used to sample ``similar'' levels for potential entry into the level buffer.
In \cref{apx:mutators}, we discuss how we build our edit distributions from elementary edit distributions of the following three types.
\begin{enumerate}
    \item \textbf{Classification preserving edits.}
        Each classification preserving edit either removes an existing wall, positions a new wall, or moves the mouse spawn position or the position of one key or chest to a random unobstructed, unoccupied position. These edits don't change the number of keys or chests, though they may change whether these positions are reachable from the mouse spawn position.
        As a result, they may change the exact classification of the level (though not its approximate classification).
    \item \textbf{Biased classification transforming edits.}
        Given a probability $\alpha$, a biased classification transforming edit sets the number of keys and chests to that of the distinguishing level generator (10 and 3) with probability~$\alpha$, or that of the non-distinguishing level generator (3 and 10) with probability $1-\alpha$.
    \item \textbf{Unrestricted classification transforming edits.}
        An unrestricted classification transforming edit is a biased classification transforming edit with $\alpha$ set to $50\%$.
\end{enumerate}

\paragraph{Oracle maximum return.}
As described above, an optimal policy in this environment collects keys and chests in the some order so as so collect as many chests as fast as possible. In particular, noting that at most $m = \min(k,c)$ chests can be collected in a given level (due to the requirement that a key must be expended to collect a chest), and supposing that they are collected after steps $s_1, s_2, \ldots, s_m \in \mathbb{N} \cup \{\infty\}$, the return is given by
\begin{equation}\label{eq:keys-chests-return}
    R(s_1, s_2, \ldots, s_m) = \sum_{i=1}^m \gamma^{s_i-1}
\end{equation}
where $\gamma \in (0,1)$ is the discount factor.

Our approach to computing the optimal value is a to enumerate a subset of paths through the network of key/chest/mouse spawn positions that must contain an optimal path, to brute-force evaluate a lower bound on the return of these paths, and to identify the greatest return lower bound as the optimal return for the level.
We explain the procedure in detail as follows.
\begin{enumerate}
    \item 
        We begin by enumerating a set of so-called \emph{viable} collection sequences. Each collection sequences comprising an $m$-permutation of the $k$ keys, an $m$-permutation of the $c$ chests, and a Dyck path of order $m$ (a permutation of $m$ keys and $m$ chests such that each chest is preceded by a corresponding key, cf.~balanced parentheses). These three combinatorial objects jointly identify a sequence in which particular keys and chests could be collected.
        
        For example, suppose $k=3$ and $c=10$ and number the keys
            $\mathrm{k}_1, \ldots, \mathrm{k}_3$ and the chests
            $\mathrm{c}_1, \ldots, \mathrm{c}_{10}$.
        Suppose the key $3$-permutation is $(3~1~2)$, the chest $3$-permutation is $(1~6~4)$, and the Dyck path is $(\mathrm{k}~\mathrm{c}~\mathrm{k}~\mathrm{k}~\mathrm{c}~\mathrm{c})$. Then the corresponding collection sequence is $\mathrm{k}_3, \mathrm{c}_1, \mathrm{k}_1, \mathrm{k}_2, \mathrm{c}_6, \mathrm{c}_4$.
        
        The total number of viable collection sequences is
        \begin{equation}\label{eq:keys-chests-combinatorics}
            \underbrace{\binom{k}{m} \cdot m!}_{\text{key $m$-permutations}}
            \times
            \underbrace{\binom{c}{m} \cdot m!}_{\text{chest $m$-permutations}}
            \times
            \quad
            \underbrace{\frac{1}{m+1}\binom{2m}{m}}_{\text{order $m$ Dyck paths}}.
        \end{equation}
        Since, in our setup, either $k=3$ and $c=10$ or vice versa, $m=3$ and \cref{eq:keys-chests-combinatorics} evaluates to $21,\!600$.

    \item
        We evaluate a lower bound on the return of each viable collection sequence as follows. First we compute all-pairs shortest path distances for the mouse spawn position, each key position, and each chest position, using the Floyd--Warshall algorithm. We represent unreachable pairs with a distance of $\infty$. We then simulate each sequence, computing the step counts required for each collection as the cumulative sum of pairwise shortest path distances for each transition along the sequences, starting at the mouse spawn position. We round any distances above $128$ up to $\infty$. We then use the step counts of each of the $m$ chest collections in the expression~\cref{eq:keys-chests-return}.

        For example, consider the collection sequence described earlier, $\mathrm{k}_3, \mathrm{c}_1, \mathrm{k}_1, \mathrm{k}_2, \mathrm{c}_6, \mathrm{c}_4$. Let $D(p, q)$ represent the shortest path distance between the positions of objects $p$ and $q$. Then we set
            $s_{\mathrm{c}_1} = D(\mathrm{spawn}, \mathrm{c}_1)$,
            $s_{\mathrm{k}_3} = s_{\mathrm{c}_1} + D(\mathrm{c}_1, \mathrm{k}_3)$,
        and so on until 
            $s_{\mathrm{c}_4} = s_{\mathrm{c}_6} + D(\mathrm{c}_6, \mathrm{c}_4)$.
        (If any of these step counts pass the timeout of $128$, we round them up to $\infty$ to account for termination.)
        Finally, we compute the lower bound on the return of this sequence as 
            $
                \gamma^{s_{\mathrm{c}_1} - 1}
                + \gamma^{s_{\mathrm{c}_6} - 1}
                + \gamma^{s_{\mathrm{c}_4} - 1}
            $.

        We note that if a viable collection sequence ever involves collecting a key or chest that is unreachable from the mouse spawn position, then the step count for this collection and all subsequent collections will be infinite and thus this neither this collection nor subsequent collections will contribute to the return lower bound.

    \item 
        The greatest return lower bound across all viable collection sequences equals the maximum return achievable, as follows.
        Following a shortest path between a given pair of positions may involve crossing over other keys and chests. This can have any of several effects that invalidate the collection sequence, including (1)~collecting keys or chests that appear later in the sequence earlier than planned, (2)~expending keys before they are intended to be spent to collect a chest later in the sequence, or (3)~terminating the episode early due to collecting the maximum available number of chests before finishing the sequence.
        However, such disruptions only ever \emph{increase} the return.
        Moreover, for each disrupted sequence, there is a viable collection sequence that represents the actual order of keys and chests, except for accounting for the case where more than $3$ keys are collected (which has no affect on the return as long as they are on the shortest paths to the necessary chests).
        It follows that for the viable collection sequence with the highest return, there are no such disruptions, and the return lower bound is tight.
\end{enumerate}
We note that the set of viable collection sequences could be further filtered by eliminating (or never enumerating) sequences involving unreachable keys or chests. 
However, we use the above approach for simplicity and uniformity. In particular, since the above approach yields a fixed computational graph given values of $k$, $c$, we can enumerate the $21,\!600$ viable sequences once at compile time and accelerate the brute-force evaluation step for a particular level (or batch of levels) using JAX. To handle a mixture of levels with $(k,c)=(3,10)$ and levels with $(k,c)=(10,3)$, we simply evaluate both ways and then dynamically keep the appropriate result for each level.

\paragraph{Origin.}
\citet{Barnett2019} originated the idea of creating a distribution shift from a navigation environment in which keys are scarce to one in which keys are not scarce.
\citet{Langosco+2022} originally implemented a \env{Keys and Chests} environment by modifying the \env{Heist} environment from OpenAI ProcGen \citep{Cobbe+2020}.
Our \env{keys and chests} environment is an interpretation of the environment from \citet{Langosco+2022}, implemented in JAX, with several changes similar to those for the other environments.

\clearpage

\section{Additional training details}
\label{apx:training-details}

\subsection{Hyperparameters}

\begin{table}[h!]
\caption{\label{table:hyperparams} Hyperparameters used for all methods and environments.}
\begin{center}
\begin{tabular}{lrl}
\toprule
\textbf{Parameter}          & Value     & Notes/exceptions \\
\midrule
\emph{Rollouts}                         \\
\# parallel environments    & 256       \\
Rollout length              & 128       \\
\# environment steps per cycle
                            & 32,768    & (\# parallel environments $\times$ rollout length) \\
Discount factor, $\gamma$   & 0.999     \\

\midrule
\emph{GAE} \\
$\lambda_{\text{GAE}}$      & 0.95      \\

\midrule
\emph{PPO loss function} \\
Clip range                  & 0.1       \\
Value clipping              & yes       \\
Critic coefficient          & 0.5       \\
Entropy coefficient         & 1e-3      & Or, 1e-2 for \env{keys and chests} \\

\midrule
\emph{PPO updates} \\
Epochs per cycle            & 5         \\
Minibatches per epoch       & 4         \\
Max gradient norm           & 0.5       \\
Adam learning rate          & 5e-5      \\
Learning rate schedule      & constant  \\

\midrule
\emph{UED configuration}    &           & N/A for DR \\
Replay rate                 \\
    ~~\RobustPLR{}          &  0.33     & On average, 1 replay cycle per
                                          2 generate cycles \\
    ~~ACCEL                 &  0.5      & On average, 1 replay cycle per
                                          1 generate cycle \\
                            &           & (1 edit cycle immediately follows each 
                                          replay cycle) \\
Buffer size                 &  4096     \\
Prioritization method       &  rank     \\
Temperature                 &   0.1     \\
Staleness coefficient       &   0.1     \\
\# elementary edits per level, $n$
                            &    12     & N/A for \RobustPLR{} \\

\bottomrule
\end{tabular}
\end{center}
\end{table}

\subsection{Compute}

We perform each training run on a single NVIDIA A100 80GB GPU.
For \env{cheese in the corner} and \env{cheese on a dish}, each training run takes around
    40~minutes (DR)
or
    80~minutes (UED methods).
For \env{keys and chests}, each training run takes around
    60~minutes (DR)
or 
    110~minutes (UED methods).
UED methods take longer than DR because UED methods require sampling additional rollouts for refining the buffer (the number of environment steps used for PPO updates is held constant across all methods).
\env{Keys and chests} runs are longer than others because we trained each method for 400 million steps instead of 200~million steps in this environment.

The experiments discussed in \cref{sec:results} took a total of
    1.2k~GPU hours.
The additional experiments discussed in Appendices
    \ref{apx:mutators},
    \ref{apx:maximin},
    \ref{apx:moresteps},
    \ref{apx:robustness-corner},
and
    \ref{apx:robustness-dish}
took, respectively, totals of 
    1.2k~GPU hours;
    500~GPU hours; 
    400~GPU hours (not counting the first 200~million environment steps used for training, which we counted with \cref{sec:results});
    210~GPU hours;
and
    350~GPU hours.

\clearpage

\section{Additional evaluation results (non-distinguishing levels and proxy goal)}
\label{apx:non-distinguishing}

In \cref{sec:results}, we investigate which training distributions and methods led to good performance on distinguishing levels. We claim that when the performance is low, this is an instance of goal misgeneralization, and therefore when the performance is high, goal misgeneralization has been prevented.

In order to justify this claim, we also check that the poor performance is explained primarily by the policy pursuing a proxy goal on distinguishing levels, rather than the policy behaving incapably in non-distinguishing or distinguishing levels. \Cref{fig:main-extra} shows (1)~return on non-distinguishing levels and (2)~proxy return on distinguishing levels for each training configuration.

\Cref{fig:main-extra}(top row) shows that the learned policies perform capably in non-distinguishing levels. For \env{cheese in the corner} and \env{cheese on a dish}, all training methods achieve high return on non-distinguishing levels for all training distributions. For \env{keys and chests}, this is the case for all training distributions except the $\alpha=1$ baseline (in this edge case, non-distinguishing levels, with few keys and many chests, are never seen during training).

\Cref{fig:main-extra}(bottom row) shows the proxy return achieved by learned policies on distinguishing levels. Note that we never use the proxy goal for training. Here we simply evaluate the policies according to each environment's respective proxy reward function.
In particular, for \env{cheese in the corner}, we use a reward function that assigns $+1$ reward the first time the mouse reaches the corner. For \env{cheese on a dish}, we use a reward function that assigns $+1$ reward if the mouse collects the dish. For \env{keys and chests}, it's difficult to define a proxy goal (see \cref{apx:environments:keys}), here we report the average number of keys in the mouse's inventory throughout the rollouts (note that distinguishing levels have at most 3 reachable chests, and so carrying more than three keys suggests the agent is over-prioritizing key collection).
The trends in proxy return mirror the trends in true return displayed in \cref{fig:main-distinguishing-return}, suggesting that our learned policies retain enough capabilities in distinguishing levels to pursue the proxy goal---a case of goal misgeneralization.

\begin{figure}[h!]
    \centering
    \begingroup
    \setlength{\tabcolsep}{1pt}
    \begin{tabular}{cccc}
    \toprule
    & \bf \env{Cheese in the corner}
    & \bf \env{Cheese on a dish}
    & \bf \env{Keys and chests}
    \\
    & \small Seeds $N{=}8$, steps $T{=}200$M
    & \small Seeds $N{=}3$, steps $T{=}200$M
    & \small Seeds $N{=}5$, steps $T{=}400$M
    \\
    \midrule
    \rotatebox[origin=l]{90}{Avg.\ ret.\ (ambig.)}
    & \includegraphics{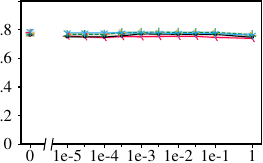}
    & \includegraphics{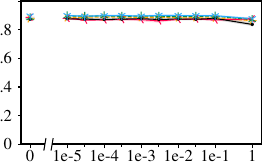}
    & \includegraphics{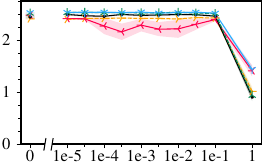}
    \\
    \rotatebox[origin=l]{90}{Proxy ret.\ (distg.)}
    & \includegraphics{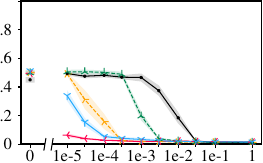}
    & \includegraphics{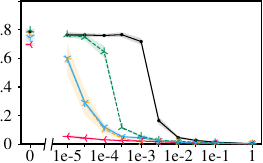}
    & \includegraphics{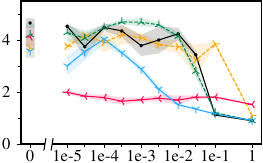}
    \\
    & \multicolumn{3}{c}{Proportion of distinguishing levels in underlying training distribution, $\alpha$ (augmented log scale)}
    \\
    \midrule
    & \multicolumn{3}{c}{\includegraphics{figures/plots/legends/main-both.pdf}}
    \\[-0.5ex]
    \bottomrule
    \end{tabular}
    \endgroup
    \caption{\label{fig:main-extra}%
        \textbf{Performance on non-distinguishing levels and with respect to the proxy goal.}
        Each policy is trained on $T$ environment steps using the indicated training method with underlying training distribution $\distrTrain[\alpha]$.
        \textit{(1\textsuperscript{st} row):}
            Average return over 512 steps for an evaluation batch of 256 non-distinguishing levels.
        \textit{(2\textsuperscript{nd} row):}
            Average \emph{proxy} return over 512 steps for an evaluation batch of 256 distinguishing levels.
            Note that the proxy goal is never used for training.
        \textit{(Both):}
            Mean over $N$ seeds, shaded region is one standard error.
            Note the split axes used to show zero on the log scale.
    }
\end{figure}

\clearpage

\section{Visualizing performance on levels with different cheese positions}
\label{apx:heatmaps}

For each training configuration (training distribution, training method) studied in the main text for \env{cheese in the corner}, we save the policy from the end of training for the first of 8 training seeds. In order to visualize the robustness of these policies to varying changes in the cheese position, we create a batch of 122 levels with a shared wall layout and a fixed mouse position, but where each level in the batch has a different cheese position. For each level, we sample 512 environment steps from each policy and compute the average per-episode return as a measure of the policy's performance in that level. This gives us a vector of 122 average return values for each policy (one for each level), which we visualize in a policy-specific heatmap such that the average return of the policy in a level with the cheese in position $(i,j)$ is indicated by the color of the cell $(i,j)$ in the heatmap. For context we overlay the wall layout and the mouse spawn position (note that we do not consider levels with cheese positions that would coincide with a wall or with the mouse spawn position).

Heatmaps for each method and training distribution follow in \cref{fig:heatmaps-part1,fig:heatmaps-part2}.
We see a rough progression whereby for more advanced algorithms or higher $\alpha$, the agent is robust to a greater proportion of cheese positions. 
There are also instances of ``blind spots'' indicating cheese positions for which certain methods are not robust, indicating that these training methods do not produce perfectly robust policies.

\begin{figure}[h!]
    \begingroup
    \centering
    \setlength{\tabcolsep}{0.25em}
    \renewcommand\arraystretch{1.12}
    \begin{tabular}{cccccc}
    \toprule
    \null & \bf \multirow{2}{*}{DR}
    & \multicolumn{2}{c}{\bf \RobustPLR{}}
    & \multicolumn{2}{c}{\bf ACCEL}
    \\
    $\alpha$
    & 
    & \small max-latest
    & \small oracle-latest
    & \small max-latest
    & \small oracle-latest
    \\
    \midrule
    \raisebox{\halfheatmapwidth}{$0$}
    & \includegraphics[width=\heatmapwidth]{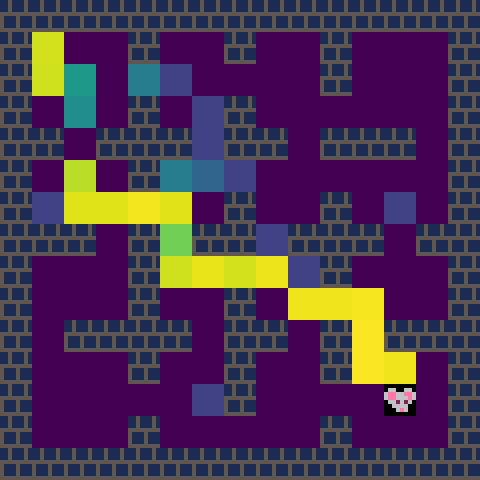}
    & \includegraphics[width=\heatmapwidth]{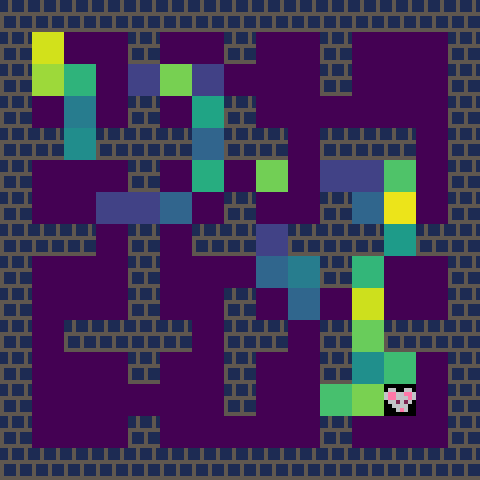}
    & \includegraphics[width=\heatmapwidth]{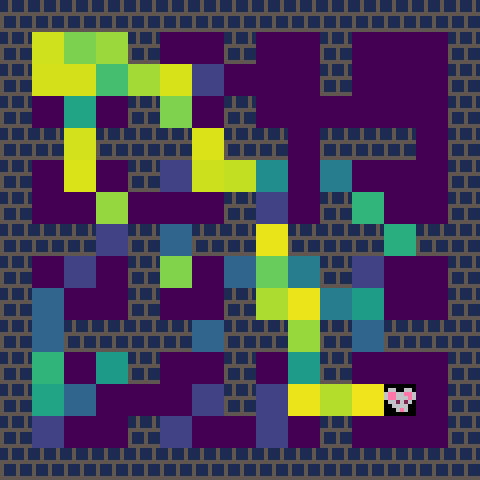}
    & \includegraphics[width=\heatmapwidth]{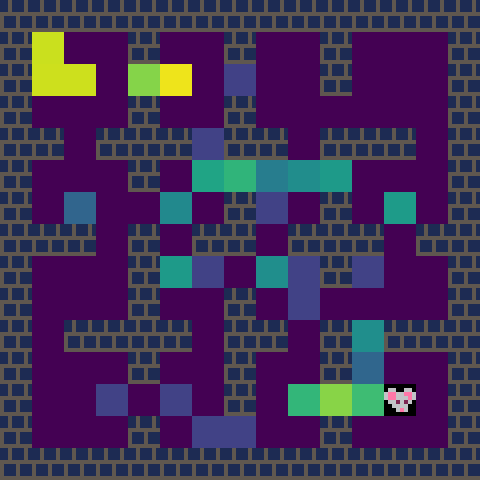}
    & \includegraphics[width=\heatmapwidth]{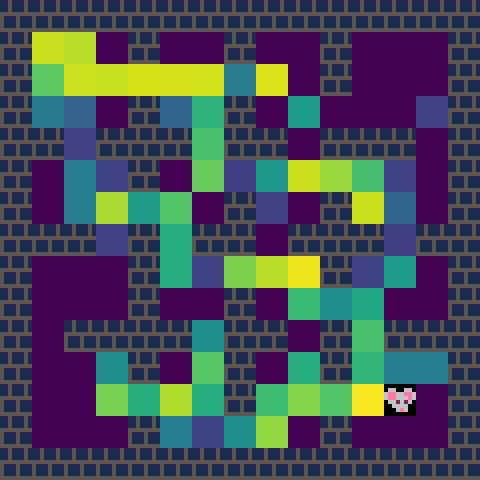}
    \\
    \raisebox{\halfheatmapwidth}{$1\text{e-}5$}
    & \includegraphics[width=\heatmapwidth]{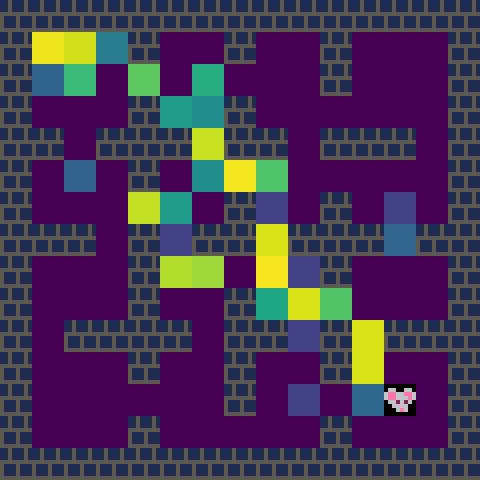}
    & \includegraphics[width=\heatmapwidth]{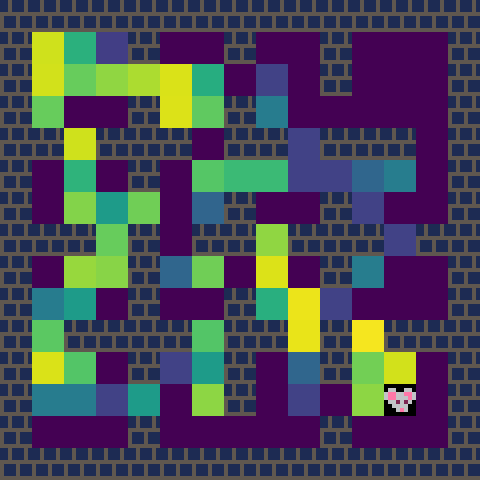}
    & \includegraphics[width=\heatmapwidth]{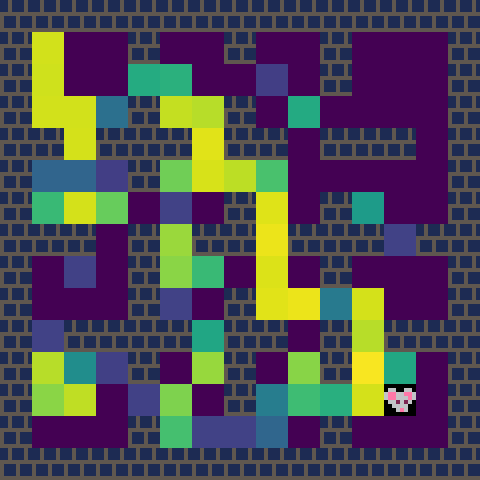}
    & \includegraphics[width=\heatmapwidth]{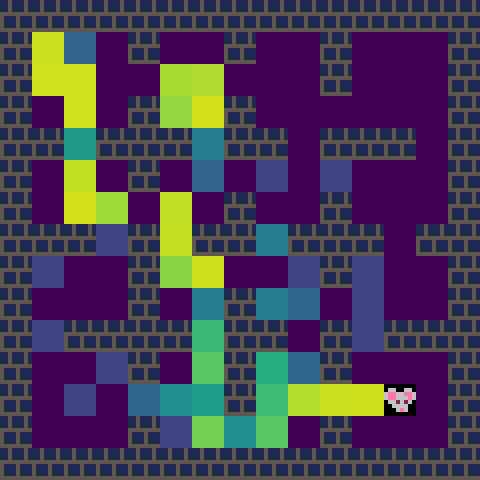}
    & \includegraphics[width=\heatmapwidth]{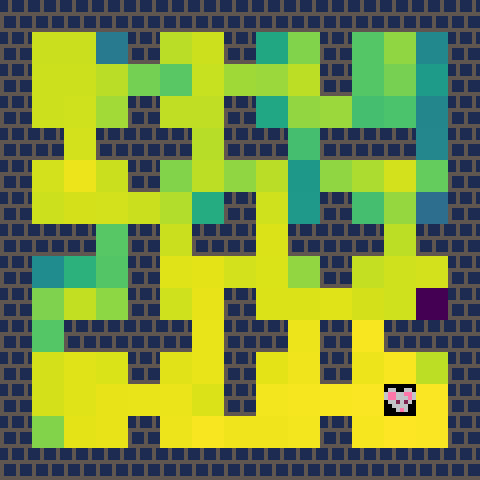}
    \\
    \raisebox{\halfheatmapwidth}{$3\text{e-}5$}
    & \includegraphics[width=\heatmapwidth]{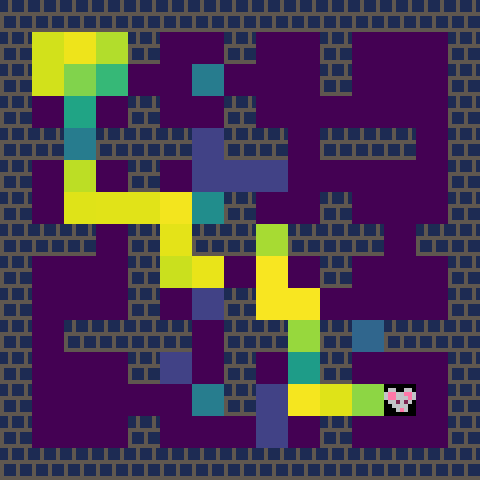}
    & \includegraphics[width=\heatmapwidth]{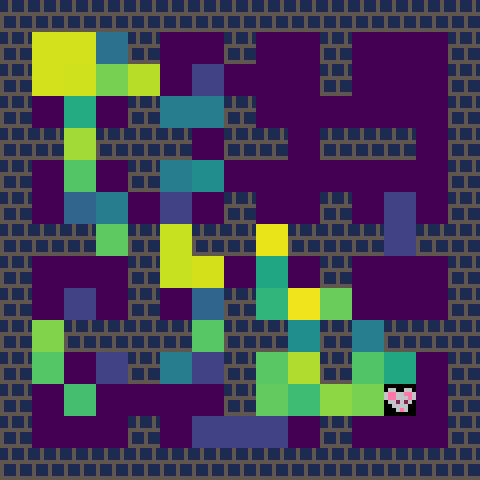}
    & \includegraphics[width=\heatmapwidth]{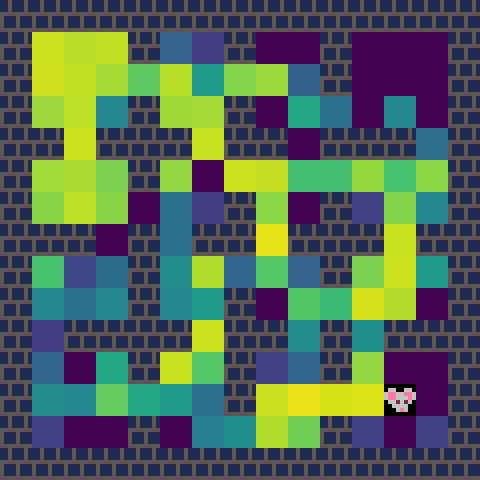}
    & \includegraphics[width=\heatmapwidth]{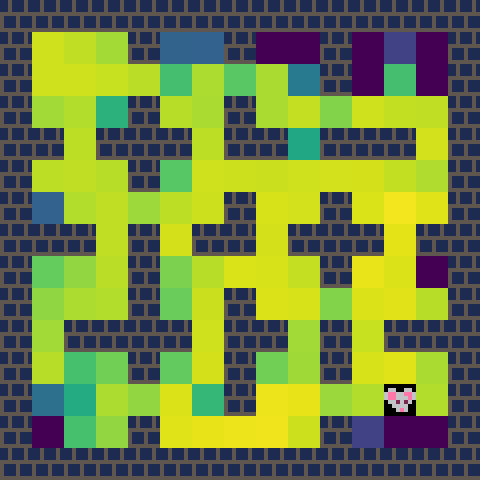}
    & \includegraphics[width=\heatmapwidth]{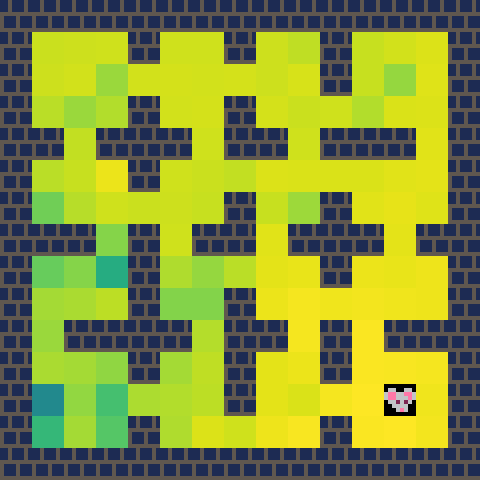}
    \\
    \raisebox{\halfheatmapwidth}{$1\text{e-}4$}
    & \includegraphics[width=\heatmapwidth]{figures/heatmaps/dr-1eneg4.png}
    & \includegraphics[width=\heatmapwidth]{figures/heatmaps/plr-est-1eneg4.png}
    & \includegraphics[width=\heatmapwidth]{figures/heatmaps/plr-ora-1eneg4.png}
    & \includegraphics[width=\heatmapwidth]{figures/heatmaps/accelid-est-1eneg4.png}
    & \includegraphics[width=\heatmapwidth]{figures/heatmaps/accelid-ora-1eneg4.png}
    \\
    \midrule
    & \multicolumn{5}{c}{\includegraphics[trim={0pt 5pt 0pt 0pt}]{figures/plots/legends/colorbar-viridis.pdf}}
    \\
    & \multicolumn{5}{c}{Average return}
    \\
    \bottomrule
    \end{tabular}
    \endgroup
    \caption{\label{fig:heatmaps-part1}%
        Heatmap visualizations (part 1 of 2). See \cref{fig:heatmaps,apx:heatmaps} for details.
    }
\end{figure}

\begin{figure}
    \begingroup
    \centering
    \setlength{\tabcolsep}{0.25em}
    \renewcommand\arraystretch{1.12}
    \begin{tabular}{cccccc}
    \toprule
    \null & \bf \multirow{2}{*}{DR}
    & \multicolumn{2}{c}{\bf \RobustPLR{}}
    & \multicolumn{2}{c}{\bf ACCEL}
    \\
    $\alpha$
    & 
    & \small max-latest
    & \small oracle-latest
    & \small max-latest
    & \small oracle-latest
    \\
    \midrule
    \raisebox{\halfheatmapwidth}{$3\text{e-}4$}
    & \includegraphics[width=\heatmapwidth]{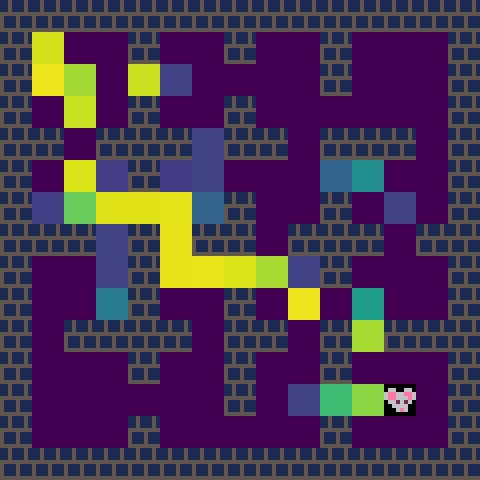}
    & \includegraphics[width=\heatmapwidth]{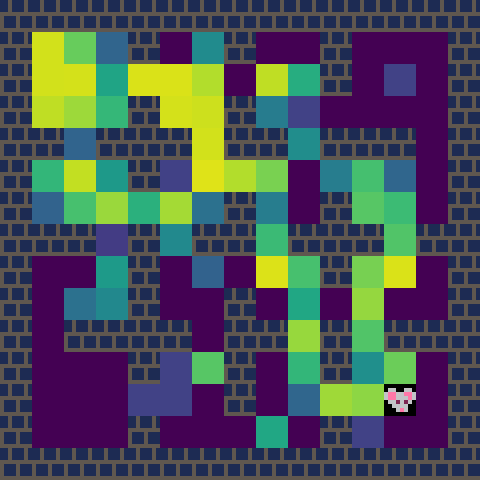}
    & \includegraphics[width=\heatmapwidth]{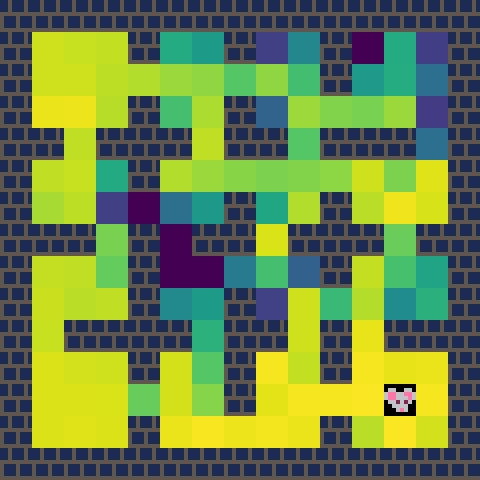}
    & \includegraphics[width=\heatmapwidth]{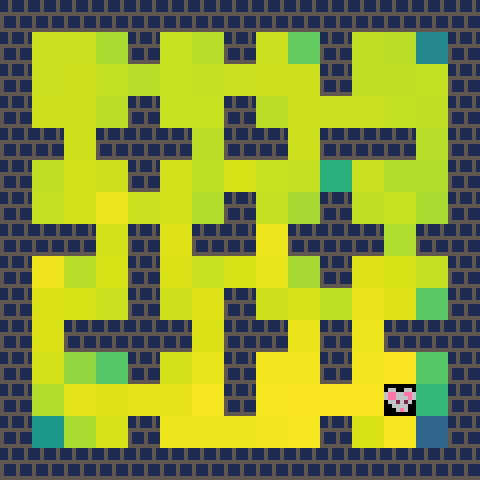}
    & \includegraphics[width=\heatmapwidth]{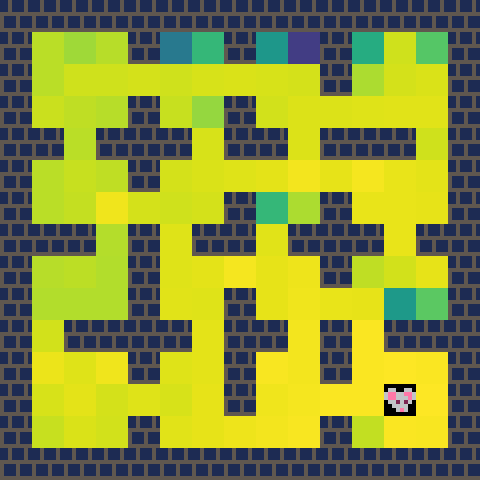}
    \\
    \raisebox{\halfheatmapwidth}{$1\text{e-}3$}
    & \includegraphics[width=\heatmapwidth]{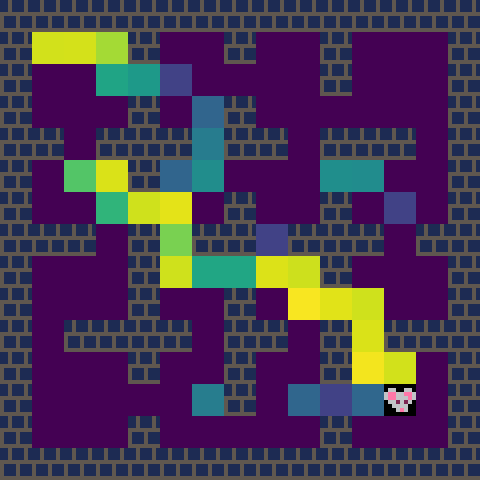}
    & \includegraphics[width=\heatmapwidth]{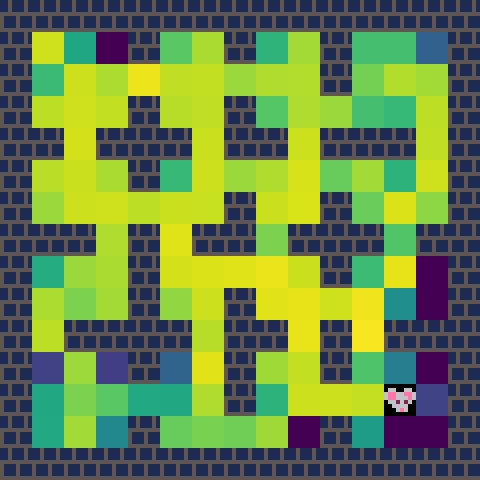}
    & \includegraphics[width=\heatmapwidth]{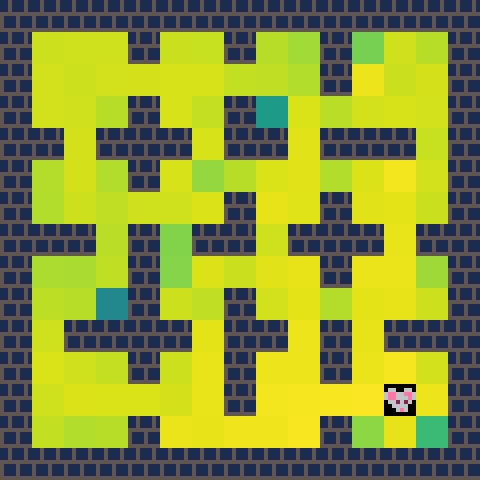}
    & \includegraphics[width=\heatmapwidth]{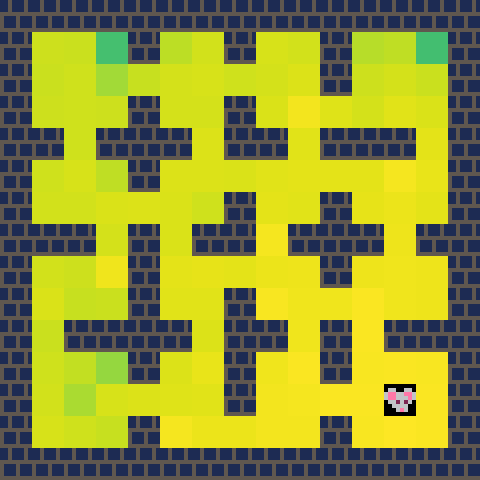}
    & \includegraphics[width=\heatmapwidth]{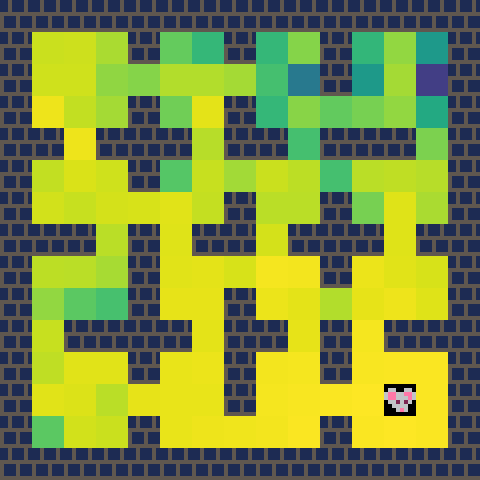}
    \\
    \raisebox{\halfheatmapwidth}{$3\text{e-}3$}
    & \includegraphics[width=\heatmapwidth]{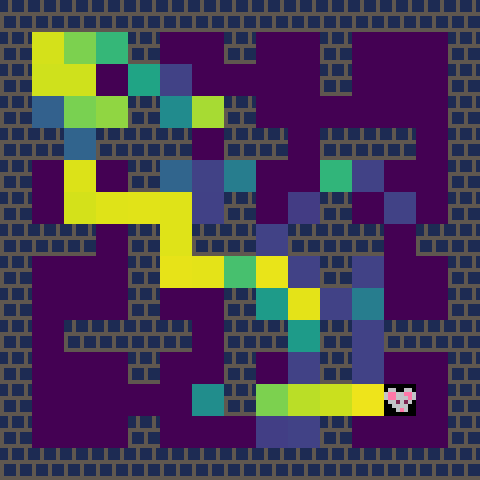}
    & \includegraphics[width=\heatmapwidth]{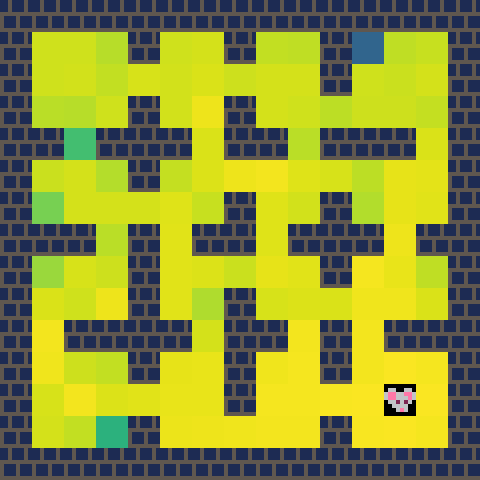}
    & \includegraphics[width=\heatmapwidth]{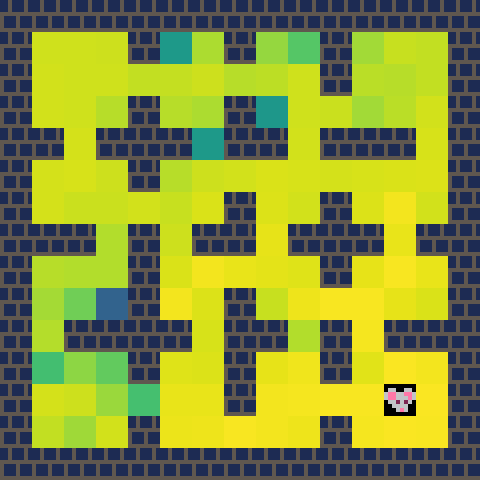}
    & \includegraphics[width=\heatmapwidth]{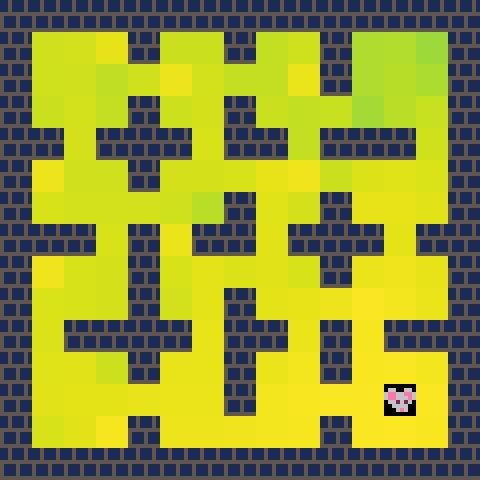}
    & \includegraphics[width=\heatmapwidth]{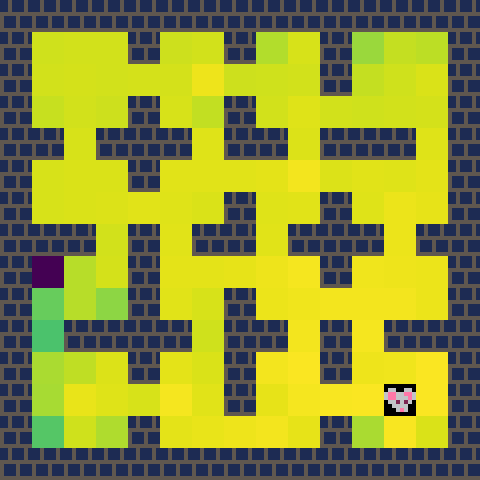}
    \\
    \raisebox{\halfheatmapwidth}{$1\text{e-}2$}
    & \includegraphics[width=\heatmapwidth]{figures/heatmaps/dr-1eneg2.png}
    & \includegraphics[width=\heatmapwidth]{figures/heatmaps/plr-est-1eneg2.png}
    & \includegraphics[width=\heatmapwidth]{figures/heatmaps/plr-ora-1eneg2.png}
    & \includegraphics[width=\heatmapwidth]{figures/heatmaps/accelid-est-1eneg2.png}
    & \includegraphics[width=\heatmapwidth]{figures/heatmaps/accelid-ora-1eneg2.png}
    \\
    \raisebox{\halfheatmapwidth}{$3\text{e-}2$}
    & \includegraphics[width=\heatmapwidth]{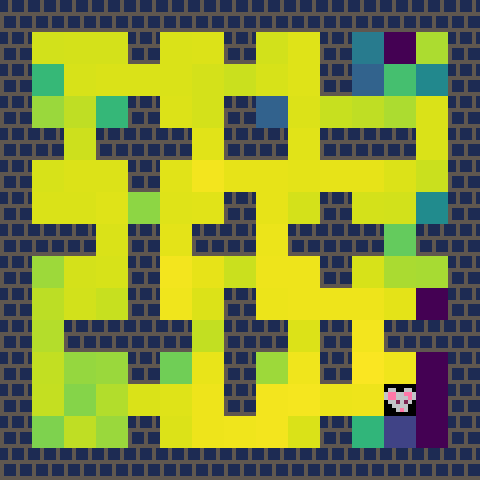}
    & \includegraphics[width=\heatmapwidth]{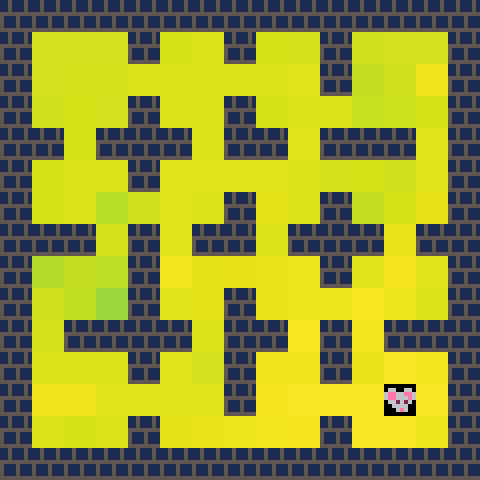}
    & \includegraphics[width=\heatmapwidth]{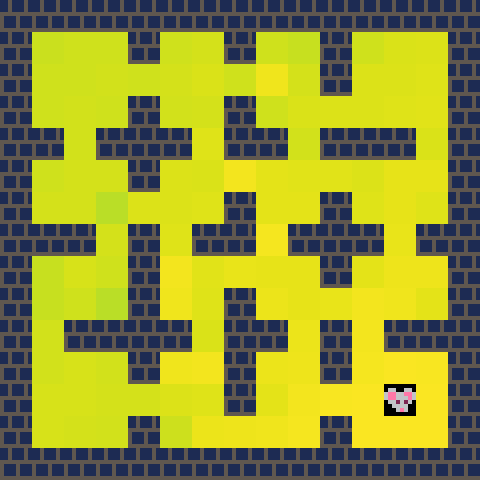}
    & \includegraphics[width=\heatmapwidth]{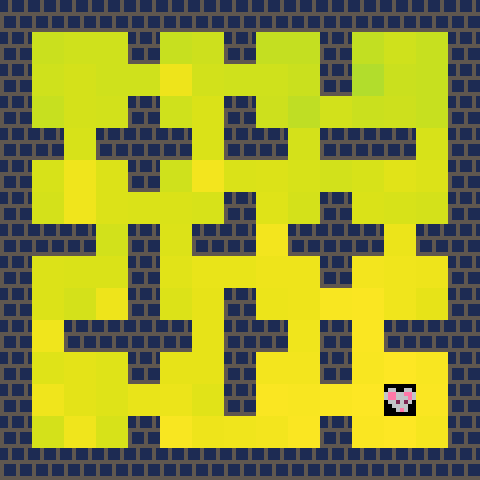}
    & \includegraphics[width=\heatmapwidth]{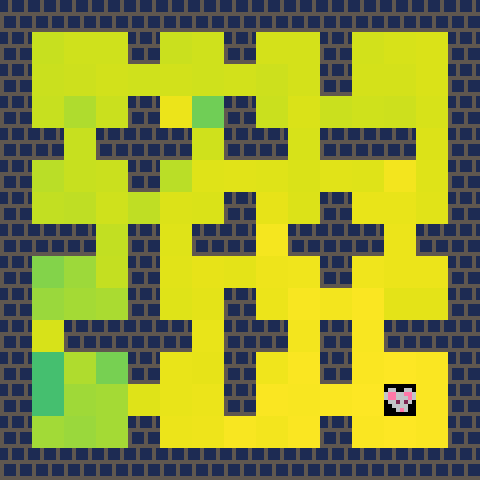}
    \\
    \raisebox{\halfheatmapwidth}{$1\text{e-}1$}
    & \includegraphics[width=\heatmapwidth]{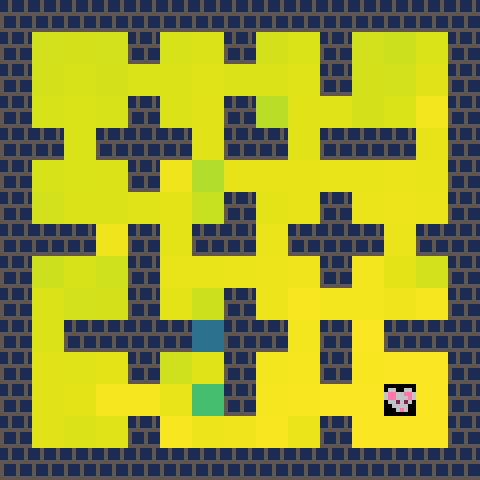}
    & \includegraphics[width=\heatmapwidth]{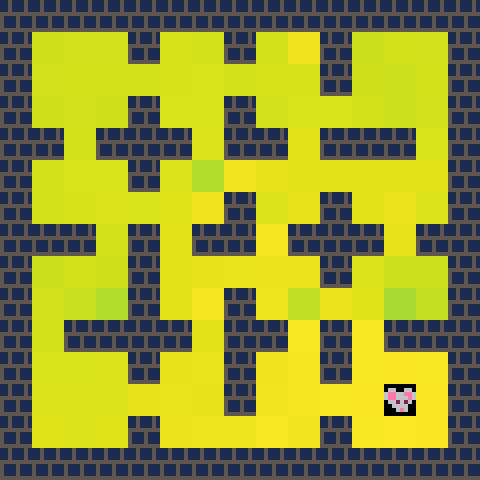}
    & \includegraphics[width=\heatmapwidth]{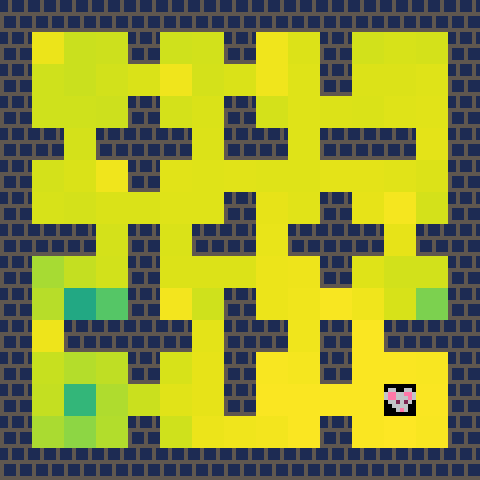}
    & \includegraphics[width=\heatmapwidth]{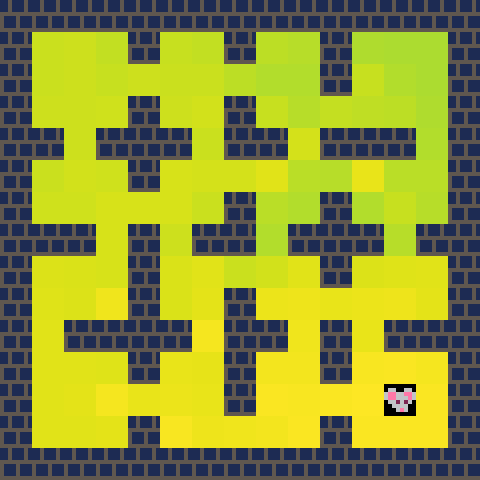}
    & \includegraphics[width=\heatmapwidth]{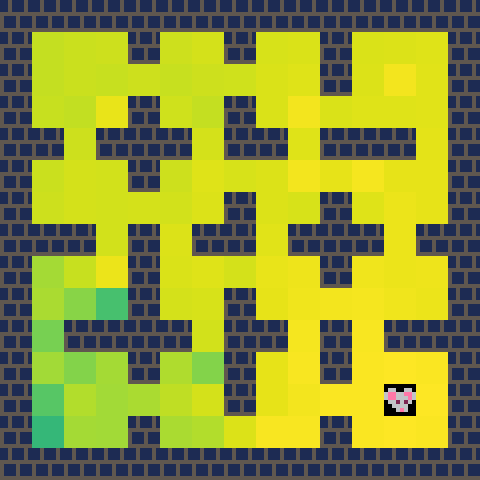}
    \\
    \raisebox{\halfheatmapwidth}{$1$}
    & \includegraphics[width=\heatmapwidth]{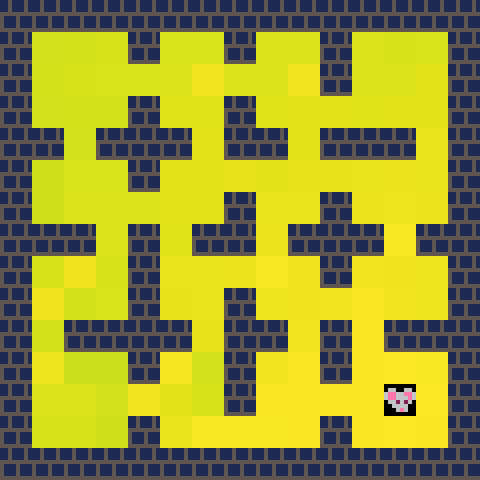}
    & \includegraphics[width=\heatmapwidth]{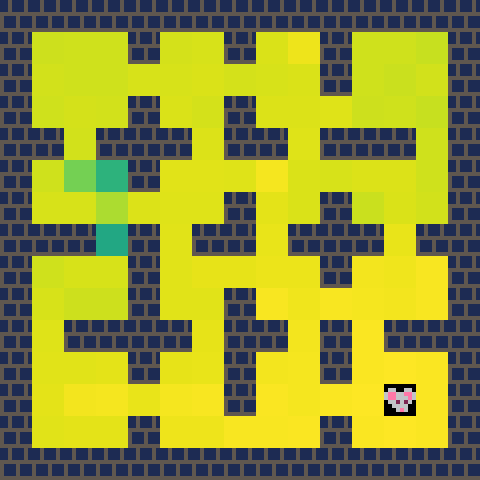}
    & \includegraphics[width=\heatmapwidth]{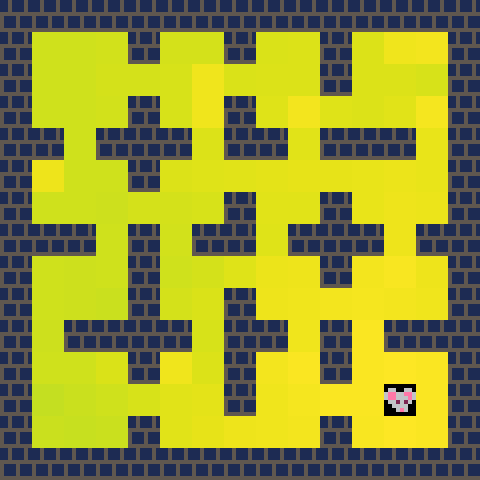}
    & \includegraphics[width=\heatmapwidth]{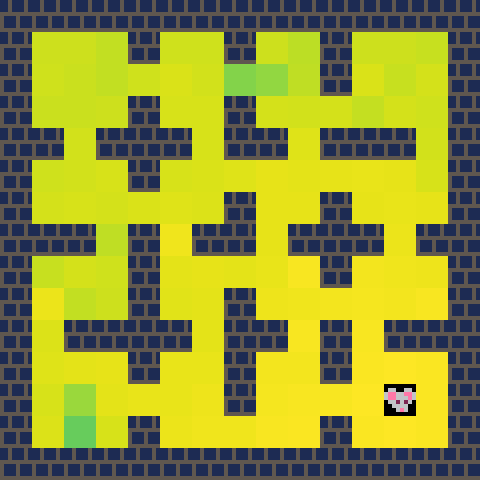}
    & \includegraphics[width=\heatmapwidth]{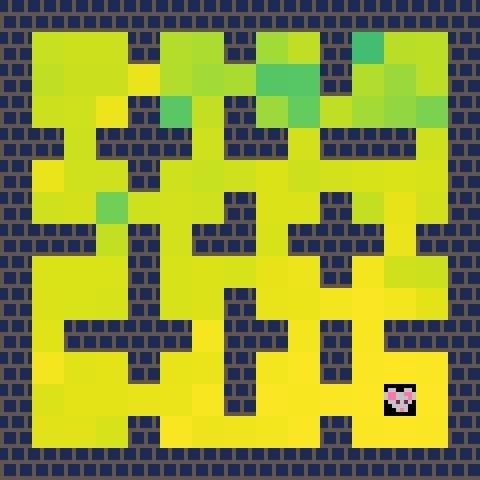}
    \\
    \midrule
    & \multicolumn{5}{c}{\includegraphics[trim={0pt 5pt 0pt 0pt}]{figures/plots/legends/colorbar-viridis.pdf}}
    \\
    & \multicolumn{5}{c}{Average return}
    \\
    \bottomrule
    \end{tabular}
    \endgroup
    \caption{\label{fig:heatmaps-part2}%
        Heatmap visualizations (part 2 of 2). See \cref{fig:heatmaps,apx:heatmaps} for details.
    }
\end{figure}

\clearpage

\section{Experiments with different edit distributions}
\label{apx:mutators}

As discussed in \cref{sec:algorithms}, ACCEL requires additionally specifying an \textbf{edit distribution} for mutating levels in the buffer. In this appendix, we explore the effect of different edit ditributions on the ability for ACCEL to mitigate goal misgeneralization.

\subsection{Elementary edits}

We assemble our edit distributions by sampling a sequence of \textbf{elementary edits} of the following three kinds.
\begin{itemize}
    \item \textbf{Classification preserving edits.}
        These edits change the level without changing whether the level is non-distinguishing or distinguishing.
        For example, in \env{cheese in the corner}, such an edit may randomly toggle a wall or move the mouse's starting position, but would not change the location of the cheese.
    \item 
        \textbf{Biased classification transforming edits.}
        These edits transform the level into a distinguishing level with probability $\alpha$ or an non-distinguishing level with probability $1-\alpha$,
        where $\alpha$ is the proportion of distinguishing levels in the underlying training distribution.
        For example, in \env{cheese in the corner}, a biased classification transforming edit may randomize the cheese position with probability $\alpha$ or move it to the corner with probability $1-\alpha$.
    \item \textbf{Unrestricted classification transforming edits.}
        These edits transformers the position of the cheese or the number of keys and chests uniformly at random given the level parameterization.
        For \env{cheese in the corner} and \env{cheese on a dish}, the cheese moves to a random position in the maze, which usually results in a distinguishing level.
        For \env{keys and chests}, we flip a coin to make either keys or chests sparse, and the other type of object dense.
\end{itemize}
We document the elementary edit distributions for each environment in full detail in \cref{apx:environments}.

\subsection{ACCEL variants}

In these terms, we list the ACCEL variant considered in the main text along with three additional variants of ACCEL with different edit distributions. We fix a hyperparameter $n$, the number of elementary edits to apply to each level (we use $n=12$).
\begin{enumerate}
    \item
        \textbf{Identity ACCEL \textnormal{(\ACCELid{}, simply ``ACCEL'' in main text)}.} 
        We make a sequence of $n$ random edits, all of which are classification preserving, resulting in the sequence of edits itself being classification preserving.
    \item
        \textbf{Constant ACCEL \textnormal{(\ACCELc{})}.}
        We make $n-1$ random classification preserving edits, followed by one biased classification transforming edit.
        Applying this operation to any distribution of levels results in a distribution with the same proportion of non-distinguishing and distinguishing levels as the underlying training distribution.
    \item
        \textbf{Binomial ACCEL \textnormal{(\ACCELbin{})}.}
        We make a sequence of $n$ random edits, each independently chosen to be either classification preserving (with probability $1-1/n$) or else biased classification transforming.
        If the sequence has only classification preserving edits (probability $(1-1/n)^n$) then it is classification preserving (like \ACCELid{}), otherwise the output is non-distinguishing with the same probability as a level sampled from the underlying training distribution (like \ACCELc{}).
    \item
        \textbf{Unrestricted ACCEL \textnormal{(\ACCELunr{})}.}
        We make a sequence of $n-1$ random classification preserving edits, followed by one unrestricted classification transforming edit.
        Applying this operation to any distribution of levels results in a distribution with a proportion of distinguishing levels that is independent of the parameter $\alpha$ that restricts access to distinguishing levels in the underlying training distribution.
\end{enumerate}

\ACCELc{} simulates restricted access to distinguishing levels. This variant is able to introduce new distinguishing levels through edit operations, however, its ability to replicate existing distinguishing levels is limited, since every time it edits a level, the mutated level reverts to an non-distinguishing level with probability $1-\alpha$.

\ACCELid{} (the variant studied in the main text) simulates a different kind of restriction, whereby we don't allow edits to turn non-distinguishing levels into distinguishing levels. However, we do allow edits to create new distinguishing levels by making similar copies of existing distinguishing levels in the buffer. Through this mechanism, \ACCELid{} can rapidly populate the buffer with distinguishing levels, providing this variant with additional capacity to amplify the small training signal from distinguishing levels (beyond simply curating levels, like in \RobustPLR{} or \ACCELc{}).

\ACCELbin{} samples elementary edits in a different way. The number of biased classification transforming edits included in the sequence follows a binomial distribution.
This means that with around $35\%$ probability, no biased classification transforming edits will be applied, and the edit will resemble an edit from \ACCELid{}.
Otherwise, with around $65\%$ probability, at least one biased classification transforming edit will be applied, and the overall edit will resemble one from \textnormal{\ACCELc{}}.
We thus expect the performance of this variant to be somewhere between that of \ACCELc{} and \ACCELid{}.

\ACCELunr{} is a baseline that simulates a situation where the edit distribution can be designed to explore the space of levels completely independently of the training distribution.
We expect this variant to be able to obtain much stronger performance comparable to using $\alpha=1$ in the training distribution, even when training with $\alpha=0$.

\subsection{Experimental results}

We train with the three new variants and compare performance to DR and \ACCELid{} (from the main text).
We report the results in
    \cref{fig:results-mutators-oracle} (oracle-latest regret estimator)
and
    \cref{fig:results-mutators-estimated} (max-latest regret estimator).
Note that we did not run
    \ACCELc{} with the oracle-latest regret estimator for \env{cheese on a dish},
or
    \ACCELc{} with the max-latest regret estimator for \env{keys and chests}.

Our results are in line with our central claim, that more advanced UED methods are more capable of mitigating goal misgeneralization.
\begin{itemize}
    \item
        As predicted, \ACCELc{} achieves lower performance than \ACCELid{}. This could be explained by the greater flexibility with which \ACCELid{} can amplify distinguishing levels through edits.
    \item 
        Moreover, \ACCELbin{} achieves performance somewhere between that of  \ACCELc{} and \ACCELid{}.
    \item 
        As predicted, \ACCELunr{} is able to populate the buffer with distinguishing levels regardless of the training distribution, even when $\alpha=0$, both with the oracle-latest regret estimator and with the max-latest regret estimator.
\end{itemize}
\ACCELbin{} with max-latest estimation achieves the same low performance as \ACCELid{} in \env{keys and chests} (as observed for \ACCELid{} in \cref{sec:results}).

\begin{figure}
    \centering
    \begingroup
    \setlength{\tabcolsep}{1pt}
    \begin{tabular}{cccc}
    \toprule
    & \bf \env{Cheese in the corner}
    & \bf \env{Cheese on a dish}
    & \bf \env{Keys and chests}
    \\
    & \small Seeds $N{=}8$, steps $T{=}200$M
    & \small Seeds $N{=}3$, steps $T{=}200$M
    & \small Seeds $N{=}5$, steps $T{=}400$M
    \\
    \midrule
    \rotatebox[origin=l]{90}{Avg. return (distg.)}
    & \includegraphics{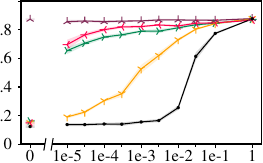}
    & \includegraphics{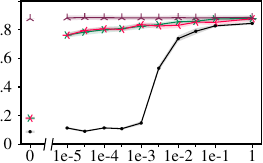}
    & \includegraphics{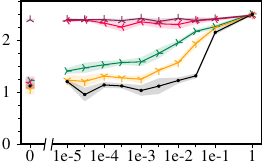} \\
    \midrule
    \multirow{2}{*}[7em]{\rotatebox[origin=l]{90}{Adv.\ prop.\ distg.\ (log)}}
    & \includegraphics{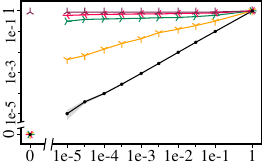}
    & \includegraphics{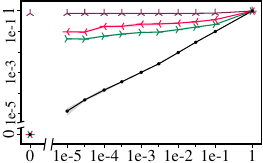}
    & \includegraphics{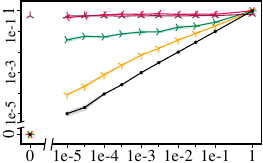} \\
    & \multicolumn{3}{c}{Proportion of distinguishing levels in underlying training distribution, $\alpha$ (augmented log scale)} \\
    \midrule
    & \multicolumn{3}{c}{\includegraphics{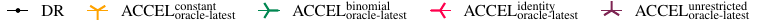}} \\[-0.5ex]
    \bottomrule
    \end{tabular}
    \endgroup
    \caption{\label{fig:results-mutators-oracle}%
        ACCEL variants with oracle-latest estimator.
        Each policy is trained on $T$ environment steps, using the indicated algorithm, with underlying training distribution $\distrTrain[\alpha]$.
        \textit{(1st row):}
            Average return over 512 steps for an evaluation batch of 256 distinguishing levels (cf.~\cref{fig:main-distinguishing-return}).
        \textit{(2nd row):}
            The proportion of distinguishing levels sampled from the adversary's buffer across training (cf.~\cref{fig:main-effective-alpha}).
        \textit{(Both):} 
            Mean over $N$ seeds, shaded region is one standard error.
            Note the split axes used to show zero on the log scale.
    }
\end{figure}

\begin{figure}
    \centering
    \begingroup
    \setlength{\tabcolsep}{1pt}
    \begin{tabular}{cccc}
    \toprule
    & \bf \env{Cheese in the corner}
    & \bf \env{Cheese on a dish}
    & \bf \env{Keys and chests}
    \\
    & \small Seeds $N{=}8$, steps $T{=}200$M
    & \small Seeds $N{=}3$, steps $T{=}200$M
    & \small Seeds $N{=}5$, steps $T{=}400$M
    \\
    \midrule
    \rotatebox[origin=l]{90}{Avg. return (distg.)}
    & \includegraphics{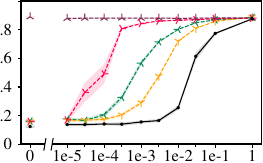}
    & \includegraphics{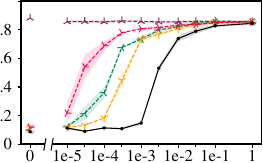}
    & \includegraphics{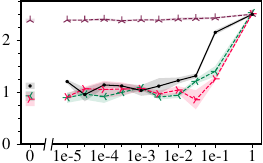} \\
    \midrule
    \multirow{2}{*}[7em]{\rotatebox[origin=l]{90}{Adv.\ prop.\ distg.\ (log)}}
    & \includegraphics{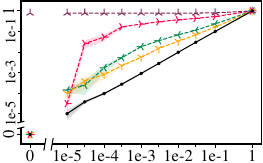}
    & \includegraphics{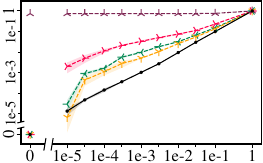}
    & \includegraphics{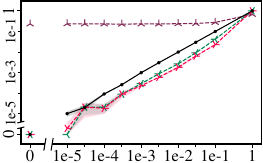} \\
    & \multicolumn{3}{c}{Proportion of distinguishing levels in underlying training distribution, $\alpha$ (augmented log scale)} \\
    \midrule
    & \multicolumn{3}{c}{\includegraphics{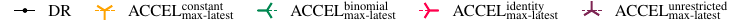}} \\[-0.5ex]
    \bottomrule
    \end{tabular}
    \endgroup
    \caption{\label{fig:results-mutators-estimated}%
        ACCEL variants with max-latest estimator.
        Each policy is trained on $T$ environment steps, using the indicated algorithm, with underlying training distribution $\distrTrain[\alpha]$.
        \textit{(1\textsuperscript{st} row):}
            Average return over 512 steps for an evaluation batch of 256 distinguishing levels (cf.~\cref{fig:main-distinguishing-return}).
        \textit{(2\textsuperscript{nd} row):}
            The proportion of distinguishing levels sampled from the adversary's buffer across training (cf.~\cref{fig:main-effective-alpha}).
        \textit{(Both):} 
            Mean over $N$ seeds, shaded region is one standard error.
            Note the split axes used to show zero on the log scale.
    }
\end{figure}

\clearpage

\section{The maximin expected value objective is susceptible to goal misgeneralization}
\label{apx:maximin}

In this section, we discuss an alternative strategy to minimax expected regret for selecting robust policies.
Namely, we consider the \textbf{maximin expected value} (MMEV) training objective, whereby one seeks a policy that achieves the highest possible expected value (return) on a worst-case level distribution
    $\MMEVDistr
    \in
    \argmin_{\distr'\in\LevelDistrs}
        \Return{\RewardFunction}{\distr'}{\Policy}
    $.

\citet{Dennis+2020} argues that the MMEV objective fails to induce robustness in an environment where the optimal value of each level varies within the level space. This is because the agent has no incentive to improve performance in any level above the maximum performance in worst-case levels. This same obstacle prevents MMEV from inducing robustness to goal misgeneralization, even though a policy that pursues a proxy goal in distinguishing levels will achieve low return in these levels.

In this appendix, we show theoretically that MMEV-based methods allow for goal misgeneralization in environments with levels with low maximum value. Moreover, we show empirically that MMEV-based training methods fail to mitigate goal misgeneralization in our environments. Indeed, they fail to produce policies that perform capably in any levels.

\subsection{Theoretical results}
We show our results for a perfect MMEV agent. It is easy to extend the result in the case where the agent is only $\eps$-optimal---the performance on deployment levels could be suboptimal by a factor $\eps$ when such is allowed by the levels seen in deployment. The same argument holds for the adversary.
\begin{apxdefinition}[MMEV policy]\label{def:approx-mmev}
    Consider
        an UMDP $\TupleUMDP$,
        a goal $\RewardFunction$.
    The \emph{MMEV policy set} is then
    \begin{equation*}
        \MMEVOptimalPolicies[]{\RewardFunction}
        =
         \argmax_{\pi\in\AllPolicies}
            \min_{\distr \in \LevelDistrs}
                \Return{\RewardFunction}{\distr}{\Policy}.
    \end{equation*}
\end{apxdefinition}

\begin{apxdefinition}[MMEV adversary]\label{def:approx-mmev-adv}
    Consider
        an UMDP $\TupleUMDP$,
        a goal $\RewardFunction$, 
        an optimal policy $\pi^*$ in every level $\level \in \LevelSpace $.
    The \emph{MMEV adversary strategy} is then
    \begin{equation*}
    \MMEVDistr \in \argmin_{\distr'\in\LevelDistrs}
             \Return{\RewardFunction}{\distr'}{\OptimalPolicy}
    \end{equation*}
\end{apxdefinition}

Crucially, the MMEV adversary only cares about minimizing the return of the agent. Given this, the adversary will play levels in which \emph{any} agent would achieve the minimum return possible. These are, for example, impossible levels if such exists in the level space.

Let's now define some additional machinery we will need for the proof:

\begin{apxdefinition}[$\alpha$-minimum achievable return]\label{def:minalphareturn}
Given a level $\theta \in \Theta$ and a threshold $\alpha$, define 
\[
c(\level, \alpha) = \min \ReturnSymbol^{\RewardFunction}_{\alpha}(\level)
\]
where $\ReturnSymbol^{\RewardFunction}_{\alpha}(\level) = \{ x \in \Reals \mid \exists \pi \in \AllPolicies, \text{such that}~ \Return{\RewardFunction}{\level}{\Policy} \geq \alpha ~\text{and}~ \Return{\RewardFunction}{\level}{\Policy} = x \}$.
\end{apxdefinition}

As mentioned before, the intuition is that an MMEV adversary will play those levels where the minimum return is achieved by any policy. Thus, any agent trained on those levels will mostly be able to reach that performance on any level in deployment. Thus, this results in suboptimal performance. This is clear to see in the case where impossible levels exist in the level space of the adversary. An agent trained on those would constitute an MMEV policy, although it basically consists of a policy not able to pursue the correct goal.

We note that $c$ as just introduced, can be thought as a \textit{return floor}.
Readers familiar with the decision theory and UED literature can think of this as an (inverse) analog of the \textit{regret floor} or \textit{regret stagnation}~\cite{Beukman+2024}. While the latter quantifies the minimum amount of regret any policy will suffer against a level distribution, we quantify the minimum amount of return a policy will incur.

\begin{restatable}[MMEV is susceptible to goal misgeneralization]{theorem}{MMEVTheorem}
\label{thm:mmev-susceptibility}
Consider
    an UMDP $\TupleUMDP$,
    a pair of goals $\RewardFunction, \ProxyRewardFunction$,
    a proxy-distinguishing distribution shift $\TupleShift$,
and
    approximation threshold $\eps\geq 0$.
    Then, 
\begin{equation*}
    \exists \MMEVPolicy \in \MMEVOptimalPolicies[]{\RewardFunction}
    \text{\ such that\ }
    \Return{\RewardFunction}{\distrTest}{\MMEVPolicy}
    =
     c(\distrTest, \alpha)
\end{equation*}

where $\alpha = \Return{\RewardFunction}{\MMEVDistr}{\Policy^*} $ and  $c(\distrTest, \alpha) = \Expect[\level\sim\distrTest]{c( \level, \alpha)} $ 
\end{restatable}
\begin{proof}
Consider $\OptimalPolicy$ as the optimal policy across all levels $\level \in \Theta$. The adversary will play a strategy 
\[
\MMEVDistr \in \argmin_{\distr'\in\LevelDistrs}
             \Return{\RewardFunction}{\distr'}{\OptimalPolicy}
\]

In words, the adversary will play levels in which the lowest possible score is achieved by the optimal policy.

We now need to construct our policy $\MMEVPolicy$. Take the policy such that
\[
\Return{\RewardFunction}{\level}{\MMEVPolicy} =
\begin{cases}
\Return{\RewardFunction}{\level}{\pi^*}, & \text{if } \theta \in \supp\MMEVDistr, \\
c(\level, \alpha) , & \text{if }\theta \notin \supp\MMEVDistr.
\end{cases}
\]

We note that we always have $c(\level, \alpha) \geq \Return{\RewardFunction}{\level'}{\pi^*} $, where $\theta \in \LevelSpace, \theta' \in \supp\MMEVDistr$ by our definitions.
In words, take the MMEV policy such that it achieves the maximum return possible on levels played by the adversary, and the smallest available return on all other levels. Note that the smallest available return on levels not played by the adversary is necessarily greater than the highest one possibly achievable on levels played by the adversary.
A policy that achieves this return is clearly in an MMEV policy. So, the following holds

\[
\Return{\RewardFunction}{\distr'}{\MMEVPolicy}
    =
    c(\distr', \alpha) 
.\]
Taking $\distr'=\distrTest$ we get 
\[
\Return{\RewardFunction}{\distrTest}{\MMEVPolicy}
    =
    c(\distrTest, \alpha) 
.\]
as desired.
\end{proof}

In our theorem, we proved that an MMEV agent will possibly achieve returns that are at most the maximum ones achievable in levels played by the adversary. If very low return levels exists, this policy would then possibly goal misgeneralize at test time if evaluated on levels where a high return is possible.

\subsection{Training methods and experimental results}

In this section, we outline our empirical evaluation of MMEV-based training methods for mitigating goal misgeneralization. We adapt three UED methods: \RobustPLR{} along with two ACCEL variants (\ACCELid{} and \ACCELbin{}, see \cref{apx:mutators}). These UED methods were originally designed for MMER training, but we can convert their regret-maximizing adversaries into return-minimizing adversaries simply by replacing the regret estimator used to refine the buffer with a negative expected return estimator,
\begin{equation}\label{eq:maximin-estimator}
    \MMActor{\RewardFunction}{\level}{\Policy}
    =
    - \hat{V}^{\RewardFunction}_{\textnormal{latest}}(\Policy; \level)
    ,
\end{equation}
where $\hat{V}^{\RewardFunction}_{\textnormal{latest}}(\Policy;\level)$ is the empirical average return achieved on $\level$ in the latest batch of rollouts with the current policy $\Policy$.

\Cref{fig:results-maximin} shows the performance of trained policies in \env{cheese in the corner} and \env{keys and chests} for non-distinguishing and distinguishing levels.
As expected, these training methods fail to mitigate goal misgeneralization, and even fail to induce robust performance in non-distinguishing levels.

We observe that these adversaries rapidly fill their buffers with levels with zero estimated value, indicating that the adversaries are working as expected.
We hypothesize that the training failure is primarily due to the adversary finding levels in which the policy not only receives low expected value, but never receives any nonzero reward and thus obtains no training signal.

\begin{figure}
    \centering
    \begingroup
    \setlength{\tabcolsep}{1pt}
    \begin{tabular}{cccc}
    \toprule
    &
    & \bf \env{Cheese in the corner}
    & \bf \env{Keys and chests}
    \\
    & 
    & \small Seeds $N{=}5$, steps $T{=}200$M
    & \small Seeds $N{=}5$, steps $T{=}400$M \\
    \midrule
    \multirow{2}{*}[2em]{\rotatebox[origin=l]{90}{Average return}}
    & \multirow{1}{*}[6em]{\rotatebox[origin=l]{90}{Distinguishing}}
    & \includegraphics{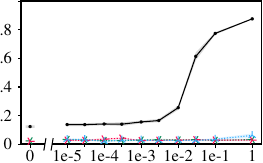}
    & \includegraphics{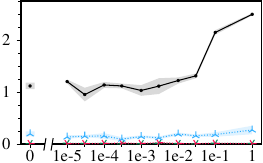}
    \\
    & \multirow{1}{*}[6em]{\rotatebox[origin=l]{90}{Non-distinguishing}}
    & \includegraphics{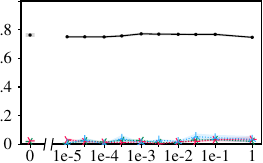}
    & \includegraphics{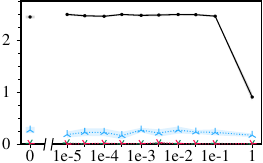}
    \\
    & & \multicolumn{2}{c}{Prop.\ distg.\ levels in underlying training distr., $\alpha$ (aug.\ log)}
    \\
    \midrule
    & & \multicolumn{2}{c}{\includegraphics{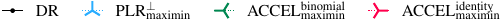}}
    \\[-0.5ex]
    \bottomrule
    \end{tabular}%
    \endgroup
    \caption{\label{fig:results-maximin}%
        Maximin training methods.
        Each policy is trained on $T$ environment steps, using the indicated algorithm, with underlying training distribution $\distrTrain[\alpha]$.
        Average return over 512 steps for an evaluation batch of 256 levels (top row: distinguishing levels, bottom row: non-distinguishing levels).
        Mean over $N$ seeds, shaded region is one standard error. Note the split axes used to show zero on the log scale.
        Note the drop in performance for DR in non-distinguishing \env{keys and chests} levels at $\alpha=1$ can be explained by noting that non-distinguishing levels are now out of distribution given this training distribution.
    }
\end{figure}

\begin{figure}
    \centering
    \begingroup
    \setlength{\tabcolsep}{1pt}
    \begin{tabular}{rc}
    \toprule
    & \bf \env{Cheese in the corner}
    \\
    & \small Seeds $N{=}5$, steps $T{=}200$M
    \\
    \midrule
    \multirow{2}{*}[6.5em]{\begin{minipage}{6em}\RaggedRight
        Average proportion of unsolvable levels in adversary's buffer
    \end{minipage}}
    & \includegraphics{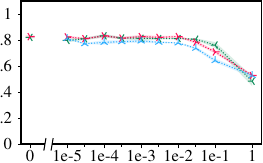}
    \\
    & \begin{minipage}{12.5em}\RaggedRight
        Proportion of distinguishing levels in underlying training distribution, $\alpha$ (augmented log scale)
    \end{minipage}
    
    \\
    \midrule
    \multicolumn{2}{c}{\includegraphics{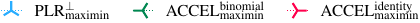}}
    \\[-0.5ex]
    \bottomrule
    \end{tabular}%
    \endgroup
    \caption{\label{fig:maximin-solvability}%
        Maximin training methods.
        Each policy is trained on $T$ environment steps, using the indicated algorithm, with underlying training distribution $\distrTrain[\alpha]$.
        Average proportion of unsolvable levels in the adversary's buffer over training.
        Mean over $N$ seeds, shaded region is one standard error. Note the split axes used to show zero on the log scale.
    }
\end{figure}

In the extreme case, the adversary could populate the buffer with levels that have zero \emph{maximum} value, preventing any policy from obtaining nonzero reward.
In \env{cheese in the corner}, levels in which the cheese position is unreachable from the mouse spawn position have zero maximum value.
In \cref{fig:maximin-solvability}, we plot the average proportion of such ``unsolvable'' levels in the adversary's buffer over training for \env{cheese in the corner}.
We find that this proportion is around $80\%$ for most training distributions, which is much higher than the baseline value of around $18\%$ of levels sampled from $\distrNonDistg$ that are unsolvable. As the proportion of distinguishing levels increases, the average proportion of unsolvable levels decreases slightly, to around $50\%$ for $\alpha=0$. Note that the baseline value for $\distrDistg$ is also lower (around $7\%$) since it's easier for walls to obstruct the cheese when it's in the corner than when it is in an arbitrary position in the interior of the grid. Therefore we hypothesize that the adversaries have a slightly harder time finding unsolvable levels as $\alpha$ increases. However, unsolvable levels still occupy a majority of the buffer.
 
\clearpage

\section{Experiments with increased training length}
\label{apx:moresteps}

In most of our experiments, we compare the performance of algorithms after a fixed amount of training time, and find that MMER-based training methods typically outperform MEV-based training methods for a fixed training budget. We have shown that this is because a regret-maximizing adversary can amplify the proportion of distinguishing levels compared to sampling from a fixed underlying training distribution.

Another method of increasing the agent's experience in distinguishing levels is to train for longer. In this appendix, we extend training times in the \env{cheese on a dish} training environment (where DR was most robust to goal misgeneralization).

\Cref{fig:results-moresteps} shows the results. We find that training for more than 200 million environment steps with a fixed training method slightly mitigates goal misgeneralization. In particular, for $\alpha = 1\text{e-}3$, DR gradually stops misgeneralizing after 200 million steps. 
The result is qualitatively consistent with \cref{thm:mev-susceptibility}, since increasing training time should have the effect of decreasing the optimization threshold.

However, for lower $\alpha$ values, further training shows diminishing returns, and even 1,500 million steps of DR training is insufficient to mitigate goal misgeneralization to the extent achieved by most UED methods within 200 million steps.
This suggests that current UED methods are both more efficient and also qualitatively more effective at mitigating goal misgeneralization in this environment.

\begin{figure}[h!]
    \centering
    \begingroup
    \setlength{\tabcolsep}{1pt}
    \begin{tabular}{cccc}
    \toprule
    \multicolumn{4}{c}{\bf \env{Cheese on a dish}}
    \\
    \midrule
    & \multirow{2}{*}{\bf DR}
    & \bf \RobustPLR{}
    & \bf \ACCELid{}
    \\
    & 
    & \multicolumn{2}{c}{Regret estimator: oracle-latest (1\textsuperscript{st} row), max-latest (2\textsuperscript{nd} row)}
    \\
    \midrule
    \multirow{2}{*}[6.5em]{\rotatebox[origin=l]{90}{Average return (ditsinguishing levels)}}
    & \multirow{2}{*}[6.98em]{\includegraphics{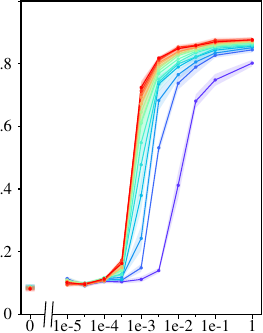}}
    & \includegraphics{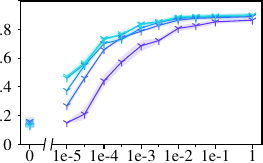}
    & \includegraphics{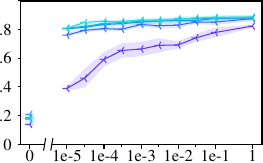}
    \\
    & 
    & \includegraphics{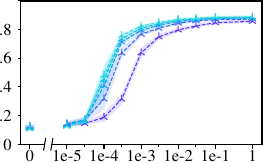}
    & \includegraphics{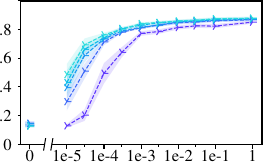}
    \\
    & \multicolumn{3}{c}{Proportion of distinguishing levels in underlying training distribution, $\alpha$ (augmented log scale)}
    \\
    \midrule
    & \multicolumn{3}{c}{\includegraphics[trim={0pt 3pt 0pt 0pt}]{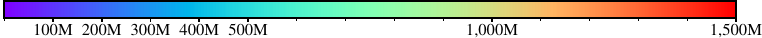}}
    \\
    & \multicolumn{3}{c}{Environment steps used for training}
    \\
    \bottomrule
    \end{tabular}%
    \endgroup
    \caption{\label{fig:results-moresteps}%
        Training for more environment steps in \env{cheese on a dish}.
        Each policy is trained for 500M environment steps (1,500M for DR), using the indicated algorithm, with underlying training distribution $\distrTrain[\alpha]$.
        Every 100M steps, we evaluate the average return over 512 steps for an evaluation batch of 256 distinguishing levels (cf.~\cref{fig:main-distinguishing-return}).
        Mean over $3$ seeds, shaded region is one standard error.
        Note the split axes used to show zero on the log scale.
    }
\end{figure}

\clearpage

\section{Experiments with a different distinguishing level generator}
\label{apx:robustness-corner}

In this section, we consider an alternative procedural level generator for distinguishing levels in the \env{cheese in the corner} environment. Our main experiments use a distinguishing level generator that places the cheese anywhere in the maze.

We consider a restricted distinguishing level generator that positions the cheese only within a region of size $c \times c$ surrounding the top-left corner, for varying $c$ (the non-distinguishing generator would be recovered with $c=1$, and the original, unrestricted distinguishing level generator would be recovered with $c=13$).

\Cref{fig:robustness_corner} shows the return is evaluated on unrestricted distinguishing levels, where the cheese is positioned anywhere in the maze.
We find that the UED methods are able to amplify the proportion of restricted distinguishing levels and in some cases mitigate goal misgeneralization.
DR achieves low return across all corner sizes $c$.

\begin{figure}[ht]
  \centering
  \begingroup
  \setlength{\tabcolsep}{1pt}
  \begin{tabular}{ccccc}
    \toprule
    \multicolumn{5}{c}{\textbf{\env{Cheese in the corner},} restricted training distribution}
    \\
    \midrule
    & $\alpha = 3\text{e-}4$
    & $\alpha = 1\text{e-}3$
    & $\alpha = 3\text{e-}3$
    & $\alpha = 1\text{e-}2$
    \\
    & \small Seeds $N{=}1$
    & \small Seeds $N{=}1$
    & \small Seeds $N{=}3$
    & \small Seeds $N{=}1$
    \\
    \midrule
    \rotatebox[origin=l]{90}{Avg.\ return (ditsg.)}
    & \includegraphics{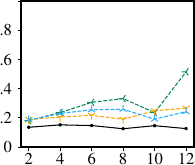}
    & \includegraphics{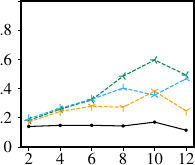}
    & \includegraphics{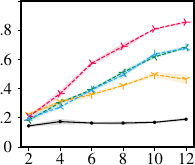}
    & \includegraphics{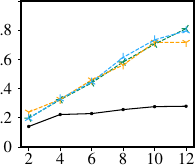}
    \\
    \midrule
    \multirow{2}{*}[7em]{\rotatebox[origin=l]{90}{Adv.\ prop.\ ditsg. (log)}}
    & \includegraphics{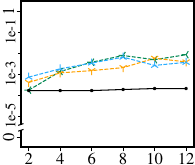}
    & \includegraphics{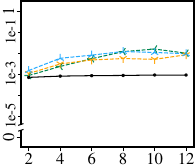}
    & \includegraphics{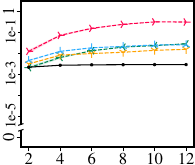}
    & \includegraphics{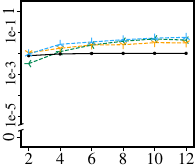}
    \\
    & \multicolumn{4}{c}{Distinguishing corner region size, $c$}
    \\
    \midrule
    & \multicolumn{4}{c}{\includegraphics{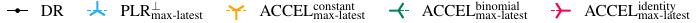}}
    \\[-0.5ex]
    \bottomrule
  \end{tabular}
  \endgroup
  \caption{\label{fig:robustness_corner}%
    \textbf{Training with varying distinguishing level generators.}
    Each policy is trained on $200$ million environment steps, using the indicated algorithm, with underlying training distribution
        $\distrTrain[\alpha,c] = (1-\alpha) \distrNonDistg + \alpha \distrDistg^{c}$,
    where $\distrDistg^{c}$ is a procedural level generator that positions the cheese in the top-left $c \times c$ region of the maze.
    \textit{(1st row):}
        Average return over 512 steps for an evaluation batch of 256 distinguishing levels (cf.~\cref{fig:main-distinguishing-return}).
    \textit{(2nd row):}
        The proportion of distinguishing levels sampled from the adversary's buffer across training (cf.~\cref{fig:main-effective-alpha}).
    \textit{(Both):} 
        Mean over $N$ seeds, shaded region is one standard error.
        Note the split vertical axes used to show zero on the log scale.
}
\end{figure}

\clearpage

\section{Experiments with different observations}
\label{apx:robustness-dish}

In this section, we vary the way we encode observations for the policy in \env{cheese on a dish} and explore its effect on goal misgeneralization.

As discussed in \cref{apx:environments:dish}, the proxy goal and the true goal are symmetric in this environment, other than the fact that the position of the dish is redundantly represented across multiple channels in the Boolean grid observation.
In the main text, we use $D=6$ channels to encode the dish position (compared to $1$ channel for the cheese position).
The additional channels break a symmetry and create a slight inductive bias in favor of a policy that pursues the proxy goal.

We conduct an experiment where we vary the number of channels and see what affect it has on goal misgeneralization. \Cref{fig:robustness_pile} shows the results.
We see that with one channel coding the dish position, all methods (including DR) are somewhat robust to goal misgeneralization, even at small $\alpha$ values.
Additional channels significantly increase DR's susceptibility to goal misgeneralization. On the other hand, UED methods remain able to identify and amplify the training signal from rare distinguishing levels, while UED methods retain comparably similar performance.

The extent of amplification remains essentially constant with $D$. We hypothesize that the number of channels does not stop the adversary from noticing high-regret distinguishing levels, though it may affect how the policy responds.

\begin{figure}[ht]
  \centering
  \begingroup
  \setlength{\tabcolsep}{1pt}
  \begin{tabular}{ccccc}
    \toprule
    \multicolumn{5}{c}{\textbf{\env{Cheese on a dish},} varying observations}
    \\
    \midrule
    & $\alpha = 3\text{e-}4$
    & $\alpha = 1\text{e-}3$
    & $\alpha = 3\text{e-}3$
    & $\alpha = 1\text{e-}2$
    \\
    & \small Seeds $N{=}3$
    & \small Seeds $N{=}3$
    & \small Seeds $N{=}3$
    & \small Seeds $N{=}3$
    \\
    \midrule
    \rotatebox[origin=l]{90}{Avg.\ return (ditsg.)}
    & \includegraphics{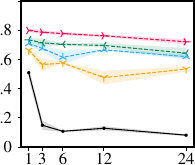}
    & \includegraphics{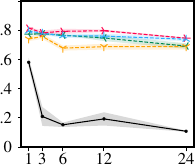}
    & \includegraphics{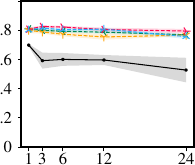}
    & \includegraphics{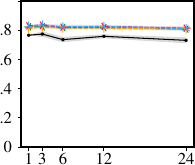}
    \\
    \midrule
    \multirow{2}{*}[7em]{\rotatebox[origin=l]{90}{Adv.\ prop.\ ditsg. (log)}}
    & \includegraphics{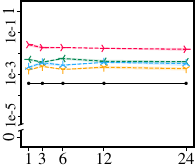}
    & \includegraphics{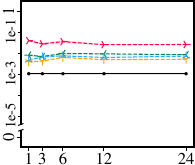}
    & \includegraphics{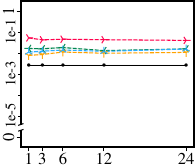}
    & \includegraphics{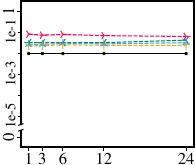}
    \\
    & \multicolumn{4}{c}{Number of channels coding the dish position, $D$}
    \\
    \midrule
    & \multicolumn{4}{c}{\includegraphics{figures/plots/legends/robustness.pdf}}
    \\[-0.5ex]
    \bottomrule
  \end{tabular}
  \endgroup
  \caption{\label{fig:robustness_pile}%
    \textbf{Training with observations with varying emphasis on the dish.}
    Each policy is trained on $200$ million environment steps, using the indicated algorithm, with underlying training distribution
        $\distrTrain[\alpha]$.
    We vary the number of channels (features) encoding the dish position, $D$.
    \textit{(1st row):}
        Average return over 512 steps for an evaluation batch of 256 distinguishing levels (cf.~\cref{fig:main-distinguishing-return}).
    \textit{(2nd row):}
        The proportion of distinguishing levels sampled from the adversary's buffer across training (cf.~\cref{fig:main-effective-alpha}).
    \textit{(Both):} 
        Mean over $N$ seeds, shaded region is one standard error.
        Note the split vertical axes used to show zero on the log scale.
}
\end{figure}

\stopcontents[appendix]

\end{document}